\documentclass[twoside,11pt]{article}
\usepackage[T1]{fontenc}
\usepackage{lmodern,microtype,mathtools,amsthm}
\usepackage[abbrvbib,preprint]{jmlr2e}
\usepackage[letterpaper,margin=1in]{geometry}
\usepackage{booktabs,array,tabularx,longtable,enumitem,needspace,float}
\usepackage{pgfplots}
\floatstyle{ruled}
\newfloat{algorithm}{tbp}{loa}
\floatname{algorithm}{Algorithm}
\makeatletter
\def\caption@documentclass{standard}
\makeatother
\usepackage[font=small,labelfont=it,labelsep=colon,format=hang,justification=raggedright]{caption}
\makeatletter
\renewenvironment{proof}[1][Proof]{\par
  \pushQED{\qed}\normalfont\topsep6\p@\@plus6\p@\relax
  \trivlist\item[\hskip\labelsep\bfseries #1\@addpunct{.}]\ignorespaces
}{\popQED\endtrivlist\@endpefalse}
\makeatother

\hypersetup{hidelinks,
  pdftitle={Geometry-Dependent Online and Offline Approximation for Non-Monotone DR-Submodular Maximization},
  pdfauthor={Vaneet Aggarwal and Yiyang Lu}}
\allowdisplaybreaks[2]
\numberwithin{theorem}{section}
\newcommand{\R}{\mathbb R}
\newcommand{\E}{\mathbb E}
\newcommand{\one}{\mathbf 1}
\newcommand{\zero}{\mathbf 0}

\newcommand{\alow}{\underline{\alpha}}
\newcommand{\aup}{\overline{\alpha}}

\newcommand{\aor}{\alpha_{\mathrm{oracle}}}
\newcommand{\aon}{\alpha_{\mathrm{on}}}
\newcommand{\ip}[2]{\langle #1,#2\rangle}
\newcommand{\norm}[1]{\lVert #1\rVert}
\DeclareMathOperator{\Reg}{Reg}
\DeclareMathOperator{\diam}{diam}

\DeclareMathOperator{\OPT}{OPT}
\pgfplotsset{compat=1.18}
\title{Geometry-Dependent Online and Offline Approximation\\
for Non-Monotone DR-Submodular Maximization}
\author{\name Vaneet Aggarwal \\
  \addr Purdue University 
  \AND
  \name Yiyang Lu \\
  \addr Purdue University}
\firstpageno{1}
\begin{document}
\pagestyle{plain}
\raggedbottom
\maketitle
\thispagestyle{plain}
\begin{abstract}
We study online and offline maximization of nonnegative, non-monotone DR-submodular functions over compact convex down-closed sets $K\subseteq[0,1]^d$ that contain a supplied diagonal point $\zeta\one$. We give an online algorithm whose expected total reward is at least $\underline\alpha(\zeta)=\tfrac12-(1-2\zeta)_+^2/[2(3-2\zeta)^2]$ times that of the best fixed feasible point, minus $O(\sqrt T)$, using one gradient query and one projection onto $K$ per round; the gradients may be stochastic, and $\underline\alpha(0)=4/9$. From value queries alone, the same coefficient is attained with $O(T^{3/4})$ regret from one post-decision noisy value query on the cube per round, or $O(\sqrt T)$ regret from two exact value queries. Conversely, no offline algorithm that makes polynomially many queries, each returning the exact value and gradient, attains a coefficient above an explicit constant $\beta_*\approx0.4704$ at $\zeta=0$. This improves the bound $2(e^{-1/2}-e^{-1})\approx0.478$ of Oveis Gharan and Vondr\'ak (2011), and for positive $\zeta$ two explicit families of hard instances give an upper bound, which we reduce exactly to finitely many convex problems below $\zeta=1/8$ and give in closed form above it. The bound $\beta_*$ also holds for nonnegative submodular set functions under a cardinality constraint, improving the previous $0.478$. An online-to-offline conversion gives the coefficient $\underline\alpha(\zeta)-\varepsilon$ offline with explicit query counts, and running an online learner on a repeated objective shows that the offline upper bound also limits online learners with polynomial query budgets and uniform regret bounds of order $T^\kappa$ with $\kappa<1$. The best online and offline coefficients both equal $1/2$ for $\zeta\ge1/2$ and fall short of $1/2$ by an amount of order $(1/2-\zeta)^2$ below this level. For quadratic objectives we attain the coefficient $1/2$, and for cubic objectives $\max\{8/17,2\sqrt{3-4r}/(1+\sqrt{3-4r})^2\}$ with $r=\min\{\zeta,1/2\}$, online and offline. For submodular set functions, the multilinear extension and established rounding give $\underline\alpha(\zeta)-\varepsilon$ under one matroid whose polytope contains $\zeta\one$ and $4/9-\varepsilon$ under a fixed number of knapsack constraints.
\end{abstract}


\section{Introduction}\label{sec:intro}
In many allocation problems, adding more of a resource helps less as the allocation grows and can eventually hurt: activating more items creates interference, oversupplying one component wastes budget, and placing probability mass on competing choices lowers the expected payoff. A differentiable function on the cube $[0,1]^d$ models this when each partial derivative is nonincreasing as any coordinate grows, including its own, that is, when $u\le w$ implies $\nabla f(u)\ge\nabla f(w)$ coordinatewise; such functions are called DR-submodular, where DR stands for diminishing returns. They are concave along nonnegative directions but not concave in general, so even offline maximization over a convex set is approximated rather than solved exactly \citep{bian2017guaranteed,bian2017continuous,bian2019}.

We study maximization of nonnegative, non-monotone DR-submodular functions over a compact convex set $K\subseteq[0,1]^d$ that is down-closed, meaning that $x\in K$ and $\zero\le y\le x$ imply $y\in K$. In the offline problem, the algorithm queries one objective and then returns a feasible point. In the online problem, a learner chooses $p_t\in K$ on each round $t=1,\ldots,T$ before observing information about the objective $f_t$ of that round. For $\alpha\in(0,1]$, its $\alpha$-regret is $\alpha\max_{o\in K}\sum_tf_t(o)-\sum_tf_t(p_t)$, and we call $\alpha$ the approximation coefficient. The geometric information is a feasible point $\zeta\one$ on the diagonal of the cube, where $\zeta\in[0,1]$ is supplied. Under a cardinality relaxation $\sum_ix_i\le k$ one can take $\zeta=k/d$, and under packing constraints $Ax\le b$ with $A\ge0$ any $\zeta$ with $\zeta A\one\le b$ works.

Our first main result is the online coefficient
\begin{equation}\label{eq:main-profile}
 \alow(\zeta)
 =\frac12-\frac{(1-2\zeta)_+^2}{2(3-2\zeta)^2},
 \qquad \text{where } (s)_+=\max\{s,0\},
\end{equation}
which equals $4/9$ at $\zeta=0$. With one gradient query on the cube and one Euclidean projection onto $K$ per round, the expected $\alow(\zeta)$-regret is $O(\sqrt T)$, also when the gradients are stochastic and only conditionally unbiased. The coefficient $4/9$ improves the online coefficient $0.401$ of \citet{aggarwal2026} and the offline coefficient $0.401$ of \citet{buchbinder2024}. The same coefficient is attained from value queries alone, with regret $O(T^{3/4})$ from one noisy value query per round or $O(\sqrt T)$ from two exact value queries. In every case the rewarded action lies in $K$, while the queries are made after the action is fixed and may lie anywhere in the cube. Table~\ref{tab:perspective} compares these query models and rates with earlier work.

Our second main result is an upper bound on the coefficient of any offline algorithm that makes polynomially many queries, each returning the exact value and the full gradient of the objective. At $\zeta=0$ the bound is
\begin{equation}\label{eq:intro-upper-endpoint}
 \beta_* = \frac{4\tau_*}{1+6\tau_*+\tau_*^2}\approx0.4704,
 \qquad \text{where } \tau_*=2q_*(1-q_*)\ \text{ and }\
 q_*=\sqrt{\frac{2-3e^{-1/2}}{2-2e^{-1/2}}},
\end{equation}
which improves the bound $2(e^{-1/2}-e^{-1})\approx0.478$ of \citet[Appendix~E.2]{oveis2011}. For $\zeta>0$, two families of hard instances give upper bounds $U(\zeta)$ and $C(\zeta)$, and we write $\aup(\zeta)=\min\{U(\zeta),C(\zeta)\}$. For $0<\zeta<1/8$ we reduce $U(\zeta)$ exactly to finitely many convex minimization problems over the weights of the hard objective, and for $\zeta\ge1/8$ we give $U$ in closed form as a piecewise rational function. The bound $\beta_*$ also holds for nonnegative submodular set functions under a single cardinality constraint, and therefore under one matroid, with value queries.

Each main result transfers to the other setting. The online-to-batch conversion of \citet{pedramfar2024linear}, called OTB there, turns the online algorithm into an offline one with an explicit number of queries. Conversely, an online learner can be run on a sequence that repeats one offline objective, so the offline upper bound also limits every online learner with a polynomial query budget and a uniform regret bound $P_R(d,1+D,1+M,1+B)\,T^\kappa$, where $P_R$ is a fixed polynomial and $\kappa<1$ is fixed (Section~\ref{sec:online-upper-transfer}). For each supplied level $\zeta$, the best online and offline coefficients are worst-case guarantees over all admissible objectives and feasible sets containing $\zeta\one$, and at $\zeta=0$ both lie in $[4/9,\beta_*]$. They equal $1/2$ for $\zeta\ge1/2$, and below this level they are smaller than $1/2$ by an amount of order $(1/2-\zeta)^2$.

\paragraph{Technical novelty in achievability.}
The algorithm is driven by an inequality that holds for every state $x$ and every comparator $o$ in $[0,1]^d$. The state determines the action $h_\zeta(x)=x/(c_\zeta+x)$, where $c_\zeta=\max\{1/2,1-\zeta\}$, and the inequality bounds $\alow(\zeta)f(o)-f(h_\zeta(x))$ by $\alow(\zeta)\ip{g}{o-x}$ for a vector $g$, which we call the field, that depends on $f$ and $x$ but not on $o$. Comparisons of this form, called upper-linearizable, were introduced by \citet{pedramfar2024linear} and have been used for online approximation guarantees for monotone up-concave functions \citep{pedramfar2024linear}, non-monotone DR-submodular functions \citep{pedramfar2024linear,lu2026}, and $\gamma$-weakly $\theta$-up-concave functions \citep{pedramfar2026weakly}. To prove it, we join the cube endpoints to $h_\zeta(x)$ by a lower and an upper coordinate path. An exact identity makes the coefficient along the lower path constant, and ordering the remaining coefficients makes the intermediate terms of their weighted sum nonpositive, although individual increments of $f$ can have either sign; the argument also keeps a nonnegative multiple of $f(\zero)$. Convexity and down-closedness turn $\zeta\one\in K$ into feasibility of $h_\zeta(x)$. Sampling one coordinate and one point on the paths gives an unbiased estimate of $g$ from a single gradient response. The online regret bound then follows from the reduction of \citet{pedramfar2024linear}, called OMBQ there, which runs projected gradient ascent on the estimated fields.

\paragraph{Technical novelty in the upper bound.}
An upper bound of this kind is proved through the symmetry gap of \citet{vondrak2013}: for an objective and feasible set that are invariant under a group of coordinate permutations, it is the ratio between the best value at points that are fixed by the group and the best value overall. Our hard objective combines four submodular components that all take the value one at a common asymmetric point but penalize symmetric points in two competing ways. Exact quadratic inequalities control both kinds of symmetric points and identify $\beta_*$. Coupled packing constraints keep the diagonal point $\zeta\one$ feasible; matching primal and dual inequalities evaluate the resulting bound $U$ for $\zeta\ge1/8$, and an explicit cutoff reduces every smaller positive level to finitely many optimizations over the component weights. A second construction gives $C$. To pass from a symmetry gap to a lower bound against algorithms that also receive full gradients, we smooth the objective so that values and gradients agree exactly near the symmetric subspace, and we show that the constraint oracle reveals no information about the hidden permutation. The lower bound therefore applies to the same access to the objective on the whole cube that the algorithms use. Replicating one group of coordinates turns the four-component instance into one with a single cardinality constraint.

Polynomial structure gives further coefficients. For DR-submodular polynomials of degree at most two, we attain the coefficient $1/2$. For degree at most three, two constructions give $8/17$ and $2\sqrt{3-4\zeta}/(1+\sqrt{3-4\zeta})^2$ for $0\le\zeta\le1/2$; their maximum is $\alpha_3(\zeta)$, and we attain $1/2$ for $\zeta\ge1/2$. The first construction chooses at random between two explicit scalar maps, and the second uses one map adapted to $\zeta$. Inequalities for each component of the polynomial combine into an upper-linearizable comparison for every state and comparator. We implement the field from gradients or from values, treating exact reconstruction of the coefficients and noisy estimation separately, and the same comparisons give offline guarantees. For submodular set functions, we apply these results to the multilinear extension, which averages the set function over a random set containing each element independently with probability equal to its coordinate, and established rounding methods give results under a matroid and under a fixed number of knapsack constraints.

\paragraph{Contributions.}
\begin{enumerate}[leftmargin=*,itemsep=4pt]
\item \textbf{Online achievability.}
Section~\ref{sec:algebraic} proves the upper-linearizable comparison with coefficient $\alow(\zeta)$. Corollaries~\ref{thm:first-order} and~\ref{cor:value-feedback} attain $\alow(\zeta)$ with $O(\sqrt T)$ regret from one gradient or two exact values per round, and with $O(T^{3/4})$ regret from one noisy value.
\item \textbf{Offline upper bound.}
Section~\ref{sec:upper} proves the bound $\beta_*$ at $\zeta=0$ and the two bounds $U$ and $C$, which Theorem~\ref{thm:upper} combines into $\aup(\zeta)$ for algorithms with exact values and full gradients. Proposition~\ref{prop:upper-evaluation} evaluates $U$.
\item \textbf{Consequences for both settings.}
Section~\ref{sec:combined} gives the offline query complexities and proves that offline upper bounds apply to online learners. Theorem~\ref{thm:intervals} states the resulting intervals for both coefficients, the common value $1/2$ for $\zeta\ge1/2$, and the quadratic order of the shortfall below it.
\item \textbf{Submodular set functions.}
Section~\ref{sec:discrete} combines the continuous guarantees with swap rounding for matroids and with the reduction for a fixed number of knapsacks. This gives a $(4/9-\varepsilon)$-approximation under one matroid or a fixed number of knapsack constraints, improving the coefficient $0.401$ of \citet[Corollary~1.3]{buchbinder2024}, and $\alow(\zeta)-\varepsilon$ for a matroid whose polytope contains $\zeta\one$. Lemma~\ref{lem:cardinality-upper} and Corollary~\ref{cor:matroid-upper} show that no algorithm with polynomially many value queries exceeds $\beta_*$ under a cardinality constraint or one matroid.
\item \textbf{Quadratic and cubic objectives.}
Theorem~\ref{thm:polynomial} gives the coefficient $1/2$ for quadratics and $\alpha_3(\zeta)$ for cubics. Section~\ref{sec:polynomial-oracles} implements both from gradients or values, including $O(T^{3/4})$ regret from one noisy value per round, and Corollary~\ref{cor:polynomial-offline} gives the offline versions with explicit query counts.
\end{enumerate}

\begin{table}[H]
\centering\small
\setlength{\tabcolsep}{4pt}
\begin{tabularx}{\textwidth}{@{}>{\raggedright\arraybackslash}p{3.7cm}>{\raggedright\arraybackslash}p{2.3cm}>{\raggedright\arraybackslash}X@{}}
\toprule
Result & Coefficient & Setting and access\\
\midrule
\multicolumn{3}{@{}l}{\textit{Achievable approximation coefficients}}\\[2pt]
\citet{buchbinder2024} & $0.401$ & Offline; DR objectives satisfying their smoothness assumptions, over polytopes on which linear functions can be optimized; set-value queries for submodular set functions.\\
\citet{lu2026} & $1/e$ & Online; one gradient query and $O(\sqrt T)$ regret.\\
\citet{aggarwal2026} & $0.401$ & Online; one noisy value query and $O(T^{4/5})$ regret.\\
Corollary~\ref{thm:first-order} & $\alow(\zeta)$; $4/9$ at $0$ & Online; one gradient, one projection, and $O(\sqrt T)$ regret.\\
Corollary~\ref{cor:value-feedback} & $\alow(\zeta)$; $4/9$ at $0$ & Online; one post-decision noisy value query on the cube and $O(T^{3/4})$ regret, or two exact value queries and $O(\sqrt T)$ regret.\\
Section~\ref{sec:offline} & $\alow(\zeta)-\varepsilon$ & Offline; gradient or value queries and projections.\\
Theorem~\ref{thm:polynomial} & $1/2$; $\alpha_3(\zeta)$ & Quadratic and cubic objectives; gradient or value queries, polynomial time.\\
\midrule
\multicolumn{3}{@{}l}{\textit{Upper bounds at $\zeta=0$ for algorithms with polynomially many queries}}\\[2pt]
\citet{oveis2011} & $2(e^{-1/2}-e^{-1})$ & Submodular set functions under one matroid; value queries.\\
\citet{qi2022} & $0.478$ & Submodular set functions under a cardinality constraint; value queries.\\
Corollary~\ref{cor:upper-endpoint} & $\beta_*$ in \eqref{eq:intro-upper-endpoint} & DR objectives over down-closed sets; exact values and full gradients.\\
Lemma~\ref{lem:cardinality-upper} & $\beta_*$ & Submodular set functions under a cardinality constraint or one matroid; value queries.\\
\bottomrule
\end{tabularx}
\caption{Earlier coefficients and the results of this paper. Our continuous algorithms query the objective on $[0,1]^d$ and access $K$ through a separate projection; the matroid results also use the independence and decomposition oracles of Section~\ref{sec:discrete-model}. The polynomial coefficients apply only to the indicated subclasses.}\label{tab:perspective}
\end{table}

\section{Related Work}\label{sec:related}
\paragraph{Offline continuous relaxations.}
Continuous relaxations and the measured continuous greedy algorithm underlie constrained non-monotone submodular maximization \citep{feldman2011,ene2016,buchbinder2019}, and DR-submodularity extends this approach to continuous objectives \citep{bian2017guaranteed,bian2017continuous}. \citet{buchbinder2024} obtain the coefficient $0.401$ from stronger DR inequalities and an asymmetric procedure that depends on the objective, under their smoothness and solvability assumptions. Our offline algorithm instead follows from a comparison that holds before the objective is observed, and its coefficient depends on the diagonal point.

\paragraph{Online DR-submodular maximization.}
Online continuous submodular optimization was introduced by \citet{chen2018}, and non-monotone objectives and weaker feedback were studied by \citet{thang2021,zhang2023,pedramfar2024unified}. \citet{aggarwal2026} obtain $0.401$ with value queries made after each decision, controlling additional comparator-dependent terms with an asymmetric online learner. For general convex constraints, \citet{mualem2023} give an online coefficient $(1-h)/4$ with $h=\min_{x\in K}\|x\|_\infty$; every down-closed set contains $\zero$, so their coefficient is $1/4$ here. Their parameter measures the distance of $K$ from the origin, whereas $\zeta$ measures how far $K$ extends along the diagonal.

\paragraph{Upper-linearizable optimization.}
The framework used here is due to \citet{pedramfar2024linear}. They define upper-linearizable comparisons, reduce online maximization to such a comparison (their OMBQ), convert online algorithms to offline ones (their OTB), and convert gradient feedback to value feedback by smoothing (their FOTZO), with applications to monotone up-concave and non-monotone DR-submodular objectives. The framework was later applied to non-monotone DR-submodular maximization over down-closed sets with coefficient $1/e$ by \citet{lu2026}, and to $\gamma$-weakly $\theta$-up-concave functions by \citet{pedramfar2026weakly}. We use OMBQ and OTB with projected gradient ascent \citep{zinkevich2003}, and adapt FOTZO so that all value queries stay in the cube. Our comparison has a single linear term in the comparator at coefficient $4/9$, and our implementation results supply the unbiased field estimate, its conditional moments, and the query procedures for polynomials; for the random cubic action we average the comparison conditional on the state, and for noisy reconstruction of polynomial coefficients we bound the additional comparison error.

\paragraph{Coordinate paths and polynomial structure.}
The two coordinate paths resemble continuous double greedy \citep{bian2019} and online unconstrained submodular maximization \citep{roughgarden2018}. Here the action is determined by the learner's state, the comparator is fractional, and the coordinate order is taken from the state. The results for fixed degree use graph representations of pairwise submodular functions \citep{kolmogorov2004} and the quadratization of cubic terms \citep{billionnet1985}; we prove the decompositions into components, their extension to general polynomials by polarization, and the scalar inequalities for the two cubic actions.

\paragraph{Upper bounds.}
\citet{oveis2011} obtain the symmetry gap $2(e^{-1/2}-e^{-1})$ under a matroid constraint (their Appendix~E.2) and $0.491$ under a cardinality constraint, and \citet{qi2022} improves the latter to $0.478$ by replicating elements. \citet{vondrak2013} develops the general symmetry-gap framework for value queries. We add the smoothing and gradient-agreement steps needed for an oracle that returns full gradients, and construct the instances behind $\beta_*$ and the bounds for positive $\zeta$.

\paragraph{Rounding for set functions.}
\citet{chekuri2010swap} round fractional points of a matroid polytope by randomized swap rounding, which preserves the expected multilinear value of a submodular objective, and \citet{kulik2013knapsack} transfer a continuous approximation to a fixed number of knapsack constraints by enumeration and rounding of the residual instances. We use the uniform coefficient $4/9$ for knapsacks because the enumeration changes the residual capacities, so the residual relaxations need not contain the original diagonal point.

\section{Problem Definition and Background}\label{sec:background}
The objectives are nonnegative diminishing-returns functions on a cube, and the feasible region is a down-closed subset of that cube. The geometric assumption concerns only the feasible actions, while the objective oracles answer queries on the entire cube.

\subsection{Objectives and feasible sets}\label{sec:model-geometry}
For vectors, inequalities, products, divisions, minima, and scalar maps are interpreted coordinatewise. Let $Q=[0,1]^d$, let $e_i$ be its $i$th coordinate vector, and let $\one$ and $\zero$ denote the all-ones and all-zeros vectors. We use $[d]=\{1,\ldots,d\}$ and the Euclidean inner product and norm.

The objective class is defined by a differential form of diminishing returns (DR).
\begin{definition}[DR-submodularity]\label{def:dr}
A continuously differentiable function $f:Q\to\R$ is DR-submodular if
\begin{equation}\label{eq:dr}
 u\le w\quad\Longrightarrow\quad \nabla f(u)\ge\nabla f(w)
 \qquad(u,w\in Q).
\end{equation}
The gradients at the boundary are the continuous extensions of the interior gradients.
\end{definition}
Every objective in this paper is nonnegative on $Q$ but need not be monotone, so a partial derivative may have either sign. When $f$ is twice continuously differentiable, \eqref{eq:dr} is equivalent to every entry of its Hessian being nonpositive. By Lemma~\ref{lem:directional-dr}, the property implies concavity along every segment whose direction is coordinatewise nonnegative or coordinatewise nonpositive.

Let $K\subseteq Q$ be nonempty, compact, convex, and down-closed, meaning that $x\in K$ and $\zero\le y\le x$ imply $y\in K$; in particular, $\zero\in K$. A supplied number $\zeta\in[0,1]$ satisfies
\begin{equation}\label{eq:zeta-condition}
 \zeta\one\in K.
\end{equation}
We call $\zeta$ the \emph{diagonal level} and $\zeta\one$ the corresponding \emph{diagonal point}. The \emph{diagonal radius} of $K$ is the largest feasible level
\(\zeta(K):=\max\{s\in[0,1]:s\one\in K\}\), which exists because $K$ is compact and contains $\zero$. Down-closedness gives $s\one\in K$ for every $s\le\zeta(K)$, so \eqref{eq:zeta-condition} is equivalent to $\zeta\le\zeta(K)$, and the supplied level need not equal the radius. For example, when $K=\{x\in Q:Ax\le b\}$ with $A\ge0$, the condition is equivalent to $\zeta A\one\le b$; for the cardinality relaxation $\sum_i x_i\le k$, we may choose $\zeta=\min\{1,k/d\}$.

We write $D=\diam(K)$, which the algorithms may replace by any supplied upper bound such as $\sqrt d$. The algorithms access $K$ only through the Euclidean projection $\Pi_K:\R^d\to K$, and projections are counted separately from objective queries.

\subsection{Online and offline oracle models}\label{sec:model-oracles}
The online objective sequence $f_1,\ldots,f_T$ is chosen by an oblivious adversary, that is, before the learner's randomization. On round $t$, the learner first commits an action $p_t\in K$ and earns $f_t(p_t)$, after which it can query the current objective at points of $Q$. Every response on that round refers to $f_t$. We distinguish gradient responses from scalar value responses and state the number of queries of each type used by every implementation.

The regularity bounds used for algorithmic guarantees are
\begin{equation}\label{eq:regularity}
 0\le f_t(y)\le M,\qquad
 \|\nabla f_t(y)\|_2\le B
 \qquad(y\in Q).
\end{equation}
Here $M$ bounds the values and $B$ the gradients. A stochastic gradient response $V_t(y)$ satisfies
\begin{equation}\label{eq:grad-oracle}
 \E[V_t(y)\mid\mathcal F]=\nabla f_t(y),\qquad
 \E[\|V_t(y)\|_2^2\mid\mathcal F]\le B_1^2,
\end{equation}
where the information $\mathcal F$ includes the objective sequence, the current query point, and all choices and responses preceding this answer. Because $\mathcal F$ contains the earlier responses, the same conditions apply when several responses are requested on one round. When a noise bound is needed separately, we assume
\begin{equation}\label{eq:gradient-noise}
 \E[\|V_t(y)-\nabla f_t(y)\|_2^2\mid\mathcal F]\le\sigma_g^2;
 \qquad B_1^2=B^2+\sigma_g^2
\end{equation}
is then an admissible response bound, since the conditional mean of $V_t(y)-\nabla f_t(y)$ is zero and therefore
$\E[\|V_t(y)\|_2^2\mid\mathcal F]=\|\nabla f_t(y)\|_2^2+\E[\|V_t(y)-\nabla f_t(y)\|_2^2\mid\mathcal F]$.
Exact gradients correspond to $\sigma_g=0$.

A stochastic value response $Z_t(y)$ satisfies
\begin{equation}\label{eq:value-oracle}
 \begin{split}
 \E[Z_t(y)\mid\mathcal F]&=f_t(y),\qquad
 \E[(Z_t(y)-f_t(y))^2\mid\mathcal F]\le\sigma_v^2,\\
 \E[Z_t(y)^2\mid\mathcal F]&\le M_1^2,
 \qquad M_1^2=M^2+\sigma_v^2\ \text{is admissible},
 \end{split}
\end{equation}
where admissibility follows from the same variance decomposition. Exact values correspond to $\sigma_v=0$. Value feedback is a post-decision query to the current objective at a point of $Q$ that may differ from the rewarded action, as in the value-query model of \citet{aggarwal2026}.

For $\alpha\in[0,1]$, the expected static approximate regret is
\begin{equation}\label{eq:regret}
 \Reg_\alpha(T)=
 \alpha\max_{o\in K}\sum_{t=1}^T f_t(o)
 -\E\sum_{t=1}^T f_t(p_t).
\end{equation}
The expectation includes the learner's randomness and the oracle noise. The comparator $o$ is one feasible point used for the whole sequence; it exists because $K$ is compact and $\sum_tf_t$ is continuous, and it is deterministic because the adversary is oblivious.

In the offline model there is one objective $f$, and the algorithm queries it before returning a point $p\in K$. Write $\OPT=\max_{o\in K}f(o)$. An additive guarantee has the form $\E f(p)\ge\alpha\OPT-\epsilon$, and a multiplicative guarantee has the form $\E f(p)\ge(\alpha-\varepsilon)\OPT$. Conversions between the two use a supplied positive scale $v\le\OPT$, and we state the dependence on $v$. The upper bounds permit the stronger exact response $(f(y),\nabla f(y))$ at each query and count all such responses, so they apply to gradient and value-query algorithms alike.

\subsection{Upper-linearizable comparisons}\label{sec:model-linearizable}
An online algorithm can keep an internal state that differs from the action receiving reward. If the reward deficit is bounded by a term linear in the displacement from this state to the comparator, a linear online learner can select the states. We use the framework of \citet{pedramfar2024linear} in the following form.

\begin{definition}[Upper-linearizability]\label{def:linearizable}
A function class on $Q$, an action map $h:K\to K$, and a vector field $g(f,x)\in\R^d$ are upper-linearizable with parameters $\alpha,\beta>0$ if
\begin{equation}\label{eq:linearizable}
 \alpha f(o)-f(h(x))
 \le\beta\ip{g(f,x)}{o-x}
 \qquad(x,o\in K).
\end{equation}
The action map is independent of $f$, and the vector field is independent of $o$.
\end{definition}
Here $x$ is the learner's state, $h(x)$ is its feasible action, $g(f,x)$ is the vector supplied to the linear learner after feedback, $\alpha$ is the reward coefficient, and $\beta$ scales the linear regret. The field need not be the gradient of a potential. Concave maximization is the case $h(x)=x$, $\alpha=\beta=1$, and $g(f,x)=\nabla f(x)$, by the first-order inequality for concave functions.

We also use an objective-independent distribution of feasible actions $H(x)$. In that case the left side of \eqref{eq:linearizable} is $\alpha f(o)-\E_H f(H(x))$, with the expectation conditional on $x$, and the distribution is selected before the current objective is queried. The meta-algorithm OMBQ of \citet{pedramfar2024linear} runs a linear online learner on the states, plays the prescribed action, and passes an estimate of the field to the learner; Section~\ref{sec:online-transfer} gives the form of its guarantee used here. The algorithmic results therefore reduce to estimating the field with polynomially many queries.

\section{Structural Results}\label{sec:algebraic}
Nonnegative DR-submodular functions are upper-linearizable on every compact convex down-closed set containing $\zeta\one$. The action \(h_\zeta(x)\) and the field \(g_\zeta(f,x)\) are explicit: Lemma~\ref{lem:comparison} compares them with every fractional comparator in the cube, and Corollary~\ref{thm:geometry} uses the diagonal point to show that the action is feasible. Throughout this section, the objectives are nonnegative and continuously differentiable on the entire cube $[0,1]^d$, and their gradients satisfy \eqref{eq:dr}.

\subsection{General DR-submodular objectives}\label{sec:general-structural}
For DR-submodular objectives with no degree restriction, the action \(h_\zeta(x)\) depends only on the state \(x\) and the diagonal level \(\zeta\), and the field \(g_\zeta(f,x)\) averages partial derivatives along two sequences of coordinate segments. These segments are processed in an order chosen so that a weighted sum of function increments is bounded above even when individual increments are negative.

For $\zeta\in[0,1]$, define
\begin{equation}\label{eq:family}
 c_\zeta=\max\{1/2,1-\zeta\},\qquad
 [h_\zeta(x)]_i=\frac{x_i}{c_\zeta+x_i},\qquad
 \alow(\zeta)=\frac{4c_\zeta}{(2c_\zeta+1)^2}.
\end{equation}
Fix a state $x\in[0,1]^d$, and write
\begin{equation}\label{eq:coordinate-weights}
 p=h_\zeta(x),\qquad
 \mu_i=\frac{x_i(1-x_i)}{c_\zeta}\quad(i\in[d]).
\end{equation}
Let $\pi:[d]\to[d]$ be the permutation satisfying $\mu_{\pi(1)}\le\cdots\le\mu_{\pi(d)}$, with ties resolved by increasing original coordinate index.
Two chains of points, starting at $\zero$ and at $\one$, change one coordinate per step and meet at $p$. All vectors retain their original coordinate indices, and coordinate $\pi(k)$ is processed at step $k$.
Starting from $L^0=\zero$ and $U^0=\one$, define
\begin{equation}\label{eq:coordinate-states}
 L^k=L^{k-1}+p_{\pi(k)}e_{\pi(k)},\qquad
 U^k=U^{k-1}-(1-p_{\pi(k)})e_{\pi(k)}
 \quad(k\in[d]).
\end{equation}
Since $0\le p_j\le1/(1+c_\zeta)<1$, all these points lie in the cube. After step $k$, both points have coordinate $p_j$ for every $j\in\{\pi(1),\ldots,\pi(k)\}$, while every remaining coordinate is zero in $L^k$ and one in $U^k$; in particular, $L^d=U^d=p$. For $i=\pi(k)$ and $s\in[0,1]$, the two segments are
\begin{equation}\label{eq:paths}
 \ell_{i,s}=L^{k-1}+sp_i e_i,\qquad
 u_{i,s}=U^{k-1}-s(1-p_i)e_i.
\end{equation}
Define the averages $A_i$ and $C_i$ over $s$ and the field $g_\zeta(f,x)$ by
\begin{equation}\label{eq:averages}
 \begin{split}
 A_i&=\int_0^1\partial_i f(\ell_{i,s})\,ds,\qquad
 C_i=\int_0^1\partial_i f(u_{i,s})\,ds,\\
 [g_\zeta(f,x)]_i&=(1-p_i)A_i+p_iC_i\qquad(i\in[d]).
 \end{split}
\end{equation}
The entry computed at step $k$ is stored in coordinate $\pi(k)$ of $g_\zeta(f,x)$. The order and the segments depend only on $x$ and $\zeta$, and the averages remain defined for zero-length segments because they involve no division by $p_i$ or $1-p_i$.

The lower and upper chains $L^k$ and $U^k$ follow the comparison strategy of continuous double-greedy analysis \citep{bian2019}. In Lemma~\ref{lem:comparison} below, the action $h_\zeta(x)$ is prescribed by the state, the comparator $o$ is an arbitrary point of the cube, and the value at the origin $f(\zero)$ is retained. The proof uses nonnegativity of $f$ at the intermediate upper points $U^k$, which lie in the cube but need not lie in $K$.
\begin{lemma}[Comparison with every fractional comparator]\label{lem:comparison}\label{thm:algebraic}
For every $\zeta\in[0,1]$, every nonnegative continuously differentiable DR-submodular function $f:[0,1]^d\to\R$, and all $x,o\in[0,1]^d$,
\begin{equation}\label{eq:endpoint-comparison}
 f(o)+c_\zeta f(\zero)
 -\frac{1}{\alow(\zeta)}f(h_\zeta(x))
 \le\ip{g_\zeta(f,x)}{o-x}.
\end{equation}
For each fixed $f$, the field $x\mapsto g_\zeta(f,x)$ is Borel measurable.
\end{lemma}
\begin{proof}
Fix $x$ and $o$, and replace the coordinates of $o$ by those of $p=h_\zeta(x)$ in the order $\pi$. Let $o^0=o$ and, for $i=\pi(k)$, let
\[
 o^k=o^{k-1}+(p_i-o_i)e_i,\qquad
 v_{i,s}=o^{k-1}+s(p_i-o_i)e_i\quad(0\le s\le1),
\]
so that $o^d=p$ and $v_{i,s}$ traverses the segment from $o^{k-1}$ to $o^k$.
The coordinates $\pi(1),\ldots,\pi(k-1)$ of each of
$\ell_{i,s},v_{i,s},u_{i,s}$ equal the corresponding coordinates of $p$. In coordinate $i$, their values are
$sp_i$, $(1-s)o_i+sp_i$, and $(1-s)+sp_i$, respectively, which are ordered because $0\le o_i\le1$. Each remaining coordinate of $v_{i,s}$ belongs to $[0,1]$, while the corresponding coordinates of $\ell_{i,s}$ and $u_{i,s}$ are zero and one. Consequently, \eqref{eq:dr} gives
\begin{equation}\label{eq:comparator-path-order}
 \ell_{i,s}\le v_{i,s}\le u_{i,s},\qquad
 \partial_i f(\ell_{i,s})\ge
 \partial_i f(v_{i,s})\ge
 \partial_i f(u_{i,s}).
\end{equation}

To use these inequalities for either sign of $o_i-p_i$, write
$a_i=(1-p_i)o_i\ge0$ and $b_i=p_i(1-o_i)\ge0$, so that $o_i-p_i=a_i-b_i$. Multiplying the first inequality in \eqref{eq:comparator-path-order} by $a_i$ and the second by $b_i$ gives
\[
 (o_i-p_i)\partial_i f(v_{i,s})
 =a_i\partial_i f(v_{i,s})-b_i\partial_i f(v_{i,s})
 \le a_i\partial_i f(\ell_{i,s})-b_i\partial_i f(u_{i,s}).
\]
Lemma~\ref{lem:calculus}(i), applied to the segment from $o^k$ to $o^{k-1}$, and integration of the last display over $s\in[0,1]$ give
\begin{align}
 f(o^{k-1})-f(o^k)
 &=(o_i-p_i)\int_0^1\partial_i f(v_{i,s})\,ds\nonumber\\
 &\le (1-p_i)o_iA_i-p_i(1-o_i)C_i.
 \label{eq:coordinate-comparison}
\end{align}

Since $o^0=o$ and $o^d=p$, summing \eqref{eq:coordinate-comparison} over $k$ and regrouping with \eqref{eq:averages} gives
\[
 f(o)-f(p)
 \le\sum_{i=1}^d\left(o_i\bigl[(1-p_i)A_i+p_iC_i\bigr]-p_iC_i\right)
 =\ip{g_\zeta(f,x)}{o}-\sum_{i=1}^dp_iC_i.
\]
Adding and subtracting $\ip{g_\zeta(f,x)}{x}$, and using
$x_i[(1-p_i)A_i+p_iC_i]-p_iC_i=x_i(1-p_i)A_i-p_i(1-x_i)C_i$, we obtain
\begin{equation}\label{eq:intercept}
 \begin{split}
 f(o)\le{}&f(p)+\ip{g_\zeta(f,x)}{o-x}\\
 &+\sum_{i=1}^d
 \left[x_i(1-p_i)A_i-p_i(1-x_i)C_i\right].
 \end{split}
\end{equation}

The last sum telescopes along the chains in
\eqref{eq:coordinate-states}. Since $1-p_i=c_\zeta/(c_\zeta+x_i)$, the definitions of $p_i$ and $\mu_i$ give
\begin{equation}\label{eq:balance}
 \begin{split}
 x_i(1-p_i)
 &=\frac{c_\zeta x_i}{c_\zeta+x_i}=c_\zeta p_i,\\
 p_i(1-x_i)
 &=\frac{x_i(1-x_i)}{c_\zeta+x_i}=\mu_i(1-p_i).
 \end{split}
\end{equation}
For $i=\pi(k)$, the segments in \eqref{eq:paths} run from $L^{k-1}$ to $L^k$ and from $U^{k-1}$ to $U^k$, with displacements $p_ie_i$ and $-(1-p_i)e_i$. Lemma~\ref{lem:calculus}(i) therefore gives
\begin{equation}\label{eq:path-increments}
 p_iA_i=f(L^k)-f(L^{k-1}),\qquad
 -(1-p_i)C_i=f(U^k)-f(U^{k-1}).
\end{equation}
By \eqref{eq:balance}, the last sum in \eqref{eq:intercept} equals
$\sum_{k}\bigl[c_\zeta p_{\pi(k)}A_{\pi(k)}-\mu_{\pi(k)}(1-p_{\pi(k)})C_{\pi(k)}\bigr]$. Substituting \eqref{eq:path-increments} and telescoping the first part with $L^0=\zero$ and $L^d=p$ yields
\begin{align}
 &\sum_{i=1}^d
 \left[x_i(1-p_i)A_i-p_i(1-x_i)C_i\right]\nonumber\\
 &\qquad=c_\zeta\bigl[f(p)-f(\zero)\bigr]
 +\sum_{k=1}^d\mu_{\pi(k)}
       \bigl[f(U^k)-f(U^{k-1})\bigr].
 \label{eq:residual-telescope}
\end{align}

The weights in the remaining sum are not constant, so we collect the coefficient of each intermediate value $f(U^k)$. Lemma~\ref{lem:calculus}(ii) with $w_k=\mu_{\pi(k)}$ and $a_k=f(U^k)$ gives
\begin{align}
 &\sum_{k=1}^d\mu_{\pi(k)}
       \bigl[f(U^k)-f(U^{k-1})\bigr]\nonumber\\
 &\quad=\mu_{\pi(d)}f(U^d)-\mu_{\pi(1)}f(U^0)
 -\sum_{k=1}^{d-1}
     \bigl(\mu_{\pi(k+1)}-\mu_{\pi(k)}\bigr)f(U^k)
 \nonumber\\
 &\quad\le\mu_{\pi(d)}f(p)
 \le\frac{1}{4c_\zeta}f(p).
 \label{eq:upper-telescope}
\end{align}
The first inequality uses $f(U^k)\ge0$ for $0\le k\le d$, $\mu_{\pi(1)}\ge0$, and
$\mu_{\pi(k+1)}-\mu_{\pi(k)}\ge0$, which holds by the choice of $\pi$; it also uses $U^d=p$. The second uses
$x_i(1-x_i)=1/4-(x_i-1/2)^2\le1/4$ and $f(p)\ge0$.
When $d=1$, the sum over $k\le d-1$ is empty and the same inequalities apply.

Substituting \eqref{eq:residual-telescope} and
\eqref{eq:upper-telescope} into \eqref{eq:intercept} gives
\[
 f(o)+c_\zeta f(\zero)
 \le\left(1+c_\zeta+\frac{1}{4c_\zeta}\right)f(p)
     +\ip{g_\zeta(f,x)}{o-x}.
\]
Since
$(\alow(\zeta))^{-1}=(2c_\zeta+1)^2/(4c_\zeta)
=(4c_\zeta^2+4c_\zeta+1)/(4c_\zeta)=1+c_\zeta+1/(4c_\zeta)$, this is
\eqref{eq:endpoint-comparison}. Lemma~\ref{lem:calculus}(i) applies to segments on the boundary of the cube, and every displayed expression remains defined when a segment has length zero.

For measurability, partition the cube according to the permutation selected by the tie rule. The set of states selecting a given permutation $\sigma$ is defined by the inequalities $\mu_{\sigma(k)}\le\mu_{\sigma(k+1)}$, strict when $\sigma(k)>\sigma(k+1)$, between continuous functions of $x$, so it is Borel. For a fixed order, the points $\ell_{i,s}$ and $u_{i,s}$ are continuous functions of $(x,s)$, because $p=h_\zeta(x)$ is continuous. Since $\nabla f$ is uniformly continuous on the compact cube, the averages $A_i$ and $C_i$ are continuous in $x$ on each part. Thus each coordinate of $g_\zeta(f,x)$ is Borel measurable on the full cube.
\end{proof}

\begin{corollary}[Geometry-dependent upper-linearizability]\label{thm:geometry}
Let $K\subseteq[0,1]^d$ be nonempty, compact, convex, and down-closed, with $\zeta\one\in K$. Then $h_\zeta(K)\subseteq K$, and
\begin{equation}\label{eq:main-comparison}
 \alow(\zeta)f(o)-f(h_\zeta(x))
 \le\alow(\zeta)\ip{g_\zeta(f,x)}{o-x}
 \qquad(x,o\in K).
\end{equation}
In particular, Definition~\ref{def:linearizable} holds with
$\alpha=\beta=\alow(\zeta)$. This coefficient is nondecreasing and equals
\begin{equation}\label{eq:geometry-profile}
 \begin{aligned}
 \alow(\zeta)&=\begin{cases}
 \displaystyle\frac{4(1-\zeta)}{(3-2\zeta)^2},&0\le\zeta\le1/2,\\[4pt]
 1/2,&1/2\le\zeta\le1,
 \end{cases}\\[4pt]
 &=\frac12-\frac{(1-2\zeta)_+^2}{2(3-2\zeta)^2}.
 \end{aligned}
\end{equation}
Consequently,
\begin{equation}\label{eq:geometry-special}
 \alow(0)=\frac49,\qquad
 \alow(\zeta)=\frac12\quad(\zeta\ge1/2).
\end{equation}
\end{corollary}
\begin{proof}
By \eqref{eq:family}, $1/2\le c_\zeta\le1$ and $c_\zeta\ge1-\zeta$, so the vector
$q=(1-c_\zeta)\one$ satisfies $\zero\le q\le\zeta\one$, and down-closedness gives $q\in K$. For each $x\in K$, convexity with weight $c_\zeta\in[1/2,1]$ therefore gives
\[
 v=c_\zeta x+(1-c_\zeta)q
   =c_\zeta x+(1-c_\zeta)^2\one\in K.
\]
For every scalar $t\in[0,1]$,
\[
 \bigl(c_\zeta t+(1-c_\zeta)^2\bigr)(c_\zeta+t)-t
 =c_\zeta t^2-2c_\zeta(1-c_\zeta)t+c_\zeta(1-c_\zeta)^2
 =c_\zeta(t+c_\zeta-1)^2,
\]
and division by $c_\zeta+t>0$ gives
\begin{equation}\label{eq:tangent-feasibility}
 c_\zeta t+(1-c_\zeta)^2-\frac{t}{c_\zeta+t}
 =\frac{c_\zeta(t+c_\zeta-1)^2}{c_\zeta+t}\ge0.
\end{equation}
Applying \eqref{eq:tangent-feasibility} with $t=x_i$ in every coordinate yields
$\zero\le h_\zeta(x)\le v$, and a second application of down-closedness proves $h_\zeta(x)\in K$.

Multiply \eqref{eq:endpoint-comparison} by $\alow(\zeta)>0$ and move the origin term to the right. Since $f(\zero)\ge0$,
\[
 \alow(\zeta)f(o)-f(h_\zeta(x))
 \le\alow(\zeta)\ip{g_\zeta(f,x)}{o-x}
     -c_\zeta\alow(\zeta)f(\zero)
 \le\alow(\zeta)\ip{g_\zeta(f,x)}{o-x},
\]
which is \eqref{eq:main-comparison}.

If $\zeta\le1/2$, then $c_\zeta=1-\zeta$ and $2c_\zeta+1=3-2\zeta$; if $\zeta\ge1/2$, then $c_\zeta=1/2$ and $4c_\zeta/(2c_\zeta+1)^2=1/2$.
These substitutions into \eqref{eq:family} give the two branches in \eqref{eq:geometry-profile}.
For the final expression, expand
\[
 \frac12-\frac{4c_\zeta}{(2c_\zeta+1)^2}
 =\frac{(2c_\zeta+1)^2-8c_\zeta}{2(2c_\zeta+1)^2}
 =\frac{(2c_\zeta-1)^2}{2(2c_\zeta+1)^2},
\]
where $2c_\zeta-1=(1-2\zeta)_+$ and $2c_\zeta+1=3-2\zeta$ whenever $\zeta\le1/2$, while the numerator vanishes for $\zeta\ge1/2$. On $0\le\zeta<1/2$, the quotient rule gives
\[
 \frac{d}{d\zeta}\,\frac{4(1-\zeta)}{(3-2\zeta)^2}
 =\frac{-4(3-2\zeta)+16(1-\zeta)}{(3-2\zeta)^3}
 =\frac{4(1-2\zeta)}{(3-2\zeta)^3}\ge0,
\]
and on $[1/2,1]$ the coefficient is constant. The two branches agree at $1/2$, so the coefficient is nondecreasing on $[0,1]$. Substituting $\zeta=0$ gives $4/9$, the first value in \eqref{eq:geometry-special}.
\end{proof}

\begin{remark}[Origin of the factor $4/9$]\label{rem:four-ninths-proof}
At $\zeta=0$ we have $c_0=1$, so the lower-chain contribution in
\eqref{eq:residual-telescope} equals $f(p)-f(\zero)$, and
\eqref{eq:upper-telescope} bounds the upper-chain contribution by
$f(p)/4$. Together with the term $f(p)$ in
\eqref{eq:intercept}, the coefficient of $f(p)$ is
$1+1+1/4=9/4$, and multiplication by $4/9$ gives the comparison.
\end{remark}

\subsection{Quadratic and cubic objectives}\label{sec:polynomial-structure}
Polynomial objectives of fixed degree admit a stronger comparison. We write the objective as a nonnegative combination of constants, unary terms $z_i$ and $1-z_i$, directed cuts $z_i(1-z_j)$, and, for cubics, two kinds of three-variable terms, and we replace each term by a concave majorant. A scalar map applied to every coordinate retains a fixed fraction of this majorant, and a supergradient of the majorant is the comparator-independent field. The representation builds on the quadratic graph-cut construction of \citet{kolmogorov2004} and the cubic reduction of \citet{billionnet1985}, and the approximation coefficients come from inequalities for the individual terms.

For $r\in[0,1/2]$, define
\begin{equation}\label{eq:cubic-analytic-coefficient}
 b_r=\sqrt{3/4-r},\qquad
 A_3(r)=\frac{4b_r}{(2b_r+1)^2}
       =\frac{2\sqrt{3-4r}}{(1+\sqrt{3-4r})^2}.
\end{equation}

\begin{theorem}[Polynomial comparison]\label{thm:polynomial}
Let $m\in\{2,3\}$ and let $f$ be a nonnegative DR-submodular polynomial of total degree at most $m$ on $[0,1]^d$. Let $K\subseteq[0,1]^d$ be nonempty, compact, convex, and down-closed, with $\zeta\one\in K$. There is an action $p_\zeta(x)$, whose distribution depends only on $x$ and $\zeta$, and a field $s_f(x)$ satisfying
\begin{equation}\label{eq:polynomial-comparison}
 \alpha_m(\zeta)f(o)-\E f(p_\zeta(x))
 \le \alpha_m(\zeta)\ip{s_f(x)}{o-x}
 \qquad(x,o\in K),
\end{equation}
where $p_\zeta(x)\in K$ almost surely,
\begin{equation}\label{eq:cubic-geometric-profile}
 \alpha_2(\zeta)=\frac12,\qquad
 \alpha_3(\zeta)=
 \begin{cases}
 \displaystyle\max\left\{\frac8{17},A_3(\zeta)\right\},&0\le\zeta\le1/2,\\[2pt]
 1/2,&1/2\le\zeta\le1.
 \end{cases}
\end{equation}
For quadratics, the action is the coordinatewise minimum $p_\zeta(x)=x\wedge\tfrac12\one$. For cubics, the action either applies one of two scalar maps, chosen at random, to every coordinate (coefficient $8/17$), or applies one deterministic scalar map that depends on $\zeta$ (coefficient $A_3(\min\{\zeta,1/2\})$). The field is invariant under addition of a constant to $f$, is computable by polynomially many arithmetic operations from the nonconstant coefficients of $f$, and satisfies
\begin{equation}\label{eq:polynomial-field-bound}
 \norm{s_f(x)}_2
 \le\Lambda_m\left(f(\tfrac12\one)-\min_{[0,1]^d}f\right)
 \le\frac{\Lambda_m}{2}B\sqrt d,
 \qquad \Lambda_2=4,\quad\Lambda_3=\frac{16}{3},
\end{equation}
where the last inequality assumes $\norm{\nabla f(y)}_2\le B$ on the cube. Monomials with repeated factors, such as $x_i^2$ and $x_i^2x_j$, are allowed.
\end{theorem}
\begin{proof}
See Appendix~\ref{app:polynomial}. Appendix~\ref{app:quadratic} treats $m=2$. For $m=3$, Lemma~\ref{lem:cubic-polarization} reduces the objective to a multilinear one, Lemma~\ref{lem:cubic-multilinear} decomposes it into nonnegative components, Lemmas~\ref{lem:cubic-rounding} and~\ref{lem:cubic-geometry} give the component inequalities for the two actions, and Appendix~\ref{app:cubic-assembly} constructs the field and proves \eqref{eq:polynomial-comparison} and \eqref{eq:polynomial-field-bound}.
\end{proof}
With exact feedback, the field is computed by reconstructing the nonconstant coefficients from $d+1$ gradient queries for $m=2$, $\binom{d+2}{2}$ gradient queries for $m=3$, or $\binom{d+m}{m}$ value queries per round (Proposition~\ref{prop:polynomial-oracle}); Corollary~\ref{cor:polynomial-one-value} gives a one-value implementation with $O(T^{3/4})$ regret.

\begin{remark}\label{rem:cubic-history}
The first arXiv version of this paper \citep{aggarwal2026geometryv1} established the online cubic coefficient $8/17$ with a computer-assisted exact rational certificate. Concurrent work of \citet[Section~4.6 and Theorem~4.5]{du2026roofs} independently obtains an offline $8/17$ approximation and gives the two scalar maps, denoted $A$ and $B$ in \eqref{eq:cubic-baseline-maps}. Appendix~\ref{app:cubic-action} proves their component inequalities analytically. Combined with our nonnegative decomposition and polarization, these inequalities give the comparison of Theorem~\ref{thm:polynomial}. For the online algorithm, the random choice between $A$ and $B$ is made before the current objective is observed, and the field is computed from the coefficients of the objective. The online results of \citet{du2026roofs} concern contention resolution and prophet settings, whereas here the comparison is upper-linearizable for every comparator and yields regret guarantees in the adversarial online model of Section~\ref{sec:model-oracles}.
\end{remark}

\begin{remark}[Strict improvement from cubic structure]\label{rem:cubic-improvement}
For $0\le\zeta<1/2$, $A_3(\zeta)>\alow(\zeta)$. Moreover,
\begin{equation}\label{eq:cubic-crossing}
 A_3(\zeta)>\frac8{17}
 \quad\Longleftrightarrow\quad
 \zeta>\frac{47-9\sqrt{17}}{128}
 \qquad(0\le\zeta<1/2).
\end{equation}
In particular, $\alpha_3(0)=8/17$, and $\alpha_3(\zeta)=A_3(\zeta)$ above this crossing.
\end{remark}
\begin{proof}
Let $b=\sqrt{3/4-\zeta}$ and $c=1-\zeta$. Then for $0\le\zeta<1/2$, we have $b,c>1/2$ and $c^2-b^2=(1/2-\zeta)^2>0$, thus $b<c$. In addition, $a(t)=4t/(2t+1)^2$ has derivative $4(1-2t)/(2t+1)^3<0$ for $t>1/2$. By \eqref{eq:family} with $c_\zeta=1-\zeta$, we have $\alow(\zeta)=a(c)$, and since $A_3(\zeta)=a(b)$, we get $A_3(\zeta)>\alow(\zeta)$. Since $a$ is decreasing on $(1/2,\infty)$, $A_3(\zeta)>8/17$ holds exactly when $b$ is smaller than the root of $a(t)=8/17$ in $(1/2,\sqrt3/2]$. Clearing denominators, $a(t)=8/17$ becomes $8t^2-9t+2=0$, whose roots are $(9\pm\sqrt{17})/16$, of which only $t_*=(9+\sqrt{17})/16$ lies in $(1/2,\sqrt3/2]$. Finally, $b<t_*$ is equivalent to
\[
 \zeta>\frac34-t_*^2=\frac34-\frac{98+18\sqrt{17}}{256}
 =\frac{47-9\sqrt{17}}{128},
\]
which is \eqref{eq:cubic-crossing}. Since $(47-9\sqrt{17})/128>0$, \eqref{eq:cubic-geometric-profile} gives $\alpha_3(0)=8/17$, and $\alpha_3(\zeta)=A_3(\zeta)>8/17$ for $\zeta$ above this threshold.
\end{proof}

\section{Online Achievability}\label{sec:feedback}
The OMBQ reduction of \citet[Theorem~1]{pedramfar2024linear} turns an upper-linearizable comparison into an online algorithm once the field can be computed or estimated in polynomial time. For the comparison of Section~\ref{sec:algebraic}, the field can be estimated from one gradient or from one or two function values per round, with explicit second-moment bounds. Each round queries the current objective only after the rewarded action has been selected.

\subsection{A single gradient query for general objectives}\label{sec:first-order}
Let $x\in K$, let $p=h_\zeta(x)$ as defined in \eqref{eq:coordinate-weights}, and use the order $\pi$ and the segments $\ell_{i,s},u_{i,s}$ of \eqref{eq:paths}, indexed by original coordinates. Draw a coordinate $I$ uniformly from $[d]$, a number $S$ uniformly from $[0,1]$, and a branch with probabilities $1-p_I$ and $p_I$, and set
\begin{equation}\label{eq:query-law}
 Y=
 \begin{cases}
  \ell_{I,S},&\text{with probability }1-p_I,\\
  u_{I,S},&\text{with probability }p_I,
 \end{cases}
 \qquad
 \text{ and } \widehat g=d[V(Y)]_I e_I ,
\end{equation}
where $V(Y)$ is the gradient response at $Y$, equal to $\nabla f(Y)$ for exact gradients. The query point lies in $Q$, and if $p_I=0$, the lower segment is constant and the upper branch has probability zero.

\begin{lemma}[Unbiased one-gradient query]\label{lem:sampler}
Conditional on the current objective $f$ and state $x$, the estimate $\widehat g$ in \eqref{eq:query-law} obeys
\begin{equation}\label{eq:sample-moments}
 \begin{gathered}
 \E\widehat g=g_\zeta(f,x),\qquad
 \E\|\widehat g\|_2^2\le V_{\rm gen},\\
 V_{\rm gen}:=
 \begin{cases}
 2dB^2,&\text{exact gradients},\\
 d^2B_1^2,&\text{gradients satisfying \eqref{eq:grad-oracle}},\\
 \min\{d^2B_1^2,\,2dB^2+d^2\sigma_g^2\},&\text{gradients satisfying \eqref{eq:grad-oracle} and \eqref{eq:gradient-noise}}.
 \end{cases}
 \end{gathered}
\end{equation}
It requires one gradient response, $O(d\log d)$ arithmetic for the order, and $O(d)$ arithmetic to form the query point.
\end{lemma}
\begin{proof}
Fix $i\in[d]$. Only the event $\{I=i\}$, of probability $1/d$, contributes to $\widehat g_i$. Conditional on $I=i$, $S$, and the branch, the point $Y$ is determined, and \eqref{eq:grad-oracle} gives $\E[[V(Y)]_i\mid I,S,Y]=\partial_i f(Y)$. By \eqref{eq:averages}, averaging over the two branches and then over $S$ gives
\begin{align}
 \E[\widehat g_i]
 &=d\cdot\frac1d\Bigl((1-p_i)\int_0^1\partial_i f(\ell_{i,s})\,ds
       +p_i\int_0^1\partial_i f(u_{i,s})\,ds\Bigr)
 =[g_\zeta(f,x)]_i. \label{eq:sampler-unbiased-proof}
\end{align}
Since $\widehat g$ has the single nonzero coordinate $I$, $\|\widehat g\|_2^2=d^2[V(Y)]_I^2\le d^2\|V(Y)\|_2^2$. For stochastic responses, \eqref{eq:grad-oracle} therefore gives $\E\|\widehat g\|_2^2\le d^2 B_1^2$.

For an exact gradient, let $m_i(f)=\max\{|\partial_i f(\zero)|,|\partial_i f(\one)|\}$ as in \eqref{eq:coordinate-envelope}. Lemma~\ref{lem:gradient-envelope} gives $|\partial_i f(y)|\le m_i(f)$ for every $y\in Q$, so every point used when $I=i$ satisfies $|\partial_i f(Y)|^2\le m_i(f)^2$. Averaging over $I$ and using \eqref{eq:envelope-sum},
\begin{equation}\label{eq:sampler-exact-proof}
 \E\|\widehat g\|_2^2
 =d^2\cdot\frac1d\sum_{i=1}^d\E\bigl[\partial_i f(Y)^2\mid I=i\bigr]
 \le d\sum_{i=1}^d m_i(f)^2
 \le2dB^2.
\end{equation}
Under \eqref{eq:gradient-noise}, put $\xi(Y)=V(Y)-\nabla f(Y)$. The choices of $I$, $S$, and the branch precede the response, so conditional unbiasedness removes the cross term and gives
\[
 \E\bigl[[V(Y)]_i^2\mid I=i,Y\bigr]
 =\partial_if(Y)^2+\E\bigl[\xi_i(Y)^2\mid I=i,Y\bigr]
 \le m_i(f)^2+\sigma_g^2,
\]
since $\xi_i(Y)^2\le\|\xi(Y)\|_2^2$. Averaging over $I$ as in \eqref{eq:sampler-exact-proof} gives
\begin{equation}\label{eq:sampler-noise-proof}
 \E\|\widehat g\|_2^2
 \le d\sum_{i=1}^d\bigl(m_i(f)^2+\sigma_g^2\bigr)
 \le2dB^2+d^2\sigma_g^2 .
\end{equation}
Both stochastic bounds hold under \eqref{eq:grad-oracle} and \eqref{eq:gradient-noise}, so the smaller one does.
For the cost, $p$ and the weights $\mu_i$ take $O(d)$ arithmetic and sorting them takes $O(d\log d)$. Given the order, the point $Y$ has coordinate $p_j$ in each coordinate $j$ processed before $I$, coordinate $Sp_I$ or $1-S(1-p_I)$ in coordinate $I$, and coordinate $0$ or $1$ elsewhere, so it takes $O(d)$ arithmetic and no integral is evaluated.
\end{proof}

\subsection{Applying the online reduction}\label{sec:online-transfer}
OMBQ keeps a state $x_t\in K$ updated by a linear online learner, plays the action prescribed by the comparison, and returns the query procedure's vector $\widehat g_t$ to the learner. We use Euclidean projected ascent \citep{zinkevich2003}, $x_{t+1}=\Pi_K(x_t+\eta\widehat g_t)$ with $x_1\in K$. For every $o\in K$, nonexpansiveness of $\Pi_K$ gives
$\norm{x_{t+1}-o}_2^2\le\norm{x_t-o}_2^2-2\eta\ip{\widehat g_t}{o-x_t}+\eta^2\norm{\widehat g_t}_2^2$. Summing over $t$ and using $\norm{x_1-o}_2\le D$ gives
\begin{equation}\label{eq:projected-ascent-regret}
 \sum_{t=1}^T\ip{\widehat g_t}{o-x_t}
 \le\frac{D^2}{2\eta}+\frac\eta2\sum_{t=1}^T\norm{\widehat g_t}_2^2.
\end{equation}
If $\E\norm{\widehat g_t}_2^2\le V$ for every $t$ and $\eta=D/\sqrt{VT}$, the expectation of the right side is at most $D\sqrt{VT}$.

\citet[Theorem~1]{pedramfar2024linear} state the stochastic case of OMBQ for bounded query outputs, and the projected ascent algorithm uses that bound only through the second moments in \eqref{eq:projected-ascent-regret}. We use the reduction in the following form, which also covers a randomized action and an additive comparison error. Let $\mathcal F_t$ contain the objective sequence and all random choices before round $t$, so that $x_t$ is $\mathcal F_t$-measurable. For a randomized action $H$, the comparison is applied conditionally on the state,
\begin{equation}\label{eq:randomized-ombq}
 \alpha f_t(o)-\E[f_t(H(x_t))\mid x_t]
 \le\beta\ip{g(f_t,x_t)}{o-x_t}.
\end{equation}
Suppose more generally that the played action $p_t$ and the vector $\widehat g_t$ satisfy, for every $o\in K$,
\begin{equation}\label{eq:ombq-conditions}
 \alpha f_t(o)-\E[f_t(p_t)\mid\mathcal F_t]
 \le\beta\,\E[\ip{\widehat g_t}{o-x_t}\mid\mathcal F_t]+e_t,
 \qquad
 \E[\norm{\widehat g_t}_2^2\mid\mathcal F_t]\le V,
\end{equation}
with deterministic $e_t$. If $\widehat g_t$ is conditionally unbiased for $g(f_t,x_t)$ given $\mathcal F_t$ and $p_t$ is drawn from the law of $H(x_t)$ independently of $\mathcal F_t$ given $x_t$, then \eqref{eq:randomized-ombq} implies the first condition with $e_t=0$, because $o-x_t$ is $\mathcal F_t$-measurable. Take $o$ to be a deterministic maximizer in \eqref{eq:regret}. Summing \eqref{eq:ombq-conditions} over $t$, taking expectations, and applying \eqref{eq:projected-ascent-regret} gives
\begin{equation}\label{eq:ombq-transfer}
 \Reg_\alpha(T)\le\beta D\sqrt{VT}+\sum_{t=1}^Te_t.
\end{equation}
If $V=0$, then $\widehat g_t=\zero$ almost surely, and a fixed state gives \eqref{eq:ombq-transfer} with zero linear regret.

Specialized to the comparison of Section~\ref{sec:general-structural}, the reduction is the following algorithm.
\begin{algorithm}[H]
\caption{OMBQ with one gradient query for DR-submodular objectives}\label{alg:gradient}
\smallskip
\textbf{Input:} $K$, $\zeta\one\in K$, horizon $T$, diameter bound $D>0$, moment bound $V_{\rm gen}>0$.

\textbf{Initialize:} $x_1=\zero$ and $\eta=D/\sqrt{T V_{\rm gen}}$.

\textbf{For $t=1,\ldots,T$:}
\begin{enumerate}[leftmargin=2em,itemsep=2pt,topsep=3pt]
 \item Commit $p_t=h_\zeta(x_t)$ for reward.
 \item Form the coordinate order and draw $I_t,S_t,Y_t$ as in \eqref{eq:query-law}.
 \item Query $V_t(Y_t)$ and set $\widehat g_t=d[V_t(Y_t)]_{I_t}e_{I_t}$.
 \item Set $x_{t+1}=\Pi_K(x_t+\eta\widehat g_t)$.
\end{enumerate}
\end{algorithm}

The guarantee retains the origin term of Lemma~\ref{lem:comparison}.
\begin{corollary}[General first-order guarantee]\label{thm:first-order}
Algorithm~\ref{alg:gradient} uses one gradient query and one projection per round. Its expected regret satisfies
\begin{equation}\label{eq:first-order-rates}
 \Reg_{\alow(\zeta)}(T)
 +\alow(\zeta)c_\zeta\sum_{t=1}^T f_t(\zero)
 \le
 \begin{cases}
  \alow(\zeta)DB\sqrt{2dT},&\text{exact gradients},\\
  \alow(\zeta)DdB_1\sqrt T,&\text{stochastic gradients},\\
  \alow(\zeta)D\sqrt{(2dB^2+d^2\sigma_g^2)T},&\text{stochastic gradients with \eqref{eq:gradient-noise}}.
 \end{cases}
\end{equation}
Under \eqref{eq:gradient-noise}, the second and third bounds both hold when $V_{\rm gen}$ is the minimum in \eqref{eq:sample-moments}.
\end{corollary}
\begin{proof}
Write $\alpha=\alow(\zeta)$. Every state lies in $K$, since $x_1=\zero\in K$ and each update is a projection onto $K$, so Corollary~\ref{thm:geometry} gives $p_t\in K$. Multiplying \eqref{eq:endpoint-comparison} for $f_t$ by $\alpha$ gives
\[
 \alpha f_t(o)-f_t(p_t)
 \le\alpha\ip{g_\zeta(f_t,x_t)}{o-x_t}-\alpha c_\zeta f_t(\zero).
\]
By Lemma~\ref{lem:sampler}, $\widehat g_t$ is conditionally unbiased for $g_\zeta(f_t,x_t)$ given $\mathcal F_t$, with conditional second moment at most $V_{\rm gen}$. Hence \eqref{eq:ombq-conditions} holds with $\beta=\alpha$, $V=V_{\rm gen}$, and $e_t=-\alpha c_\zeta f_t(\zero)$, and \eqref{eq:ombq-transfer} gives
$\Reg_\alpha(T)+\alpha c_\zeta\sum_tf_t(\zero)\le\alpha D\sqrt{TV_{\rm gen}}$. Substituting $V_{\rm gen}=2dB^2$, $V_{\rm gen}=d^2B_1^2$, or $V_{\rm gen}=\min\{d^2B_1^2,2dB^2+d^2\sigma_g^2\}$ gives the three cases of \eqref{eq:first-order-rates}.
\end{proof}
The term involving $f_t(\zero)$ appears in the bound but is never queried. If $D=0$ or $V_{\rm gen}=0$, a fixed state gives zero linear regret and the same bound.

\subsection{Value feedback on the cube}\label{sec:value-feedback}
Function values suffice for the same geometry-dependent coefficient. We adapt the zeroth-order constructions FOTZO and FOTZO-2P of \citet[Theorems~5--6]{pedramfar2024linear}, which estimate gradients of a smoothed objective from values, by contracting the query domain into the cube before smoothing. Uniform smoothing over a small box gives a selected partial derivative as a difference of values on two opposite faces. The smoothed objective remains DR-submodular, and the gradient bound $B$ controls its distance from the original objective.

For $0<\rho<1/2$, put $s_\rho=1-2\rho$ and define
\begin{equation}\label{eq:value-smoothed-objective}
 F_\rho(y)=\E_{U\sim\operatorname{Unif}([-1,1]^d)}
       f(s_\rho y+\rho\one+\rho U),\qquad y\in Q.
\end{equation}
Each coordinate of the argument of $f$ lies in $[\rho-\rho,\,s_\rho+\rho+\rho]=[0,1]$, so every argument belongs to $Q$. The learner still plays $h_\zeta(x)$; smoothing changes only the function whose field is estimated. A value response $Z(y)$ obeys \eqref{eq:value-oracle}, and exact values are the special case $Z(y)=f(y)$.

\begin{lemma}[Value estimates of a smoothed field]\label{lem:value-smoothing}
The function $F_\rho$ is nonnegative, continuously differentiable, and DR-submodular on $Q$, with
\begin{equation}\label{eq:value-smoothing-bias}
 \norm{\nabla F_\rho(y)}_2\le B,
 \qquad \sup_{y\in Q}|F_\rho(y)-f(y)|\le \frac{2}{\sqrt3}B\rho\sqrt d.
\end{equation}
Draw $I,S,Y$ and the branch as in \eqref{eq:query-law}. Independently draw a vector $U^{(-I)}$ whose $I$th coordinate is zero and whose other coordinates are independent and uniform on $[-1,1]$. Put
\begin{equation}\label{eq:value-query-points}
 W^\pm=s_\rho Y+\rho\one+\rho U^{(-I)}\pm\rho e_I.
\end{equation}
With exact values, use two queries and return
\begin{equation}\label{eq:value-two-point-estimator}
 \widehat g^{(2)}=
 \frac{d s_\rho}{2\rho}\bigl(f(W^+)-f(W^-)\bigr)e_I.
\end{equation}
With stochastic values, draw an independent uniform sign $E\in\{-1,1\}$, query only $W^E$, and return
\begin{equation}\label{eq:value-one-point-estimator}
 \widehat g^{(1)}=\frac{d s_\rho}{\rho}E Z(W^E)e_I.
\end{equation}
Both estimators are conditionally unbiased for $g_\zeta(F_\rho,x)$. Their conditional second moments satisfy
\begin{equation}\label{eq:value-estimator-moments}
 \E\norm{\widehat g^{(2)}}_2^2\le2dB^2,
 \qquad
 \E\norm{\widehat g^{(1)}}_2^2\le\frac{d^2M_1^2}{\rho^2}.
\end{equation}
All query points lie in $Q$. Apart from the values and the projection, each round uses $O(d\log d)$ arithmetic.
\end{lemma}
\begin{proof}
Write $\phi_\rho(y,U)=s_\rho y+\rho\one+\rho U$. Since $\nabla f$ is continuous and bounded on the compact cube, differentiation under the expectation is justified by dominated convergence and gives
\begin{equation}\label{eq:value-smoothed-gradient}
 \nabla F_\rho(y)=s_\rho\E_U\nabla f(\phi_\rho(y,U)),
\end{equation}
which is continuous in $y$ by uniform continuity of $\nabla f$. If $y\le y'$, then $\phi_\rho(y,U)\le\phi_\rho(y',U)$ for every $U$ because $s_\rho>0$, so \eqref{eq:dr} for $f$ and \eqref{eq:value-smoothed-gradient} give $\nabla F_\rho(y)\ge\nabla F_\rho(y')$. Nonnegativity passes through the expectation, and $\norm{\nabla F_\rho(y)}_2\le s_\rho\E_U\norm{\nabla f(\phi_\rho(y,U))}_2\le B$. For the value bias, $\phi_\rho(y,U)-y=\rho(\one-2y+U)$, and Lemma~\ref{lem:calculus}(i) on the segment from $y$ to $\phi_\rho(y,U)$, which lies in the convex set $Q$, gives $|f(\phi_\rho(y,U))-f(y)|\le B\rho\norm{\one-2y+U}_2$. The coordinates of $U$ are independent with $\E U_i=0$ and $\E U_i^2=1/3$, and $|1-2y_i|\le1$, so averaging over $U$ and applying Jensen's inequality give
\[
 |F_\rho(y)-f(y)|
 \le B\rho\,\E\norm{\one-2y+U}_2
 \le B\rho\sqrt{\norm{\one-2y}_2^2+d/3}
 \le\frac{2}{\sqrt3}B\rho\sqrt d,
\]
which proves \eqref{eq:value-smoothing-bias}.

For fixed $i$ and $y$, write $w=s_\rho y+\rho\one+\rho U^{(-i)}$, so that $\phi_\rho(y,U)=w+\rho U_ie_i$. By Fubini's theorem, the $i$th coordinate of \eqref{eq:value-smoothed-gradient} can be integrated first over $U_i$, and Lemma~\ref{lem:calculus}(i) gives
\[
 \frac12\int_{-1}^1\partial_if(w+\rho ue_i)\,du
 =\frac1{2\rho}\bigl[f(w+\rho e_i)-f(w-\rho e_i)\bigr].
\]
Therefore
\begin{equation}\label{eq:value-face-identity}
 \partial_i F_\rho(y)=\frac{s_\rho}{2\rho}
 \E_{U^{(-i)}}\bigl[f(w+\rho e_i)-f(w-\rho e_i)\bigr].
\end{equation}
Conditional on $I=i$ and $Y$, the two query points in \eqref{eq:value-query-points} are $w\pm\rho e_i$ with $y=Y$, so averaging \eqref{eq:value-two-point-estimator} over $U^{(-I)}$ gives $d\,\partial_iF_\rho(Y)e_i$. Averaging over the coordinate, segment, and branch as in \eqref{eq:sampler-unbiased-proof}, with $F_\rho$ in place of $f$, proves $\E\widehat g^{(2)}=g_\zeta(F_\rho,x)$; the segments are the same because they depend only on $x$ and $\zeta$. For \eqref{eq:value-one-point-estimator}, conditional unbiasedness of $Z$ and the independent uniform sign give
\[
 \E\bigl[EZ(W^E)\mid I,Y,U^{(-I)}\bigr]
 =\tfrac12\bigl[f(W^+)-f(W^-)\bigr],
\]
so $\widehat g^{(1)}$ has the same conditional mean as $\widehat g^{(2)}$.

For the exact estimator, Lemma~\ref{lem:calculus}(i) on the segment from $W^-$ to $W^+$ gives
$\frac{s_\rho}{2\rho}\bigl(f(W^+)-f(W^-)\bigr)=s_\rho\int_0^1\partial_If(W^-+2\rho re_I)\,dr$, an average of $\partial_If$ over points of the cube multiplied by $s_\rho$. By Lemma~\ref{lem:gradient-envelope}, its absolute value is at most $s_\rho m_I(f)$. As in \eqref{eq:sampler-exact-proof}, averaging over $I$ and using \eqref{eq:envelope-sum} and $s_\rho\le1$ gives
\[
 \E\norm{\widehat g^{(2)}}_2^2
 \le d^2\cdot\frac1d\sum_{i=1}^ds_\rho^2m_i(f)^2\le2dB^2.
\]
For the one-query estimator, $E^2=1$, and the response bound in \eqref{eq:value-oracle} gives
\[
 \E\norm{\widehat g^{(1)}}_2^2
 =\frac{d^2s_\rho^2}{\rho^2}\E[Z(W^E)^2]
 \le\frac{d^2M_1^2}{\rho^2}.
\]
Since $s_\rho Y+\rho\one\in[\rho,1-\rho]^d$, each coordinate of $W^\pm$ lies in $[0,1]$, so both points lie in $Q$. After the coordinate order is formed, constructing $W^\pm$ and sampling $U^{(-I)}$ and $E$ take $O(d)$ arithmetic, and no integral is evaluated.
\end{proof}

The uniform value error in \eqref{eq:value-smoothing-bias} transfers the comparison for $F_\rho$ to the original rewards at an additive cost per round.
\begin{corollary}[Online value-feedback guarantee]\label{cor:value-feedback}
Let $\alpha=\alow(\zeta)$ and define
\begin{equation}\label{eq:value-online-constants}
 K_0=\frac{1+\alpha}{2\sqrt3}B\sqrt d,\qquad
 K_1=4\alpha DdM_1+K_0,\qquad
 K_2=\alpha DB\sqrt{2d}.
\end{equation}
Replace the gradient response in Algorithm~\ref{alg:gradient} by either estimator in Lemma~\ref{lem:value-smoothing}, and use the step size $D/\sqrt{TV_\rho}$, where $V_\rho$ is its bound in \eqref{eq:value-estimator-moments}; use a fixed state if $V_\rho=0$. Then
\begin{equation}\label{eq:value-online-general}
 \Reg_\alpha(T)\le
 \alpha D\sqrt{T V_\rho}
 +\frac{2}{\sqrt3}(1+\alpha)B\rho\sqrt d\,T.
\end{equation}
In particular, two exact value queries per round with $\rho=1/(4T)$ give
\begin{equation}\label{eq:value-online-exact}
 \Reg_\alpha(T)\le K_2\sqrt T+K_0.
\end{equation}
One stochastic value query per round with $\rho=T^{-1/4}/4$ gives
\begin{equation}\label{eq:value-online-stochastic}
 \Reg_\alpha(T)\le K_1 T^{3/4}.
\end{equation}
Both procedures use one projection per round, reward an action in $K$, and query the current objective only at points of $Q$ after that action is fixed.
\end{corollary}
\begin{proof}
By Lemma~\ref{lem:value-smoothing}, $F_\rho$ is a nonnegative continuously differentiable DR-submodular function on $Q$, so Corollary~\ref{thm:geometry} applies to it. For every comparator $o\in K$ and played action $p=h_\zeta(x)$, the uniform estimate in \eqref{eq:value-smoothing-bias}, applied at $o$ with weight $\alpha$ and at $p$ with weight one, gives
\begin{align*}
 \alpha f(o)-f(p)
 &\le\alpha F_\rho(o)-F_\rho(p)
      +\frac{2}{\sqrt3}(1+\alpha)B\rho\sqrt d\\
 &\le\alpha\ip{g_\zeta(F_\rho,x)}{o-x}
      +\frac{2}{\sqrt3}(1+\alpha)B\rho\sqrt d.
\end{align*}
Lemma~\ref{lem:value-smoothing} supplies a conditionally unbiased estimate of $g_\zeta(F_\rho,x_t)$ with second moment at most $V_\rho$. Thus \eqref{eq:ombq-conditions} holds with $\beta=\alpha$, $V=V_\rho$, and $e_t=(2/\sqrt3)(1+\alpha)B\rho\sqrt d$, and \eqref{eq:ombq-transfer} gives \eqref{eq:value-online-general}.

Since $T\ge1$, both choices of $\rho$ satisfy $0<\rho\le1/4<1/2$, so Lemma~\ref{lem:value-smoothing} applies and every query point lies in $Q$. For two exact values, $V_\rho=2dB^2$ and $\rho=1/(4T)$, so
\[
 \alpha D\sqrt{2dB^2T}+\frac{2}{\sqrt3}(1+\alpha)B\sqrt d\,T\cdot\frac{1}{4T}
 =\alpha DB\sqrt{2d}\,\sqrt T+\frac{1+\alpha}{2\sqrt3}B\sqrt d
 =K_2\sqrt T+K_0.
\]
For one stochastic value, $V_\rho=d^2M_1^2/\rho^2$ and $\rho=T^{-1/4}/4$, so $\sqrt{TV_\rho}=4dM_1T^{3/4}$ and
\[
 4\alpha DdM_1T^{3/4}+\frac{2}{\sqrt3}(1+\alpha)B\sqrt d\,T\cdot\frac{T^{-1/4}}4
 =\Bigl(4\alpha DdM_1+\frac{1+\alpha}{2\sqrt3}B\sqrt d\Bigr)T^{3/4}
 =K_1T^{3/4}.
\]
\end{proof}
Since the two-value moment bound $2dB^2$ does not depend on $\rho$, the radius can decrease like $1/T$, and the coefficient $K_2$ of $\sqrt T$ in \eqref{eq:value-online-exact} equals the exact-gradient constant in \eqref{eq:first-order-rates}.

Bandit feedback, which supplies only a noisy value at the rewarded action, is more restrictive than this post-decision value-oracle feedback, whose queried point may belong to $Q\setminus K$. The online $0.401$ value-oracle result of \citet{aggarwal2026} builds on the offline framework of \citet{buchbinder2024}; its one-query rate is $O(T^{4/5})$, while its $O(T^{3/4})$ implementation uses $O(dT^{1/4})$ queries per round. In the cube-query and projection model, \eqref{eq:value-online-stochastic} attains this rate with one value query per round and coefficient $\alow(\zeta)\ge4/9$ under the bounds in \eqref{eq:regularity} and \eqref{eq:value-oracle}.

\subsection{Gradient and value queries for polynomial objectives}\label{sec:polynomial-oracles}
For quadratic and cubic objectives, either gradients or values at a fixed set of polynomially many points determine the nonconstant coefficients, so exact responses compute the field of Theorem~\ref{thm:polynomial} without smoothing. With noisy responses, interpolation gives unbiased coefficient estimates, a linear program restores DR-submodularity, and the coefficient error of the result bounds the additional comparison error in the online reduction.

Write
\begin{equation}\label{eq:polynomial-oracle-dimensions}
 P_m=\binom{d+m}{m}-1,\qquad
 Q_2=d+1,\quad Q_3=\binom{d+2}{2},\qquad
 \kappa_2=2,\quad\kappa_3=16,
\end{equation}
where $P_m$ is the number of nonconstant monomials of degree at most $m$, $Q_m$ is the number of gradient query points, and $\kappa_m$ bounds the gradient interpolation weights. For value interpolation, define
\begin{equation}\label{eq:polynomial-value-dimensions}
 \overline Q_m=P_m+1=\binom{d+m}{m},\qquad
 \overline\kappa_2=16,\quad\overline\kappa_3=216.
\end{equation}
Each feedback type then has a query count, a weight bound, and a noise level:
\begin{equation}\label{eq:polynomial-feedback-parameters}
 (q_m,k_m,\sigma)=
 \begin{cases}
  (Q_m,\kappa_m,\sigma_g),&\text{gradient responses},\\
  (\overline Q_m,\overline\kappa_m,\sigma_v),&\text{value responses}.
 \end{cases}
\end{equation}
At each of the $q_m$ prescribed points the procedure averages $n$ responses, which satisfy \eqref{eq:gradient-noise} or \eqref{eq:value-oracle}, respectively.

\begin{proposition}[Polynomial field oracle]\label{prop:polynomial-oracle}
Under the assumptions of Theorem~\ref{thm:polynomial}, with $\norm{\nabla f(y)}_2\le B$ on the cube, either feedback choice in \eqref{eq:polynomial-feedback-parameters} uses $q_mn$ queries and a number of arithmetic operations polynomial in $d$ and $n$. It returns a vector $\widehat s_f(x)$ such that, conditionally on the current objective and state,
\begin{align}
 \alpha_m(\zeta)f(o)-\E f(p_\zeta(x))
 &\le\alpha_m(\zeta)\E\ip{\widehat s_f(x)}{o-x}
       +2\ell_{m,n}\qquad(o\in K),\label{eq:polynomial-noisy-comparison}\\
 \left(\E\norm{\widehat s_f(x)}_2^2\right)^{1/2}
 &\le\Lambda_m\left(\frac{B\sqrt d}{2}+2\ell_{m,n}\right),\qquad
 \ell_{m,n}=\frac{2P_m k_m\sigma}{\sqrt n},
 \label{eq:polynomial-noisy-moment}
\end{align}
with $\Lambda_m$ from \eqref{eq:polynomial-field-bound}. For exact responses of either type, $n=1$ gives the exact field of Theorem~\ref{thm:polynomial} and zero comparison error.
\end{proposition}
\begin{proof}
\emph{Gradient interpolation.}
Let $\overline V(y)$ denote the average of the $n$ gradient responses at $y$. For $m=2$, each $\partial_if$ is affine; query $\zero,e_1,\ldots,e_d$ and set
\begin{equation}\label{eq:quadratic-gradient-interpolation}
 \widehat r_i(x)=\overline V_i(\zero)+\sum_{j=1}^d\bigl(\overline V_i(e_j)-\overline V_i(\zero)\bigr)x_j .
\end{equation}
Each coefficient of $\widehat r_i$ is an unbiased estimate of the corresponding coefficient of $\partial_i f$, and the weights on the responses have absolute sum at most two.
For $m=3$, each $\partial_i f$ is a quadratic polynomial; query
\begin{equation}\label{eq:cubic-gradient-nodes}
 \zero,\quad e_j,\quad e_j/2\ (1\le j\le d),\quad
 (e_j+e_k)/2\ (1\le j<k\le d).
\end{equation}
These are $1+2d+\binom d2=\binom{d+2}{2}=Q_3$ points in the cube. Define
\begin{align}
 \widehat b_i&=\overline V_i(\zero),\nonumber\\
 \widehat a_{ij}&=4\overline V_i(e_j/2)-\overline V_i(e_j)-3\overline V_i(\zero),\nonumber\\
 \widehat c_{ij}&=2\overline V_i(e_j)-4\overline V_i(e_j/2)+2\overline V_i(\zero),\label{eq:cubic-gradient-coefficients}\\
 \widehat d_{i;jk}&=4\bigl(\overline V_i((e_j+e_k)/2)-\overline V_i(e_j/2)-\overline V_i(e_k/2)+\overline V_i(\zero)\bigr),\nonumber\\
 \widehat r_i(x)&=\widehat b_i+\sum_j\widehat a_{ij}x_j
                   +\sum_j\widehat c_{ij}x_j^2
                   +\sum_{j<k}\widehat d_{i;jk}x_jx_k.\nonumber
\end{align}
If a quadratic $r(x)=\beta+\sum_j\gamma_jx_j+\sum_j\eta_jx_j^2+\sum_{j<k}\theta_{jk}x_jx_k$ replaces $\overline V_i$, then $r(e_j)=\beta+\gamma_j+\eta_j$, $r(e_j/2)=\beta+\gamma_j/2+\eta_j/4$, and $r((e_j+e_k)/2)=r(e_j/2)+r(e_k/2)-\beta+\theta_{jk}/4$, so the four formulas return $\beta,\gamma_j,\eta_j,\theta_{jk}$. With exact responses they therefore recover $\partial_if$, and with noisy responses each coefficient is unbiased. The absolute sums of the response weights in $\widehat b_i,\widehat a_{ij},\widehat c_{ij},\widehat d_{i;jk}$ are $1,8,8,16$.

Index the nonconstant monomials by multi-indices $\nu\in\mathbb N_0^d$ with $1\le|\nu|\le m$. Fix $i(\nu)=\min\{i:\nu_i>0\}$ and set
\begin{equation}\label{eq:polynomial-coefficient-estimator}
 \widehat a_\nu
   =\frac{[x^{\nu-e_{i(\nu)}}]\widehat r_{i(\nu)}(x)}{\nu_{i(\nu)}},
\end{equation}
where $[x^\mu]r$ is the coefficient of $x^\mu$ in $r$. Write
$q(x)=f(x)-f(\zero)=\sum_\nu a_\nu x^\nu$. The coefficient of $x^{\nu-e_i}$ in $\partial_iq$ is $\nu_ia_\nu$ whenever $\nu_i>0$, so $\E\widehat a_\nu=a_\nu$. Since $\nu_{i(\nu)}\ge1$, each $\widehat a_\nu$ is a linear combination of averaged gradient coordinates whose weights have absolute sum at most $2=\kappa_2$ for $m=2$ and at most $16=\kappa_3$ for $m=3$.

\emph{Value interpolation.}
Query all nodes
\begin{equation}\label{eq:polynomial-value-nodes}
 \mathcal N_m=\{\lambda/m:\lambda\in\mathbb N_0^d,\ |\lambda|\le m\}\subseteq Q,
\end{equation}
of which there are $\binom{d+m}{m}=\overline Q_m$, and let $\overline Z(\lambda/m)$ be the averaged value response. For multi-indices write
$\binom\nu\lambda=\prod_i\binom{\nu_i}{\lambda_i}$, and form
\begin{align}
 \widehat\Delta_\nu
 &=\sum_{0\le\lambda\le\nu}
       (-1)^{|\nu-\lambda|}\binom\nu\lambda \overline Z(\lambda/m),\label{eq:polynomial-value-differences}\\
 \widehat q(x)
 &=\sum_{1\le|\nu|\le m}\widehat\Delta_\nu
             \prod_{i=1}^d\binom{m x_i}{\nu_i},
 \label{eq:polynomial-value-newton}
\end{align}
where $\binom{t}{k}=t(t-1)\cdots(t-k+1)/k!$ for $k\ge1$ and $\binom t0=1$.
For exact values, $\widehat\Delta_\nu$ is the forward difference of $\lambda\mapsto f(\lambda/m)$ at $\zero$, and the multivariate Newton forward formula gives $\sum_{|\nu|\le m}\widehat\Delta_\nu\binom\lambda\nu=f(\lambda/m)$ for every node; since $\widehat\Delta_\zero=f(\zero)$, this says $\widehat q(\lambda/m)=f(\lambda/m)-f(\zero)$. The polynomials
$\prod_i\binom{m x_i}{\nu_i}$ with $|\nu|\le m$ form a basis of the polynomials of degree at most $m$: their evaluation matrix $\bigl(\binom\lambda\nu\bigr)$ on the nodes vanishes unless $\nu\le\lambda$, so it is triangular with unit diagonal when the multi-indices are ordered by total degree. Hence \eqref{eq:polynomial-value-newton} equals $q=f-f(\zero)$ as a polynomial. With noisy values, take $\widehat a_\nu=[x^\nu]\widehat q(x)$; this is linear in the responses, so conditional unbiasedness gives $\E\widehat a_\nu=a_\nu$.

The absolute sum of the response weights in $\widehat\Delta_\nu$ is $\sum_{\lambda\le\nu}\binom\nu\lambda=2^{|\nu|}$. For $m=2$, $\binom{2x}{1}=2x$ and $\binom{2x}{2}=2x^2-x$; the coefficient of $x_i$ is $2\widehat\Delta_{e_i}-\widehat\Delta_{2e_i}$, that of $x_i^2$ is $2\widehat\Delta_{2e_i}$, and that of $x_ix_j$ is $4\widehat\Delta_{e_i+e_j}$, with weight sums $8,8,16$. For $m=3$,
\[
 \binom{3x}{1}=3x,\qquad
 \binom{3x}{2}=\tfrac92x^2-\tfrac32x,
 \qquad
 \binom{3x}{3}=\tfrac92x^3-\tfrac92x^2+x .
\]
Collecting terms, the monomial types
$x_i,x_i^2,x_ix_j,x_i^3,x_i^2x_j,x_ix_jx_k$
have weight sums
$3\cdot2+\tfrac32\cdot4+8=20$, $\tfrac92\cdot4+\tfrac92\cdot8=54$, $9\cdot4+2\cdot\tfrac92\cdot8=108$, $\tfrac92\cdot8=36$, $\tfrac92\cdot3\cdot8=108$, and $27\cdot8=216$. These types exhaust the nonconstant monomials, which proves the bounds $\overline\kappa_m$ in \eqref{eq:polynomial-value-dimensions}.

\emph{Coefficient error.}
Under either feedback choice, each coordinate of a single response has conditional mean-square error at most $\sigma^2$. The errors of the $n$ responses at a node are conditionally mean-zero given the past, even when successive responses are dependent, so the cross terms vanish when the square of their sum is expanded and each averaged coordinate has mean-square error at most $\sigma^2/n$. Minkowski's inequality in $L^2$ then gives
\begin{equation}\label{eq:coefficient-moment-bound}
 \E|\widehat a_\nu-a_\nu|^2\le\frac{k_m^2\sigma^2}{n},\qquad
 \E\norm{\widehat a-a}_1^2\le\frac{P_m^2 k_m^2\sigma^2}{n}.
\end{equation}

\emph{Enforcing DR-submodularity.}
For a coefficient vector $v$, write $q_v(x)=\sum_\nu v_\nu x^\nu$. Every Hessian entry is affine when $m\le3$:
\begin{equation}\label{eq:polynomial-hessian-affine}
 \partial_i\partial_jq_v(x)=H_{ij,0}(v)+\sum_{k=1}^dH_{ij,k}(v)x_k,
 \qquad 1\le i\le j\le d,
\end{equation}
where each $H_{ij,k}$ is a linear function of $v$ with integer coefficients of absolute value at most six (the largest arises from $\partial_i^2x_i^3=6x_i$). An affine function is nonpositive on the cube if and only if its maximum over the vertices is nonpositive, that is,
$H_{ij,0}(v)+\sum_k(H_{ij,k}(v))_+\le0$.
Therefore the following is a linear program with polynomially many variables and constraints:
\begin{align}
 \text{minimize }&\sum_\nu t_\nu\nonumber\\
 \text{subject to }&-t_\nu\le v_\nu-\widehat a_\nu\le t_\nu,\quad t_\nu\ge0,\nonumber\\
 &z_{ijk}\ge0,\quad z_{ijk}\ge H_{ij,k}(v),\quad
 H_{ij,0}(v)+\sum_kz_{ijk}\le0
 \quad(1\le i\le j\le d).\label{eq:polynomial-shape-lp}
\end{align}
For quadratics, $H_{ij,k}=0$ for $k\ge1$, and the last line reduces to $H_{ij,0}(v)\le0$. The point $v=\zero$, $z=\zero$, $t_\nu=|\widehat a_\nu|$ is feasible and the objective is bounded below by zero, so an optimum exists. The constraint matrix has integer entries of absolute value at most six that do not depend on the data; the estimates enter only the right-hand side. The algorithm of \citet{tardos1986} therefore solves the program in a number of arithmetic operations polynomial in $d$. Select an optimum $\widetilde a$ by a fixed tie rule. The true coefficient vector $a$ is feasible because $f$ is DR-submodular, so $\norm{\widetilde a-\widehat a}_1\le\norm{a-\widehat a}_1$ and
\begin{equation}\label{eq:polynomial-projection-error}
 \delta:=\norm{\widetilde a-a}_1
 \le\norm{\widetilde a-\widehat a}_1+\norm{\widehat a-a}_1
 \le2\norm{\widehat a-a}_1.
\end{equation}
By \eqref{eq:coefficient-moment-bound}, $\E\delta^2\le4P_m^2k_m^2\sigma^2/n=\ell_{m,n}^2$, and hence $\E\delta\le\ell_{m,n}$. Since each monomial takes values in $[0,1]$ on the cube,
\begin{equation}\label{eq:polynomial-uniform-error}
 |q_{\widetilde a}(y)-q(y)|\le\delta\qquad(y\in[0,1]^d).
\end{equation}

\emph{Normalization and the field.}
Compute $\mu=\min_{[0,1]^d}q_{\widetilde a}$ and apply the construction of Appendix~\ref{app:polynomial} to the nonnegative DR-submodular polynomial
$F=q_{\widetilde a}-\mu$. The minimum is computable in polynomial time: since $q_{\widetilde a}$ is concave in each coordinate, it is attained at a vertex of the cube, and for $m=2$ the graph construction of Lemma~\ref{lem:quadratic-representation} computes the vertex minimum by one minimum cut. For $m=3$, Lemma~\ref{lem:cubic-polarization} shows that polarization preserves the minimum, and the auxiliary-variable construction in the proof of Lemma~\ref{lem:cubic-multilinear} computes the minimum of the polarized polynomial by one minimum cut. The same residual graph gives the nonnegative decomposition and the field. Let $\widehat s_f(x)$ be the resulting supergradient, which is the field that Theorem~\ref{thm:polynomial} assigns to $F$ and, by shift invariance, also to $q_{\widetilde a}$.

Fix a realization of the estimated coefficients, and put $e=q-q_{\widetilde a}$ and $c=f(\zero)+\mu$. Then $f=F+e+c$ and $|e|\le\delta$ by \eqref{eq:polynomial-uniform-error}. Moreover, \eqref{eq:polynomial-uniform-error} gives $\mu\ge\min_Q q-\delta$, so $c\ge\min_Q f-\delta\ge-\delta$ because $f\ge0$. Write $\alpha=\alpha_m(\zeta)$ and let $\E_{\rm act}$ average the action law with the estimated coefficients held fixed. Applying \eqref{eq:polynomial-comparison} to $F$, and using $\alpha e(o)\le\alpha\delta$, $-\E_{\rm act}e(p_\zeta(x))\le\delta$, and $-(1-\alpha)c\le(1-\alpha)\delta$, we obtain
\begin{align}
 \alpha f(o)-\E_{\rm act}f(p_\zeta(x))
 &=\alpha F(o)-\E_{\rm act}F(p_\zeta(x))
   +\alpha e(o)-\E_{\rm act}e(p_\zeta(x))-(1-\alpha)c\nonumber\\
 &\le\alpha\ip{\widehat s_f(x)}{o-x}
       +(1+\alpha)\delta+(1-\alpha)\delta
 =\alpha\ip{\widehat s_f(x)}{o-x}+2\delta.
 \label{eq:polynomial-perturbed-comparison}
\end{align}
The left side depends only on $f$ and $x$, since the action law does not depend on the coefficients. Averaging over the oracle responses and using $\E\delta\le\ell_{m,n}$ proves \eqref{eq:polynomial-noisy-comparison}.

Finally, \eqref{eq:polynomial-field-bound} applied to $q_{\widetilde a}$ and \eqref{eq:polynomial-uniform-error} give
\begin{equation}\label{eq:polynomial-surrogate-field}
 \begin{split}
 \norm{\widehat s_f(x)}_2
 &\le\Lambda_m\left(q_{\widetilde a}(\tfrac12\one)
                  -\min_Q q_{\widetilde a}\right)\\
 &\le\Lambda_m\left(f(\tfrac12\one)-\min_Q f+2\delta\right)
 \le\Lambda_m\left(\frac{B\sqrt d}{2}+2\delta\right).
 \end{split}
\end{equation}
The last inequality applies \eqref{eq:segment-ftc} on the segment from a minimizer $y$ of $f$ to $\tfrac12\one$, whose length is at most $\sqrt d/2$. Taking the $L^2$ norm and using Minkowski's inequality with $(\E\delta^2)^{1/2}\le\ell_{m,n}$ proves \eqref{eq:polynomial-noisy-moment}. Exact responses give $\widehat a=a$, hence $\widetilde a=a$ and $\delta=0$, which proves the exact assertion. The node sets, the linear program, and the cut constructions all have polynomial size.
\end{proof}

\begin{corollary}[Online polynomial objectives]\label{cor:polynomial-online}
Suppose every objective has degree at most $m\in\{2,3\}$ and satisfies the preceding assumptions. Set $\alpha=\alpha_m(\zeta)$, and define
\begin{equation}\label{eq:polynomial-online-constants}
 C_m=4P_m k_m\sigma,\qquad
 R_m=\frac{\alpha D\Lambda_m}{2}B\sqrt d
       +\left(\alpha D\Lambda_m+1\right)C_m,
\end{equation}
with $\Lambda_m$ from \eqref{eq:polynomial-field-bound}. Using Proposition~\ref{prop:polynomial-oracle} in OMBQ with projected linear ascent gives
\begin{equation}\label{eq:polynomial-online-regret}
 \Reg_{\alpha}(T)\le R_m\sqrt T.
\end{equation}
For exact feedback, the algorithm uses $Q_m$ gradient queries or $\overline Q_m$ value queries per round, and $R_m=(\alpha D\Lambda_m/2)B\sqrt d$. For noisy feedback, take $n=T$; the corresponding counts are $Q_mT$ gradient queries or $\overline Q_mT$ value queries per round. Each version uses one projection per round, queries the objective only in the cube, and plays actions in $K$ almost surely.
\end{corollary}
\begin{proof}
For exact responses, \eqref{eq:polynomial-field-bound} bounds the field norm by $\Lambda_mB\sqrt d/2$, and the comparison error is zero. Apply the OMBQ transfer following \eqref{eq:randomized-ombq} with this bound, using a fixed state if the bound is zero. For noisy responses and $n=T$, $\ell_{m,T}=C_m/(2\sqrt T)$, so Proposition~\ref{prop:polynomial-oracle} gives comparison error at most $C_m/\sqrt T$ per round and root second moment at most $G=\Lambda_m(B\sqrt d/2+C_m/\sqrt T)$. Projected ascent with step size $D/(G\sqrt T)$ has expected linear regret at most $DG\sqrt T$, and the comparison errors add $T\cdot C_m/\sqrt T$. Since $\beta=\alpha$ and $T\ge1$,
\begin{equation}\label{eq:polynomial-online-intermediate}
 \Reg_{\alpha}(T)
 \le\alpha D\Lambda_m\left(\frac{B\sqrt d}{2}+\frac{C_m}{\sqrt T}\right)\sqrt T
       +C_m\sqrt T
 \le R_m\sqrt T.
\end{equation}
The expectation form \eqref{eq:randomized-ombq} covers the randomized cubic action, whose law is chosen before the current objective is observed. The query counts follow from \eqref{eq:polynomial-feedback-parameters}.
\end{proof}

If the state is held fixed for a block of rounds, a single value query per round suffices, at the cost of the rate $O(T^{3/4})$: the per-round estimates are unbiased for the coefficients of the block-average objective.
\begin{corollary}[One value query for polynomial objectives]\label{cor:polynomial-one-value}
Under the oblivious-adversary and value-oracle model of Section~\ref{sec:model-oracles}, suppose every objective has degree at most $m\in\{2,3\}$ and satisfies the assumptions of Theorem~\ref{thm:polynomial} with $\norm{\nabla f_t}_2\le B$ on the cube. Put $\alpha=\alpha_m(\zeta)$ and
\begin{equation}\label{eq:polynomial-block-parameters}
 L=\lceil\sqrt T\rceil,\qquad b=\lceil T/L\rceil,\qquad
 \Gamma_m=2P_m\overline\kappa_m\sqrt{\overline Q_m}\,M_1.
\end{equation}
There is an algorithm using one value query per round and one projection per block whose regret satisfies
\begin{equation}\label{eq:polynomial-block-regret}
 \Reg_\alpha(T)
 \le\frac{\alpha D\Lambda_m}{2}B\sqrt d\,L\sqrt b
      +2\alpha D\Lambda_m\Gamma_m\sqrt{bL}
      +2\Gamma_m b\sqrt L
 =O(T^{3/4}).
\end{equation}
The constants in this rate are polynomial in $d$, $B$, $D$, and $M_1$, the actions lie in $K$ almost surely, and the value queries lie in $Q$.
\end{corollary}
\begin{proof}
List the nodes in \eqref{eq:polynomial-value-nodes} as $y_1,\ldots,y_{\overline Q_m}$. The value interpolation in the proof of Proposition~\ref{prop:polynomial-oracle} supplies fixed weights $c_{\nu j}$ such that the coefficients $a_{\nu,t}$ of $f_t-f_t(\zero)$ satisfy
\begin{equation}\label{eq:polynomial-node-linear-weights}
 a_{\nu,t}=\sum_{j=1}^{\overline Q_m}c_{\nu j}f_t(y_j),
 \qquad \sum_j|c_{\nu j}|\le\overline\kappa_m
 \quad(1\le|\nu|\le m).
\end{equation}
Pad the sequence to length $bL$ with zero objectives, which contribute neither reward nor queries, and split it into $b$ blocks of $L$ rounds. During a block the internal state $x$ is fixed. Each real round first plays an independent draw from the action law $p_\zeta(x)$, then draws a uniform node index $J_t$, queries $Z_t(y_{J_t})$, and forms
\begin{equation}\label{eq:polynomial-block-observation}
 \widehat a_{\nu,t}=\overline Q_m c_{\nu,J_t}Z_t(y_{J_t})
 \qquad(1\le|\nu|\le m).
\end{equation}
A padded round has $\widehat a_{\nu,t}=0$. At the end of the block, average these vectors, enforce DR-submodularity with \eqref{eq:polynomial-shape-lp}, normalize by the cube minimum, and compute the field as in Proposition~\ref{prop:polynomial-oracle}. Projected linear ascent uses this field to select the state for the next block.

To bound the coefficient error, condition on the objective sequence and the history before a real round. Averaging over the uniform node and using the conditional moments in \eqref{eq:value-oracle},
\begin{align}
 \E \widehat a_{\nu,t}&=a_{\nu,t},\nonumber\\
 \E|\widehat a_{\nu,t}-a_{\nu,t}|^2
 &\le\E \widehat a_{\nu,t}^2
 \le\overline Q_m M_1^2\sum_jc_{\nu j}^2
 \le\overline Q_m\overline\kappa_m^2M_1^2.
 \label{eq:polynomial-block-variance}
\end{align}
The centered errors $\widehat a_{\nu,t}-a_{\nu,t}$ therefore form a martingale difference sequence even though the objectives vary across rounds, so their cross terms have zero mean and the block average estimates the coefficient $\overline a_\nu$ of
$\overline f=L^{-1}\sum_{t\text{ in block}}f_t$
with mean-square error at most $\overline Q_m\overline\kappa_m^2M_1^2/L$.
Summing over the $P_m$ coefficients with Minkowski's inequality and applying \eqref{eq:polynomial-projection-error} gives
\begin{equation}\label{eq:polynomial-block-shape-error}
 \left(\E\norm{\widetilde a-\overline a}_1^2\right)^{1/2}
 \le\frac{\Gamma_m}{\sqrt L}.
\end{equation}
The block average $\overline f$ is a nonnegative DR-submodular polynomial of degree at most $m$ with gradient norm at most $B$. Repeating the last two steps of the proof of Proposition~\ref{prop:polynomial-oracle} with $\overline f$ in place of $f$, its estimated field has comparison error at most $2\Gamma_m/\sqrt L$ in expectation and root second moment at most
\begin{equation}\label{eq:polynomial-block-field-moment}
 G_L=\Lambda_m\left(\frac{B\sqrt d}{2}+\frac{2\Gamma_m}{\sqrt L}\right)
 =\frac{\Lambda_m}{2}B\sqrt d+\frac{2\Lambda_m\Gamma_m}{\sqrt L}.
\end{equation}

Conditional on the state at the start of a block, the independent action draws have total expected reward
$L\E_{\rm act}\overline f(p_\zeta(x))$, and $L\overline f(o)$ is the comparator's reward in the block. Multiplying the comparison for $\overline f$ by $L$ therefore controls the regret of the block. If $G_L=0$, then $B=M_1=0$ and all rewards vanish, so a fixed state suffices. Otherwise, projected ascent over the $b$ blocks with step size $D/(G_L\sqrt b)$ has expected linear regret at most $DG_L\sqrt b$. Summing the block comparisons and the errors from \eqref{eq:polynomial-block-shape-error} gives
\[
 \Reg_\alpha(T)\le\alpha L D G_L\sqrt b+2\Gamma_m b\sqrt L,
\]
which expands to \eqref{eq:polynomial-block-regret}. Finally, $L\le\sqrt T+1\le2\sqrt T$, $b\le T/L+1\le2\sqrt T$, and $bL\le T+L\le2T$. Hence $L\sqrt b$ and $b\sqrt L$ are at most $2\sqrt2\,T^{3/4}$ and $\sqrt{bL}\le\sqrt{2T}$, which proves the rate. The coefficient operations and the reconstruction at each block boundary have polynomial size.
\end{proof}

\section{Offline Upper Bound}\label{sec:upper}
We prove two upper bounds for offline approximation under the diagonal condition $\zeta\one\in K$. The first, $U(\zeta)$, comes from a construction with four DR-submodular components; it equals $\beta_*$ at $\zeta=0$, reduces to a finite minimization for $0<\zeta<1/8$, and has a piecewise rational formula for $\zeta\ge1/8$. The second, $C(\zeta)$, is given by a closed formula and is the smaller of the two for $\zeta\ge1/4$. Both rest on a passage from a finite symmetry gap to hardness for algorithms that receive exact function values and full gradients.

\subsection{Oracle model and the symmetry-gap argument}\label{sec:upper-setup}\label{sec:upper-instance}
Since the supplied level only has to satisfy $\zeta\one\in K$, an instance whose diagonal radius $\zeta(K)$ exceeds $\zeta$ also belongs to the class for $\zeta$. The following definition fixes the approximation and query requirements, including the normalization needed for a relative guarantee.

\begin{definition}[Offline oracle coefficient]\label{def:oracle-coefficient}
Fix $\zeta\in[0,1]$. An instance consists of a nonempty compact convex down-closed set $K\subseteq[0,1]^d$ containing $\zeta\one$, a nonnegative continuously differentiable DR-submodular function $f$ on $[0,1]^d$, and supplied bounds
\[
 0\le f(y)\le M,\qquad \|\nabla f(y)\|_2\le B
 \quad(y\in[0,1]^d),\qquad
 0<v\le\OPT:=\max_{y\in K}f(y).
\]
Let $D=\operatorname{diam}(K)$, or a supplied upper bound on this diameter. An objective query at $y\in[0,1]^d$ returns the exact pair $(f(y),\nabla f(y))$, and exact Euclidean projection onto $K$ is supplied separately.

A coefficient $a\in[0,1]$ is achievable if one randomized algorithm and one fixed polynomial $P$ have the following property: for every such instance and every $\varepsilon\in(0,1)$, the algorithm returns $p\in K$ almost surely, satisfies
\begin{equation}\label{eq:oracle-coefficient-guarantee}
 \E f(p)\ge(a-\varepsilon)\OPT,
\end{equation}
and makes, on every execution, at most
\begin{equation}\label{eq:oracle-polynomial-budget}
 P\bigl(d,\varepsilon^{-1},1+D,1+M/v,1+B/v\bigr)
\end{equation}
calls to each oracle. The algorithm and polynomial may depend on the fixed $\zeta$ and $a$. Define $\aor(\zeta)$ as the supremum of the achievable coefficients.
\end{definition}

The ratios $M/v$ and $B/v$ are the value and gradient bounds of the rescaled objective $f/v$, whose optimum is at least one. In the hard instances below, $M$, $B$ and $v$ do not depend on the dimension and $D\le\sqrt d$, so the budget \eqref{eq:oracle-polynomial-budget} is polynomial in $d$.

Let $\mathcal G$ be a finite group of permutations of $[d_0]$, acting on $y\in\R^{d_0}$ by permuting coordinates, and let $\mathsf T y=|\mathcal G|^{-1}\sum_{\sigma\in\mathcal G}\sigma(y)$. For each $i$, every element of the orbit of $i$ is the image of $i$ under the same number of group elements, so $\mathsf T$ replaces the coordinates in each orbit by their mean. A function $F$ is $\mathcal G$-invariant if $F(\sigma(y))=F(y)$ for all $\sigma\in\mathcal G$. For a $\mathcal G$-invariant objective $F$ and a feasible set $P$, define
\begin{equation}\label{eq:symmetry-gap-main}
 O=\max_{y\in P}F(y),\qquad
 S=\max_{y\in P}F(\mathsf T y),\qquad
 \gamma=S/O.
\end{equation}
The ratio $\gamma$ is the symmetry gap. Lemma~\ref{lem:continuous-refinement} adapts the hidden-permutation argument of \citet{vondrak2013} to queries that return the full gradient; its smoothing step, proved in Appendix~\ref{app:upper-smoothing}, makes values and gradients agree near the invariant subspace.

\begin{lemma}[Exact value-and-gradient hardness from a symmetry gap]\label{lem:continuous-refinement}
Let $F:[0,1]^{d_0}\to\R_+$ be a multilinear DR-submodular polynomial invariant under $\mathcal G$. Let $P\subseteq[0,1]^{d_0}$ be a nonempty compact convex down-closed polytope such that
\[
 y\in P\quad\Longleftrightarrow\quad\mathsf T y\in P.
\]
Suppose $\zeta\one\in P$ and $0<S\le O$ in \eqref{eq:symmetry-gap-main}. For each fixed $\varepsilon>0$, there are instances in dimensions $d=md_0$ on which an expected approximation coefficient $\gamma+\varepsilon$ requires $\exp(\Omega(d))$ exact value-and-full-gradient queries. Their feasible polytopes are independent of the hidden objective, and their value bounds, gradient bounds, and positive optimum lower bounds can be chosen independently of $m$. Supplying exact projections onto these polytopes preserves the lower bound. Consequently $\aor(\zeta)\le\gamma$.
\end{lemma}
\begin{proof}
Fix $\varepsilon>0$. Since $O\ge S>0$, we can choose $\nu$ with
\[
 0<\nu<\frac S2,\qquad
 \frac{S+\nu}{O-\nu}<\gamma+\varepsilon.
\]
Lemma~\ref{lem:upper-smoothing} in Appendix~\ref{app:upper-smoothing}, applied with error $\nu$, gives nonnegative twice continuously differentiable DR-submodular functions $\widehat F,\widehat G$ on $[0,1]^{d_0}$ and a number $\delta>0$ such that
\begin{equation}\label{eq:upper-symmetrization}
 \begin{gathered}
 \widehat G(y)=\widehat F(\mathsf T y),\qquad
 \|\widehat F-F\|_\infty\le\nu,\\
 \bigl(\widehat F(y),\nabla\widehat F(y)\bigr)
 =\bigl(\widehat G(y),\nabla\widehat G(y)\bigr)
 \quad\text{if }\|y-\mathsf T y\|_2^2\le\delta.
 \end{gathered}
\end{equation}
The function $\widehat F$ interpolates between $F$ and $F\circ\mathsf T$ through a cutoff in the squared distance $\|y-\mathsf Ty\|_2^2$ and adds a small invariant concave quadratic; the Hessians of $\widehat F$ and $\widehat G$ are entrywise nonpositive on the whole cube, and \eqref{eq:upper-symmetrization} also holds on its boundary.

Independently for each $j\in[m]$, choose $\sigma_j$ uniformly from $\mathcal G$, and write $\sigma=(\sigma_1,\ldots,\sigma_m)$. For $z\in[0,1]^{m\times d_0}$ with rows $z_1,\ldots,z_m$, define
\begin{equation}\label{eq:refinement-map}
 (\Xi_\sigma z)_i=\frac1m\sum_{j=1}^m z_{j,\sigma_j(i)},
 \qquad \bar z_i=\frac1m\sum_{j=1}^m z_{j,i}.
\end{equation}
Each $\sigma_j$ maps every orbit of $\mathcal G$ onto itself, so the sum of $(\Xi_\sigma z)_i$ over an orbit equals the sum of $\bar z_i$ over the same orbit. Since $y\in P$ if and only if $\mathsf Ty\in P$, it follows that
\begin{equation}\label{eq:refinement-symmetry}
 \mathsf T\Xi_\sigma z=\mathsf T\bar z,
 \qquad
 P_m:=\{z\in[0,1]^{m\times d_0}:\mathsf T\bar z\in P\}
       =\{z:\Xi_\sigma z\in P\}.
\end{equation}
The set $P_m$ is the preimage of $P$ under the linear map $z\mapsto\mathsf T\bar z$, which has an entrywise nonnegative matrix and maps the cube into $[0,1]^{d_0}$. Hence $P_m$ is a compact convex polytope; it is down-closed because $P$ is, and it contains $\zeta\one$ because $\mathsf T\bar z=\zeta\one$ for $z=\zeta\one$. Neither $P_m$ nor any projection onto it depends on $\sigma$.

Define the hidden and reference objectives by
\begin{equation}\label{eq:refined-objectives}
 f_\sigma(z)=\widehat F(\Xi_\sigma z),\qquad
 g_m(z)=\widehat G(\Xi_\sigma z)=\widehat F(\mathsf T\bar z).
\end{equation}
The last equality follows from $\widehat G=\widehat F\circ\mathsf T$ and \eqref{eq:refinement-symmetry}, so $g_m$ does not depend on $\sigma$. Both $\Xi_\sigma$ and $z\mapsto\mathsf T\bar z$ satisfy the hypotheses of Lemma~\ref{lem:affine-dr}, which shows that $f_\sigma$ and $g_m$ are nonnegative and DR-submodular. Let $M_0=\max F$ and $\widehat B=\max\|\nabla\widehat F\|_2$, both over $[0,1]^{d_0}$. By the chain rule, $\partial_{z_{j,\ell}}f_\sigma(z)=m^{-1}\partial_{\sigma_j^{-1}(\ell)}\widehat F(\Xi_\sigma z)$, so $\|\nabla f_\sigma(z)\|_2^2=m^{-2}\sum_{j}\|\nabla\widehat F(\Xi_\sigma z)\|_2^2\le\widehat B^2/m$. The row of $\nabla g_m(z)$ indexed by $j$ is $m^{-1}\mathsf T\nabla\widehat F(\mathsf T\bar z)$, and $\mathsf T$ is an orthogonal projection, so the same bound holds for $g_m$. With $\|\widehat F-F\|_\infty\le\nu$ and $\widehat F\ge0$, this gives
\begin{equation}\label{eq:refined-gradient}
 \begin{gathered}
 0\le f_\sigma,g_m\le M_0+\nu,\qquad
 \partial_{z_{j,\ell}}f_\sigma(z)
  =\frac1m\partial_{\sigma_j^{-1}(\ell)}\widehat F(\Xi_\sigma z),\\
 \|\nabla f_\sigma(z)\|_2,\|\nabla g_m(z)\|_2\le \widehat B/\sqrt m.
 \end{gathered}
\end{equation}
For every $y\in P$, the point $z$ with $z_{j,\sigma_j(i)}=y_i$ for all $j,i$ satisfies $\Xi_\sigma z=y$, so $z\in P_m$, $f_\sigma(z)=\widehat F(y)$, and $g_m(z)=\widehat F(\mathsf Ty)$. Taking $y$ to attain $O$ or $S$, and using $\|\widehat F-F\|_\infty\le\nu$, gives the lower bounds below. Conversely, for $z\in P_m$ the point $\mathsf T\bar z\in P$ is invariant, so $g_m(z)\le F(\mathsf T\bar z)+\nu=F(\mathsf T(\mathsf T\bar z))+\nu\le S+\nu$. Therefore
\begin{equation}\label{eq:refined-optima}
 \max_{P_m}f_\sigma\ge O-\nu,
 \qquad
 S-\nu\le\max_{P_m}g_m\le S+\nu.
\end{equation}
Since $\nu<S/2$, both objectives admit the supplied bounds $M=M_0+\nu$, $B=\widehat B$, $v=S/2$, and $D=\sqrt{md_0}$, none of which depends on $\sigma$. These bounds make the reference run an admissible instance of Definition~\ref{def:oracle-coefficient}, and the algorithm receives the same inputs in the hidden and reference runs.

Fix a cube point $z$ and a coordinate $i\in[d_0]$. For $\sigma_j$ uniform on $\mathcal G$, the index $\sigma_j(i)$ is uniform on the orbit of $i$, so the summands $z_{j,\sigma_j(i)}$, $j\in[m]$, are independent random variables in $[0,1]$ with means $(\mathsf T z_j)_i$. Their average $(\Xi_\sigma z)_i$ therefore has mean $(\mathsf T\bar z)_i=(\mathsf T\Xi_\sigma z)_i$. By Hoeffding's inequality \citep{hoeffding1963} with deviation $\sqrt{\delta/(2d_0)}$, each coordinate deviates by at least this amount with probability at most $2\exp(-m\delta/d_0)$. If no coordinate does, then $\|\Xi_\sigma z-\mathsf T\Xi_\sigma z\|_2^2<\delta/2$, so a union bound over the $d_0$ coordinates gives
\begin{equation}\label{eq:refinement-concentration}
 \Pr\!\left\{\|\Xi_\sigma z-\mathsf T\Xi_\sigma z\|_2^2\ge\delta/2\right\}
 \le 2d_0\exp(-m\delta/d_0).
\end{equation}
On the complementary event, \eqref{eq:upper-symmetrization} at $y=\Xi_\sigma z$ gives equal values and gradients of $\widehat F$ and $\widehat G$. Since $f_\sigma=\widehat F\circ\Xi_\sigma$ and $g_m=\widehat G\circ\Xi_\sigma$, the chain rule yields
\begin{equation}\label{eq:upper-exact-response}
 (f_\sigma(z),\nabla f_\sigma(z))=(g_m(z),\nabla g_m(z)).
\end{equation}

Now consider an algorithm that makes at most $Q(m)$ objective queries on every execution. Fix its internal random seed and run it against $(g_m,\nabla g_m)$ and the projection onto $P_m$. The resulting query points and output $z_{\rm out}$ are deterministic and independent of $\sigma$. Let $E$ be the event that $\|\Xi_\sigma y-\mathsf T\Xi_\sigma y\|_2^2<\delta/2$ for each of these query points $y$ and for $y=z_{\rm out}$. By the union bound, $\Pr(E^c)\le p_m$, where
\begin{equation}\label{eq:upper-transcript-failure}
 p_m=2d_0\bigl(Q(m)+1\bigr)\exp(-m\delta/d_0).
\end{equation}
On $E$, induction on the oracle calls shows that the run against $f_\sigma$ makes the same calls and receives the same responses. If the first $l$ calls and responses coincide, the next call coincides; an objective query then receives the same response by \eqref{eq:upper-exact-response}, and a projection call receives the same response because $P_m$ does not depend on $\sigma$. Hence on $E$ the hidden run also returns $z_{\rm out}\in P_m$, and \eqref{eq:upper-exact-response} at $z_{\rm out}$ together with \eqref{eq:refined-optima} gives $f_\sigma(z_{\rm out})=g_m(z_{\rm out})\le S+\nu$. On $E^c$ the returned value is at most $M_0+\nu$. Averaging over $\sigma$ and then over the seed gives
\begin{equation}\label{eq:refinement-output}
 \E_{\sigma,\mathrm{alg}} f_\sigma(z_{\rm out})
 \le S+\nu+(M_0+\nu)p_m.
\end{equation}
If $Q(m)\le\exp(m\delta/(2d_0))$, then $p_m\le2d_0\{\exp(-m\delta/(2d_0))+\exp(-m\delta/d_0)\}\to0$. By the choice of $\nu$, $S+\nu<(\gamma+\varepsilon)(O-\nu)$, so for all sufficiently large $m$, \eqref{eq:refinement-output} and \eqref{eq:refined-optima} give
\[
 \E_{\sigma}\E_{\mathrm{alg}} f_\sigma(z_{\rm out})
 <(\gamma+\varepsilon)(O-\nu)
 \le(\gamma+\varepsilon)\min_\sigma\max_{P_m}f_\sigma.
\]
Some $\sigma$ therefore satisfies $\E_{\mathrm{alg}}f_\sigma(z_{\rm out})<(\gamma+\varepsilon)\max_{P_m}f_\sigma$, and the algorithm does not attain expected coefficient $\gamma+\varepsilon$ on that instance. The base dimension $d_0$ is fixed before $m$ grows, so $\exp(m\delta/(2d_0))=\exp(\Omega(d))$.

Finally, suppose that some $a>\gamma$ were achievable with polynomial $P$, and put $\varepsilon'=(a-\gamma)/2$. On the instances above with $\varepsilon=\varepsilon'$, the algorithm run with accuracy $\varepsilon'$ attains expected coefficient $a-\varepsilon'=\gamma+\varepsilon'$, and by \eqref{eq:oracle-polynomial-budget} with the supplied bounds above it makes at most $P(md_0,1/\varepsilon',1+\sqrt{md_0},1+M/v,1+B/v)$ objective queries. This number is bounded by a polynomial in $m$ and is therefore smaller than $\exp(m\delta/(2d_0))$ for large $m$, a contradiction. Hence $\aor(\zeta)\le\gamma$.
\end{proof}

\subsection{A four-component construction for every diagonal level}\label{sec:upper-base}
For $0\le\zeta<1/2$, the hard instances combine four multilinear components. Two of them, a cut between two central coordinates and a coverage function, extend the construction of \citet[Appendix~E.2]{oveis2011}; the other two also take the value one at its asymmetric optimal point but penalize symmetric points in different ways.

Fix an integer $n\ge2$ and write $N=2n$. The coordinates are two central coordinates $u,v$ and $N$ matched pairs $(a_i,b_i)$, whose indices are partitioned into two sets $I_L,I_R$ of size $n$. Define
\begin{equation}\label{eq:hard-parameters}
 r_n=\frac1{2n},\qquad
 \rho_{n,\zeta}=\max\{\zeta,r_n\},\qquad
 \xi_{n,\zeta}=\frac{(\zeta-r_n)_+}{1/2-\zeta}.
\end{equation}
The feasible polytope is
\begin{equation}\label{eq:hard-polytope}
 P_{n,\zeta}=\left\{(u,v,a,b)\in[0,1]^{4n+2}:
 \begin{array}{l}
 u+v\le1,\qquad \displaystyle\sum_{i=1}^{N}a_i\le N-1,\\[2pt]
 \displaystyle\sum_{i=1}^{N}b_i\le N\rho_{n,\zeta},\\[2pt]
 \displaystyle\sum_{i=1}^{N}b_i+n\xi_{n,\zeta}(u+v)\le1+n\xi_{n,\zeta}
 \end{array}\right\}.
\end{equation}
Write $P_n=P_{n,0}$. If $\zeta\le r_n$, then $\rho_{n,\zeta}=r_n$ and $\xi_{n,\zeta}=0$, so \eqref{eq:hard-polytope} reduces to $u+v\le1$, $\sum_i a_i\le N-1$, and $\sum_i b_i\le1$, and $P_{n,\zeta}=P_n$.

Define the following four multilinear polynomials:
\begin{align}\label{eq:four-components}
 F_{\mathrm{cut}}&=u(1-v)+v(1-u),
 &M_n&=\sum_{i=1}^{N}b_i(1-a_i),\nonumber\\
 J_n&=\frac{\sum_{i=1}^{N}a_i}{N-1}(1-uv),
 &H_n&=(1-u)\left[1-\prod_{i\in I_L}(1-b_i)\right]
       +(1-v)\left[1-\prod_{i\in I_R}(1-b_i)\right].
\end{align}
We call $F_{\mathrm{cut}}$ the central cut, $M_n$ the matched term, and $H_n$ the coverage term. For weights $\omega=(\eta,\theta,\lambda,\mu)$ in the simplex
$\Delta_4=\{\omega\in\R_+^4:\eta+\theta+\lambda+\mu=1\}$, let
\begin{equation}\label{eq:four-objective}
 F_{n,\omega}=\eta F_{\mathrm{cut}}+\theta M_n+\lambda J_n+\mu H_n.
\end{equation}
Let $\mathcal G$ be the group generated by the permutations of matched pairs within $I_L$, the permutations of matched pairs within $I_R$, and one involution that maps $I_L$ onto $I_R$ through a fixed bijection and exchanges $u$ with $v$; every element moves $a_i$ and $b_i$ together. The three coordinate orbits are the central coordinates, all $a$-coordinates, and all $b$-coordinates. Thus a symmetric point has $u=v=q$, $a_i=s$, and $b_i=t$. Set
\begin{equation}\label{eq:upper-sequence}
 \chi_n=1-(1-r_n)^n,\qquad
 \ell_n=n(1-r_n)^{n-1}.
\end{equation}
With $x=1/2-q$, Lemma~\ref{lem:upper-u0} bounds the objective at symmetric points with $s=0$ and $s=1-r_n$ by the following quadratics:
\begin{align}\label{eq:parameterized-branches}
 Q_{0,n,\zeta}(x;\omega)
 &=\eta(1/2-2x^2)+\theta(1+N\xi_{n,\zeta}x)
   +\mu(1+2x)(\chi_n+\ell_n\xi_{n,\zeta}x),\nonumber\\
 Q_{1,n,\zeta}(x;\omega)
 &=\eta(1/2-2x^2)+\theta(r_n+\xi_{n,\zeta}x)
   +\lambda(3/4+x-x^2)
   +\mu(1+2x)(\chi_n+\ell_n\xi_{n,\zeta}x).
\end{align}
Optimizing these quadratic bounds defines
\begin{equation}\label{eq:upper-profile}
 \begin{split}
 V_n(\zeta)&=\min_{\omega\in\Delta_4}
   \max_{j\in\{0,1\}}\max_{0\le x\le1/2}Q_{j,n,\zeta}(x;\omega),\\
 U(\zeta)&=\inf_{n\ge2}V_n(\zeta)\quad(0\le\zeta<1/2),
 \qquad U(\zeta)=1/2\quad(1/2\le\zeta\le1).
 \end{split}
\end{equation}
For fixed $\omega$, each inner maximum is attained at $x=0$, at $x=1/2$, or at the vertex of the quadratic when it lies in $[0,1/2]$.

\begin{lemma}[Four-component upper bound]\label{lem:upper-u0}\label{prop:finite-upper}\label{lem:upper-components}
For every $n\ge2$, $0\le\zeta<1/2$, and $\omega\in\Delta_4$, the pair $(F_{n,\omega},P_{n,\zeta})$ satisfies the hypotheses of Lemma~\ref{lem:continuous-refinement}. Its optimum is at least one and its symmetry gap is at most $\max_{j,x}Q_{j,n,\zeta}(x;\omega)$. For $\zeta\le r_n$, its optimum is one and this bound is the exact symmetry gap.

Consequently $\aor(\zeta)\le U(\zeta)<1/2$ for $\zeta<1/2$, and every fixed coefficient $U(\zeta)+\varepsilon$ requires exponentially many exact value-and-gradient queries. The function $U$ is nondecreasing.
\end{lemma}
\begin{proof}
Each polynomial in \eqref{eq:four-components} is nonnegative on the cube, since every factor lies in $[0,1]$, and multilinear, so its diagonal second derivatives vanish. The nonzero mixed derivatives of the first three components are
\begin{equation}\label{eq:upper-component-derivatives}
 \begin{gathered}
 \partial_{uv}F_{\mathrm{cut}}=-2,\qquad \partial_{a_i b_i}M_n=-1,\\
 \partial_{uv}J_n=-\frac{\sum_i a_i}{N-1},\qquad
 \partial_{u a_i}J_n=-\frac{v}{N-1},\qquad
 \partial_{v a_i}J_n=-\frac{u}{N-1}.
 \end{gathered}
\end{equation}
For the left summand of $H_n$ and distinct $i,j\in I_L$, they are
\[
 \partial_{u b_i}H_n=-\prod_{k\in I_L\setminus\{i\}}(1-b_k),\qquad
 \partial_{b_i b_j}H_n=-(1-u)\prod_{k\in I_L\setminus\{i,j\}}(1-b_k),
\]
and the right summand has the corresponding formulas with $v$ and $I_R$. All these derivatives are nonpositive on the cube, so every component is DR-submodular, and so is the nonnegative combination $F_{n,\omega}$.

The identity
\begin{equation}\label{eq:upper-diagonal-identity}
 \rho_{n,\zeta}(1+\xi_{n,\zeta})=r_n+\xi_{n,\zeta}/2
\end{equation}
holds with both sides equal to $r_n$ when $\zeta\le r_n$. When $\zeta>r_n$, we have $\rho_{n,\zeta}=\zeta$ and $\zeta(1+\xi_{n,\zeta})-r_n-\xi_{n,\zeta}/2=(\zeta-r_n)-\xi_{n,\zeta}(1/2-\zeta)=0$. At the diagonal point $\rho_{n,\zeta}\one$, the central sum is $2\rho_{n,\zeta}\le1$ because $\zeta<1/2$ and $r_n\le1/4$, the $a$-sum is $N\rho_{n,\zeta}\le N/2\le N-1$, the $b$-sum is $N\rho_{n,\zeta}$, and the left side of the last constraint is $2n\rho_{n,\zeta}(1+\xi_{n,\zeta})=1+n\xi_{n,\zeta}$ by \eqref{eq:upper-diagonal-identity}. Hence $\rho_{n,\zeta}\one\in P_{n,\zeta}$, and $\zeta\one\in P_{n,\zeta}$ by down-closedness since $\zeta\le\rho_{n,\zeta}$.

Choose $j\in I_R$, set $u=1,v=0,b_j=1,a_j=0$, and set $a_i=1,b_i=0$ for $i\ne j$. This vertex satisfies $u+v=1$, $\sum_ia_i=N-1$, $\sum_ib_i=1\le N\rho_{n,\zeta}$, and $1+n\xi_{n,\zeta}\le1+n\xi_{n,\zeta}$, so this point, which we call the asymmetric witness, is feasible for every $\zeta$. All four components equal one there, hence $O\ge1$. When $\zeta\le r_n$, each component is at most one throughout $P_n$, because
\[
 F_{\mathrm{cut}}\le u+v,\qquad M_n\le\sum_i b_i,\qquad
 J_n\le\frac{\sum_i a_i}{N-1},\qquad
 H_n\le\sum_i b_i,
\]
where the last inequality follows from $1-\prod_{i\in I}(1-b_i)\le\sum_{i\in I}b_i$. Thus $O=1$ in this range.

Each component is invariant under $\mathcal G$. Membership in \eqref{eq:hard-polytope} depends only on $u+v$, $\sum_ia_i$, and $\sum_ib_i$, which $\mathsf T$ preserves, so $y\in P_{n,\zeta}$ if and only if $\mathsf Ty\in P_{n,\zeta}$; in particular $S\le O$. At a symmetric point, the constraints of \eqref{eq:hard-polytope} become
\begin{equation}\label{eq:upper-symmetric-region}
 0\le q\le\frac12,\qquad
 0\le s\le1-r_n,\qquad
 0\le t\le\min\{\rho_{n,\zeta},r_n+\xi_{n,\zeta}(1/2-q)\},
\end{equation}
where the last bound divides $Nt+2n\xi_{n,\zeta}q\le1+n\xi_{n,\zeta}$ by $N=1/r_n$. The symmetric objective is
\begin{equation}\label{eq:upper-symmetric-four}
 2\eta q(1-q)+\theta Nt(1-s)
 +\lambda\frac{Ns}{N-1}(1-q^2)
 +2\mu(1-q)\{1-(1-t)^n\}.
\end{equation}
It is nondecreasing in $t$ and affine in $s$, and the function $\Phi_n(t)=1-(1-t)^n$ in its last term is concave with $\Phi_n(r_n)=\chi_n$ and $\Phi_n'(r_n)=\ell_n\ge0$. Since $t-r_n\le\xi_{n,\zeta}x$ by \eqref{eq:upper-symmetric-region}, its tangent inequality gives
\begin{equation}\label{eq:upper-coverage-tangent}
 \Phi_n(t)\le\Phi_n(r_n)+\Phi_n'(r_n)(t-r_n)
 \le\chi_n+\ell_n\xi_{n,\zeta}x,
 \qquad x=1/2-q.
\end{equation}
In terms of $x$, $2q(1-q)=1/2-2x^2$, $2(1-q)=1+2x$, and $1-q^2=3/4+x-x^2$. At $s=0$, the matched term is $\theta Nt\le\theta(1+N\xi_{n,\zeta}x)$ and the $J_n$ term vanishes. At $s=1-r_n$, the matched term is $\theta Nr_nt=\theta t\le\theta(r_n+\xi_{n,\zeta}x)$ and the $J_n$ term is $\lambda(1-q^2)$, because $N(1-r_n)=N-1$. With \eqref{eq:upper-coverage-tangent}, these two cases give \eqref{eq:parameterized-branches}. Since the objective is affine in $s\in[0,1-r_n]$, its value at every intermediate $s$ lies between its endpoint values, so $\max_{j,x}Q_{j,n,\zeta}(x;\omega)$ bounds $S$, and also $\gamma=S/O$ because $O\ge1$. When $\zeta\le r_n$, we have $\xi_{n,\zeta}=0$ and $\rho_{n,\zeta}=r_n$; the point $t=r_n$ is feasible for every $q$ and $s$, and there $Nt=1$ and $\Phi_n(t)=\chi_n$, so every bound above is attained and $\gamma=\max_{j,x}Q_{j,n,\zeta}(x;\omega)$. Finally, the symmetric point with $q=1/4$, $s=(1-r_n)/2$, and $t=r_n$ is feasible for every $\zeta$ and gives every component a positive value, which proves $S>0$.

The function $\omega\mapsto\max_{j,x}Q_{j,n,\zeta}(x;\omega)$ is continuous, since each $Q_{j,n,\zeta}$ is jointly continuous on $[0,1/2]\times\Delta_4$, so the minimum defining $V_n(\zeta)$ is attained on the compact simplex. For a minimizing $\omega$, Lemma~\ref{lem:continuous-refinement} gives $\aor(\zeta)\le\gamma\le V_n(\zeta)$, and taking the infimum over $n$ gives $\aor(\zeta)\le U(\zeta)$. Given $\varepsilon>0$, choose $n$ with $V_n(\zeta)<U(\zeta)+\varepsilon/2$ and a minimizing $\omega$. Then $\gamma+\varepsilon/2<U(\zeta)+\varepsilon$, and Lemma~\ref{lem:continuous-refinement} with $\varepsilon/2$ shows that expected coefficient $U(\zeta)+\varepsilon$ requires exponentially many queries.

Each quadratic in \eqref{eq:parameterized-branches} is nondecreasing in $\xi_{n,\zeta}$ for $0\le x\le1/2$, and $\xi_{n,\zeta}$ is nondecreasing in $\zeta$ on $[0,1/2)$. Taking the maxima over $j$ and $x$, the minimum over $\omega$, and the infimum over $n$ preserves this monotonicity, so $U$ is nondecreasing on $[0,1/2)$. To prove $U(\zeta)<1/2$, take any $n\ge2$, put $z=\ell_n\xi_{n,\zeta}\ge0$ and $\upsilon_n=1-2\chi_n$, and choose
\begin{equation}\label{eq:upper-two-feasible-weights}
 \theta=\lambda=0,\qquad
 \mu=\frac{2\upsilon_n}{(1+z)\{z+2(1-\chi_n)\}},\qquad
 \eta=1-\mu.
\end{equation}
Bernoulli's inequality gives $1-\chi_n=(1-1/(2n))^n>1/2$, and $1-r<e^{-r}$ gives $1-\chi_n<e^{-1/2}<2/3$. Hence $1/3<\chi_n<1/2$, so $0<\upsilon_n<1/3$ and $0<\mu\le\upsilon_n/(1-\chi_n)=1-\chi_n/(1-\chi_n)<1/2$. With these weights the two quadratics coincide, and expanding shows that both equal
\begin{equation}\label{eq:upper-strict-half}
 \frac12-\frac{\upsilon_n^2}{2(1+z)^2}
 -\frac{2(z+2\chi_n)}{z+2(1-\chi_n)}
   \left(x-\frac{\upsilon_n}{2(1+z)}\right)^2.
\end{equation}
The vertex $\upsilon_n/(2(1+z))$ lies in $(0,1/6)$, so
\begin{equation}\label{eq:upper-two-feasible}
 V_n(\zeta)\le
 \frac12-\frac{(1-2\chi_n)^2}{2(1+\ell_n\xi_{n,\zeta})^2}<\frac12.
\end{equation}
Hence $U(\zeta)<1/2=U(1/2)$ for $\zeta<1/2$, and $U$ is nondecreasing on $[0,1]$.
\end{proof}

The endpoint $\zeta=0$ admits a closed-form evaluation. Define
\begin{equation}\label{eq:upper-constants}
 \chi=1-e^{-1/2},\qquad
 q_0=\sqrt{\frac{3\chi-1}{2\chi}},\qquad
 \tau=2q_0(1-q_0),\qquad q_1=\frac{1-\tau}{1+\tau},
\end{equation}
and set
\begin{equation}\label{eq:upper-zero}
 \beta_*:=\frac{4\tau}{1+6\tau+\tau^2}.
\end{equation}
Since $3\chi-1=2-3e^{-1/2}$ and $2\chi=2-2e^{-1/2}$, these are the constants $q_*=q_0$ and $\tau_*=\tau$ of \eqref{eq:intro-upper-endpoint}; numerically $q_0\approx0.47880$, $q_1\approx0.33413$, $\tau\approx0.49910$, and $\beta_*\approx0.470439$.

\begin{corollary}[Offline upper bound at zero diagonal level]\label{cor:upper-endpoint}
The endpoint satisfies $U(0)=\beta_*$. Hence $\aor(0)\le\beta_*$, and every fixed expected coefficient $\beta_*+\varepsilon$ requires exponentially many exact value-and-full-gradient queries.
\end{corollary}
\begin{proof}
The alternating series for $e^{-1/2}$ gives $29/48<e^{-1/2}<5/8$, so $3/8<\chi<19/48<2/5$. Since $q_0^2=3/2-1/(2\chi)$, this gives $1/6<q_0^2<1/4$, so $2/5<1/\sqrt6<q_0<1/2$ and $0<\tau<1/2$. Then $0<q_1<1$, and $q_1<q_0$ is equivalent to $1-q_0<\tau(1+q_0)=2q_0(1-q_0)(1+q_0)$, that is, to $2q_0(1+q_0)>1$, which holds because $q_0>2/5$. Thus $0<q_1<q_0<1/2$.

At $\zeta=0$ we have $\xi_{n,0}=0$. Writing $q=1/2-x$, the bounds in \eqref{eq:parameterized-branches} become $2\eta q(1-q)+\theta+2\mu\chi_n(1-q)$ and $2\eta q(1-q)+\theta r_n+\lambda(1-q^2)+2\mu\chi_n(1-q)$. Since $r_n\to0$ and $\chi_n\to\chi$, they converge to
\begin{equation}\label{eq:base-branches}
 \begin{split}
 B_0(q)&=2\eta q(1-q)+\theta+2\mu\chi(1-q),\\
 B_1(q)&=2\eta q(1-q)+\lambda(1-q^2)+2\mu\chi(1-q).
 \end{split}
\end{equation}
Consider the first expression at $q_0$ with probability $\beta_*$ and the second at $q_1$ with probability $1-\beta_*$. We claim that this distribution gives expectation $\beta_*$ to each of the four component contributions. The definitions \eqref{eq:upper-constants} and \eqref{eq:upper-zero} imply
\[
 1-q_1^2=\frac{4\tau}{(1+\tau)^2},\quad
 \frac{\beta_*}{1-\beta_*}=\frac{4\tau}{(1+\tau)^2},\quad
 \chi(3-2q_0^2)=1.
\]
The expectations for $M_n$ and $J_n$ are therefore $\beta_*$ and $(1-\beta_*)(1-q_1^2)=\beta_*$. Moreover, $2q_1(1-q_1)=4\tau(1-\tau)/(1+\tau)^2$ and $1-q_1=2\tau/(1+\tau)$, so $(1-\beta_*)2q_1(1-q_1)=\beta_*(1-\tau)$ and $(1-\beta_*)(1-q_1)=\beta_*(1+\tau)/2$. With $\tau=2q_0-2q_0^2$, the expectations for $F_{\mathrm{cut}}$ and $H_n$ are
\begin{align}\label{eq:endpoint-component-expectations}
 \beta_*\tau+(1-\beta_*)2q_1(1-q_1)
 &=\beta_*\tau+\beta_*(1-\tau)=\beta_*,\nonumber\\
 2\chi\{\beta_*(1-q_0)+(1-\beta_*)(1-q_1)\}
 &=\beta_*\chi(3-2q_0+\tau)=\beta_*\chi(3-2q_0^2)=\beta_*.
\end{align}
For every $\omega\in\Delta_4$, the maximum of $B_0,B_1$ over $q\in[0,1/2]$ is at least this expectation, $\beta_*(\eta+\theta+\lambda+\mu)=\beta_*$. The finite expressions exceed their limiting counterparts pointwise, since $r_n>0$ and $\chi_n>\chi$ (the latter because $1-r_n<e^{-r_n}$). Hence $V_n(0)\ge\beta_*$ for every $n$, proving $U(0)\ge\beta_*$.

For the reverse inequality, define weights by
\begin{equation}\label{eq:endpoint-weights}
 \begin{split}
 Z&=1+2(q_0/q_1-1)(2-q_0q_1)+(1-2q_0)/\chi,\\
 \eta_*&=Z^{-1},\qquad
 \lambda_*=2\eta_*(q_0/q_1-1),\\
 \theta_*&=\lambda_*(1-q_0q_1),\qquad
 \mu_*=\eta_*(1-2q_0)/\chi.
 \end{split}
\end{equation}
Since $0<q_1<q_0<1/2$, all four weights are positive, and their sum is $\eta_*Z=1$. Numerically $(\eta_*,\theta_*,\lambda_*,\mu_*)\approx(0.3702,0.2693,0.3206,0.0399)$. The definitions give
\begin{equation}\label{eq:endpoint-stationarity}
 \eta_*-\mu_*\chi=2\eta_*q_0=(2\eta_*+\lambda_*)q_1,
 \qquad \theta_* =\lambda_*(1-q_0q_1).
\end{equation}
With these weights,
\[
 B_0(q)=\theta_*+2\mu_*\chi+2(\eta_*-\mu_*\chi)q-2\eta_*q^2,\qquad
 B_1(q)=\lambda_*+2\mu_*\chi+2(\eta_*-\mu_*\chi)q-(2\eta_*+\lambda_*)q^2,
\]
so the first identity in \eqref{eq:endpoint-stationarity} places the vertices of $B_0,B_1$ at $q_0,q_1$. Their maximal values differ by
\[
 \theta_*-\lambda_*+
 (\eta_*-\mu_*\chi)^2
 \left(\frac1{2\eta_*}-\frac1{2\eta_*+\lambda_*}\right)
 =\theta_*-\lambda_*+\lambda_*q_0q_1=0,
\]
where the middle expression follows from \eqref{eq:endpoint-stationarity} and $2\eta_*+\lambda_*=2\eta_*q_0/q_1$. Their common maximal value equals $\beta_*B_0(q_0)+(1-\beta_*)B_1(q_1)$, which is $\beta_*$ by \eqref{eq:endpoint-component-expectations}. Since both vertices lie in $(0,1/2)$,
\begin{equation}\label{eq:endpoint-squares}
 B_0(q)=\beta_*-2\eta_*(q-q_0)^2,\qquad
 B_1(q)=\beta_*-(2\eta_*+\lambda_*)(q-q_1)^2,
\end{equation}
and the maximum of both over $[0,1/2]$ is $\beta_*$.

For these weights, the finite expressions exceed $B_0,B_1$ by $2\mu_*(\chi_n-\chi)(1-q)$ and $\theta_*r_n+2\mu_*(\chi_n-\chi)(1-q)$, respectively, with $0\le1-q\le1$. Since $\zeta=0\le r_n$, Lemma~\ref{lem:upper-u0} identifies the exact finite symmetry gap $\gamma_n$ of $(F_{n,\omega_*},P_n)$ with the maximum of these expressions, and the first part of the proof gives $\gamma_n\ge V_n(0)\ge\beta_*$. Hence
\begin{equation}\label{eq:finite-endpoint-gap}
 \beta_*\le\gamma_n\le\beta_*+\frac{\theta_*}{2n}
                    +\frac{\mu_*e^{-1/2}}{2(2n-1)}.
\end{equation}
For the last term, expanding the logarithm gives
\[
 -x-\frac{x^2}{2(1-x)}\le\log(1-x)\le-x
 \quad(0<x<1),
\]
because $\sum_{k\ge2}x^k/k\le\frac12\sum_{k\ge2}x^k$. At $x=1/(2n)$, multiplication by $n$ gives $n\log(1-r_n)\ge-1/2-1/(4(2n-1))$, so $1-\chi_n\ge e^{-1/2}\{1-1/(4(2n-1))\}$ and $0<\chi_n-\chi\le e^{-1/2}/[4(2n-1)]$. Letting $n\to\infty$ in \eqref{eq:finite-endpoint-gap} proves $U(0)\le\beta_*$. Lemma~\ref{lem:upper-u0} then gives $\aor(0)\le\beta_*$ and the query lower bound.
\end{proof}

The central-cut and coverage construction of \citet[Appendix~E.2]{oveis2011} gives the limiting value $2(e^{-1/2}-e^{-1})=2\chi(1-\chi)$, and the improvement in \eqref{eq:upper-zero} is strict. The function $t\mapsto4t/(1+6t+t^2)$ has derivative $4(1-t^2)/(1+6t+t^2)^2>0$ on $[0,1)$ and equals $8/17$ at $t=1/2$, so $0<\tau<1/2$ gives $\beta_*<8/17$. The inequality $e^{1/2}>1+1/2+1/8=13/8$ gives $\chi>5/13$, and since $2\chi(1-\chi)$ increases for $\chi<1/2$,
\[
 \beta_*<\frac8{17}<\frac{80}{169}<2\chi(1-\chi).
\]
The improvement comes from $M_n$ and $J_n$, which equal one at the asymmetric witness while a larger symmetric $a$-coordinate decreases $M_n$ and increases $J_n$; \eqref{eq:endpoint-component-expectations}--\eqref{eq:endpoint-squares} show how the two branches balance.

\subsection{Exact evaluation of the four-component bound}\label{sec:upper-evaluation}
\begin{proposition}[Evaluation of $U$]\label{prop:upper-evaluation}
At zero, $U(0)=\beta_*$ as defined in \eqref{eq:upper-zero}. For $0<\zeta<1/8$, let
\begin{equation}\label{eq:upper-index-range}
 m_\zeta=\left\lfloor\frac1{2\zeta}\right\rfloor,
 \qquad N_\zeta=\left\lceil\frac5{2\zeta}-4\right\rceil.
\end{equation}
Then
\begin{equation}\label{eq:upper-finite-evaluation}
 U(\zeta)=\min_{m_\zeta\le n\le N_\zeta}V_n(\zeta).
\end{equation}
Put $\zeta_{43}=3329/24151$ and $\zeta_{32}=103/506$. For $1/8\le\zeta\le1$,
\begin{equation}\label{eq:upper-rational-evaluation}
 U(\zeta)=
 \begin{cases}
 \displaystyle\frac12-
 \frac{353^2(1-2\zeta)^2}{32(169+1720\zeta)^2},
       &\frac18\le\zeta\le\zeta_{43},\\[6pt]
 \displaystyle\frac{11375}{23328},
       &\zeta_{43}\le\zeta\le\frac16,\\[6pt]
 \displaystyle\frac12-
 \frac{17^2(1-2\zeta)^2}{18(11+78\zeta)^2},
       &\frac16\le\zeta\le\zeta_{32},\\[6pt]
 \displaystyle\frac{63}{128},
       &\zeta_{32}\le\zeta\le\frac14,\\[6pt]
 \displaystyle\frac12-
 \frac{(1-2\zeta)^2}{8(1+4\zeta)^2},
       &\frac14\le\zeta<\frac12,\\[6pt]
 \displaystyle\frac12,&\frac12\le\zeta\le1.
 \end{cases}
\end{equation}
The adjacent formulas agree at each breakpoint, and $U(\zeta)\to\beta_*$ as $\zeta\downarrow0$.
\end{proposition}
The proof is in Appendix~\ref{app:upper-evaluation}. On each interval of \eqref{eq:upper-rational-evaluation}, the weights \eqref{eq:upper-two-feasible-weights} for $n=4$, $3$, or $2$ give the stated value as an upper bound, and a distribution over the two branches at one value of $x$ gives the same value as a lower bound for every $n\ge2$.

\subsection{A closed-form bound near the half-diagonal level}\label{sec:upper-c2}
Taking $n=2$ and $\theta=\lambda=0$ in \eqref{eq:four-objective} gives the objective $(1-\mu)F_{\mathrm{cut}}+\mu H_2$ on $P_{2,\zeta}$. For $n=2$ the coverage term $1-(1-t)^2=2t-t^2$ is quadratic, and choosing the tangent point and $\mu$ so that the resulting bound is attained at a feasible symmetric point gives the exact symmetry gap $C(\zeta)$.

\begin{lemma}[Explicit geometric upper bound]\label{prop:upper-near}
For $0\le\zeta<1/2$, an instance on $P_{2,\zeta}$ with objective $(1-\mu)F_{\mathrm{cut}}+\mu H_2$ for a suitable $\mu\in(0,2/9]$ has optimum one and exact symmetry gap
\begin{equation}\label{eq:upper-n2-function}
 C(\zeta)=
 \begin{cases}
 63/128,&0\le\zeta\le1/4,\\
 \displaystyle\frac12-\frac{(1-2\zeta)^2}{2(1+2\sqrt\zeta)^4},
       &1/4\le\zeta<1/2.
 \end{cases}
\end{equation}
Thus $\aor(\zeta)\le C(\zeta)$, and every fixed coefficient $C(\zeta)+\varepsilon$ requires exponentially many exact value-and-gradient queries.
\end{lemma}
\begin{proof}
First let $1/4\le\zeta<1/2$, so that $\rho_{2,\zeta}=\zeta$ and $\xi_{2,\zeta}=(\zeta-1/4)/(1/2-\zeta)$. Define
\begin{equation}\label{eq:upper-n2-parameters}
 \begin{gathered}
 x_* =\frac{1/2-\zeta}{(1+2\sqrt\zeta)^2},\quad
 t_* =\frac{\sqrt\zeta}{1+2\sqrt\zeta},\quad
 q_* =\frac12-x_*,\\
 d_* =2(1-t_*),\qquad
 \mu=\frac{4x_*}{(1+2x_*)(1+d_*\xi_{2,\zeta})}.
 \end{gathered}
\end{equation}
With $a=\sqrt\zeta\in[1/2,1/\sqrt2)$, the derivative of $x_*=(1/2-a^2)/(1+2a)^2$ in $a$ is $-2(1+a)/(1+2a)^3<0$, so $x_*$ decreases from its value $1/16$ at $\zeta=1/4$ and remains positive. Also $x_*\le1/2-\zeta$ because $(1+2\sqrt\zeta)^2\ge1$. Since $d_*\xi_{2,\zeta}\ge0$ and $4x/(1+2x)$ increases, $0<\mu\le4x_*/(1+2x_*)\le2/9$. The definitions imply
\begin{equation}\label{eq:upper-n2-identities}
 t_* =\frac14+\xi_{2,\zeta}x_*,\qquad
 2t_*-t_*^2=q_*,\qquad
 \mu(1+d_*\xi_{2,\zeta})=\frac{1-2q_*}{1-q_*}.
\end{equation}
For the first, $\xi_{2,\zeta}x_*=(a^2-1/4)/(1+2a)^2=(2a-1)/[4(1+2a)]=t_*-1/4$. For the second, both sides equal $a(2+3a)/(1+2a)^2$. For the third, both sides equal $4x_*/(1+2x_*)$, since $1-2q_*=2x_*$ and $1-q_*=1/2+x_*$.

By \eqref{eq:upper-symmetric-four} with $n=2$ and $\theta=\lambda=0$, and with $q=1/2-x$, the symmetric objective is $(1+2x)[(1-\mu)(1/2-x)+\mu(2t-t^2)]$, and \eqref{eq:upper-symmetric-region} gives $t\le1/4+\xi_{2,\zeta}x$. Since $2t-t^2$ is concave with derivative $d_*\ge0$ at $t_*$, and $t-t_*\le\xi_{2,\zeta}(x-x_*)$ by the first identity in \eqref{eq:upper-n2-identities}, the second identity yields
\begin{equation}\label{eq:upper-n2-tangent}
 2t-t^2\le2t_*-t_*^2+d_*(t-t_*)
 \le q_*+d_*\xi_{2,\zeta}(x-x_*).
\end{equation}
Put $y=x-x_*$, so that $1+2x=2(1-q_*+y)$ and $1/2-x=q_*-y$. The symmetric objective is at most
\begin{align}\label{eq:upper-n2-square}
 &(1+2x)\left[(1-\mu)(1/2-x)
       +\mu\{q_*+d_*\xi_{2,\zeta}(x-x_*)\}\right]\nonumber\\
 &\quad=2(1-q_*+y)\left[q_*-\frac{q_*}{1-q_*}y\right]
 =2q_*(1-q_*)-\frac{2q_*}{1-q_*}y^2
 \le\frac12-2x_*^2.
\end{align}
The first equality uses $(1-\mu)(q_*-y)+\mu q_*+\mu d_*\xi_{2,\zeta}y=q_*-\{1-\mu(1+d_*\xi_{2,\zeta})\}y$ and the last identity in \eqref{eq:upper-n2-identities}; the final step uses $2q_*(1-q_*)=1/2-2x_*^2$. Since $x_*\le1/2-\zeta$ and $\xi_{2,\zeta}(1/2-\zeta)=\zeta-1/4$, the first identity gives $t_*\le\zeta=\rho_{2,\zeta}$. Hence $q=q_*$, $s=0$, and $t=t_*=1/4+\xi_{2,\zeta}x_*$ form a feasible symmetric point by \eqref{eq:upper-symmetric-region}, and at this point $t=t_*$ and $y=0$, so every inequality in \eqref{eq:upper-n2-tangent}--\eqref{eq:upper-n2-square} holds with equality. The symmetric maximum is therefore $1/2-2x_*^2$, which is the right-hand side of \eqref{eq:upper-n2-function}.

The common asymmetric witness of Lemma~\ref{lem:upper-u0} has value one. Conversely, $F_{\mathrm{cut}}\le u+v$ and $H_2\le(1-u)+(1-v)$, so everywhere in $P_{2,\zeta}$
\[
 (1-\mu)F_{\mathrm{cut}}+\mu H_2
 \le(1-\mu)(u+v)+\mu(2-u-v)
 =2\mu+(1-2\mu)(u+v)\le1,
\]
using $u+v\le1$ and $\mu\le1/2$. Thus the optimum is one, and the symmetry gap is $C(\zeta)$. At $\zeta=1/4$, the formula gives $1/2-2/16^2=63/128$. For $\zeta\le1/4=r_2$, we have $P_{2,\zeta}=P_2=P_{2,1/4}$, so the instance constructed for $\zeta=1/4$ applies and contains $\zeta\one$. By Lemma~\ref{lem:upper-u0} with $\omega=(1-\mu,0,0,\mu)$, these instances satisfy the hypotheses of Lemma~\ref{lem:continuous-refinement}, which gives $\aor(\zeta)\le C(\zeta)$ and the query lower bound.
\end{proof}

\subsection{The combined offline upper bound}\label{sec:upper-results}
The proof of Theorem~\ref{thm:upper} shows that a two-coordinate cut gives the bound $1/2$ for every $\zeta$. Extend $C(\zeta)=1/2$ for $\zeta\ge1/2$ and combine the two constructions as
\begin{equation}\label{eq:strengthened-profile}
 \boxed{\aup(\zeta)=\min\{U(\zeta),C(\zeta)\}
 =\begin{cases}
 U(\zeta),&0\le\zeta\le1/4,\\
 C(\zeta),&1/4\le\zeta\le1.
 \end{cases}}
\end{equation}
For $\zeta\le1/4$, monotonicity of $U$ and Proposition~\ref{prop:upper-evaluation} give $U(\zeta)\le U(1/4)=63/128=C(\zeta)$. For $1/4\le\zeta<1/2$, comparing \eqref{eq:upper-rational-evaluation} with \eqref{eq:upper-n2-function} shows that $C(\zeta)\le U(\zeta)$ is equivalent to $(1+2\sqrt\zeta)^4\le4(1+4\zeta)^2$, which holds because
\[
 2(1+4\zeta)-(1+2\sqrt\zeta)^2=(2\sqrt\zeta-1)^2\ge0.
\]

\begin{theorem}[Offline upper bound]\label{thm:upper}\label{thm:upper-parameterized}\label{cor:upper-profile}
For every $\zeta\in[0,1]$, $\aor(\zeta)\le\aup(\zeta)$. For each fixed $\varepsilon>0$, uniformly attaining expected coefficient $\aup(\zeta)+\varepsilon$ requires $\exp(\Omega(d))$ exact value-and-full-gradient queries. The hard instances have explicit down-closed constraints with polynomial-time exact Euclidean projection in real arithmetic. Their values and gradient norms are bounded uniformly in the growing dimension, and their optima have a common positive lower bound. They can be chosen with diagonal radius $\zeta(K)=\zeta$.

The function $\aup$ is nondecreasing, with $\aup(0)=\beta_*$, $\aup(\zeta)<1/2$ for $\zeta<1/2$, and $\aup(\zeta)=1/2$ for $\zeta\ge1/2$.
\end{theorem}
\begin{proof}
On the square $[0,1]^2$, which is down-closed and contains every $\zeta\one$, the function $F(y_1,y_2)=y_1+y_2-2y_1y_2$ is a nonnegative multilinear DR-submodular polynomial with $\partial_{12}F=-2$ and optimum $F(1,0)=1$. Under exchange of its two coordinates, its symmetric maximum is $\max_{0\le q\le1}2q(1-q)=1/2$. Lemma~\ref{lem:continuous-refinement} gives the bound $1/2$, with the query lower bound, on the replicated cubes $[0,1]^{2m}$. This justifies the constant extensions of $U$ and $C$.

For $\zeta<1/2$, Lemmas~\ref{lem:upper-u0} and~\ref{prop:upper-near} give $\aor(\zeta)\le\min\{U(\zeta),C(\zeta)\}$ and exponential hardness above this minimum. The bounds $M,B,v$ supplied in the proof of Lemma~\ref{lem:continuous-refinement} do not depend on $m$.

For the seeds $P_{n,\zeta}$, the replicated polytope $P_m$ of \eqref{eq:refinement-symmetry} is obtained by multiplying each right-hand side in \eqref{eq:hard-polytope} by $m$ and taking each coordinate sum over all rows. Its coordinates form three groups (central, $a$, and $b$ coordinates), and apart from the box every constraint involves only the three group sums. Let $y$ be the point to be projected. The Karush--Kuhn--Tucker conditions of the projection, a strictly convex quadratic program with linear constraints, show that the projection has the form $z_i=\min\{1,\max\{0,y_i-h_g\}\}$ for $i$ in group $g$, where each shift $h_g$ is a combination of the multipliers of the sum constraints. Every point of this form lies in the box, so the projection minimizes $\sum_i(z_i-y_i)^2$ over shifts $(h_1,h_2,h_3)$ subject to the sum constraints. The breakpoints $y_i$ and $y_i-1$ divide each shift axis into at most $2d+1$ intervals. On each product of intervals, the group sums are affine in the shifts and the objective is a convex quadratic, since each unclipped coordinate contributes $h_g^2$ and each clipped coordinate a constant. The resulting problem has three variables and a fixed number of linear inequalities, so it is solved exactly by enumerating its active constraints. Comparing the minima over the $O(d^3)$ products of intervals gives the exact projection with polynomially many arithmetic operations. For the replicated full cube, projection is coordinatewise clipping.

To prescribe the diagonal radius, replace a hard set $K$ by $K\times[0,\zeta]$ and extend the objective so that it does not depend on the added coordinate. Since $K$ contains $\zeta\one$, the new diagonal radius is $\zeta$. The new set is compact, convex, and down-closed, the optimum, value bound, and gradient bound are unchanged, the new gradient component is zero, and projection separates into the original projection and one clipping operation.

Finally, $U$ is nondecreasing by Lemma~\ref{lem:upper-u0}. The quantity $(1/2-\zeta)/(1+2\sqrt\zeta)^2$ decreases on $[1/4,1/2)$, as shown in the proof of Lemma~\ref{prop:upper-near}, so $C$ is nondecreasing on $[1/4,1/2)$; its two formulas agree at $1/4$, and $C(\zeta)<1/2=C(1/2)$ for $\zeta<1/2$. The minimum $\aup$ is therefore nondecreasing. At zero, $U(0)=\beta_*<8/17<63/128=C(0)$, so $\aup(0)=\beta_*$. The strict inequality below $1/2$ follows from Lemma~\ref{lem:upper-u0}, and both functions equal $1/2$ for $\zeta\ge1/2$.
\end{proof}

\section{Offline and Online Results Transfer}\label{sec:combined}
The online algorithms of Section~\ref{sec:feedback} and the offline upper bounds of Section~\ref{sec:upper} each imply a result in the other model. The online-to-batch procedure OTB gives offline algorithms with explicit sample complexity, and running an arbitrary online learner on a constant objective sequence transfers the offline upper bounds to online optimization.

\subsection{Offline approximation from OTB}\label{sec:offline}
The OTB procedure of \citet[Theorem~8]{pedramfar2024linear} runs an online algorithm on the constant sequence $f_t=f$ and returns one of its played actions, chosen uniformly at random, so the returned point is feasible. Below, $\epsilon$ denotes additive accuracy and $\varepsilon$ denotes multiplicative accuracy.

\begin{corollary}[Offline first-order approximation]\label{cor:offline}
Let $f$ and $K$ satisfy the assumptions of Corollary~\ref{thm:first-order}, and let $\epsilon>0$. For
\begin{equation}\label{eq:offline-budget}
 T_\epsilon=
 \max\!\left\{1,\left\lceil
 \frac{\alow(\zeta)^2D^2 V_{\rm gen}}{\epsilon^2}
 \right\rceil\right\},
\end{equation}
OTB uses $T_\epsilon$ gradient queries and $T_\epsilon$ projections and returns $p\in K$ with
\begin{equation}\label{eq:offline-additive}
 \E f(p)\ge
 \alow(\zeta)\bigl[\OPT+c_\zeta f(\zero)\bigr]-\epsilon.
\end{equation}
Here $V_{\rm gen}$ is the bound in \eqref{eq:sample-moments}. If a positive number $v\le\OPT$ is supplied, setting $\epsilon=\varepsilon v$ gives
\begin{equation}\label{eq:offline-multiplicative}
 \E f(p)\ge
 (\alow(\zeta)-\varepsilon)\OPT+
 \alow(\zeta)c_\zeta f(\zero)
\end{equation}
with the corresponding query count in \eqref{eq:offline-budget}.
\end{corollary}
\begin{proof}
Run Algorithm~\ref{alg:gradient} with $T=T_\epsilon$ on $f_t=f$, let $\tau$ be uniform on $[T]$ and independent of the run, and return $p=p_\tau$. For the constant sequence, the maximum in \eqref{eq:regret} equals $T\OPT$, so Corollary~\ref{thm:first-order} reads
\[
 \alow(\zeta)T\OPT-\E\sum_{t=1}^Tf(p_t)+\alow(\zeta)c_\zeta Tf(\zero)\le\alow(\zeta)D\sqrt{TV_{\rm gen}}.
\]
Dividing by $T$ and using the independence of $\tau$ gives
\begin{align}
 \E f(p_\tau)
 &=\frac1T\E\sum_{t=1}^T f(p_t)\nonumber\\
 &\ge\alow(\zeta)\bigl[\OPT+c_\zeta f(\zero)\bigr]
       -\alow(\zeta)D\sqrt{V_{\rm gen}/T}.
 \label{eq:offline}
\end{align}
Since $T\ge\alow(\zeta)^2D^2V_{\rm gen}/\epsilon^2$, the last term is at most $\epsilon$, which proves \eqref{eq:offline-additive}. Each round uses one gradient query and one projection. If $\epsilon=\varepsilon v$, then $\epsilon\le\varepsilon\OPT$, and \eqref{eq:offline-additive} gives \eqref{eq:offline-multiplicative}.
\end{proof}
The origin term in \eqref{eq:offline} comes from Lemma~\ref{lem:comparison}. The procedure never observes the earned rewards, and for fixed $d$, $D$, and gradient bounds its sample complexity is $O(\epsilon^{-2})$.

The same averaging applies to the value-feedback algorithms. With $\alpha=\alow(\zeta)$ and the constants $K_0,K_1,K_2$ from \eqref{eq:value-online-constants}, choose
\begin{equation}\label{eq:offline-value-budgets}
 T_\epsilon^{(1)}
 =\max\{1,\lceil(K_1/\epsilon)^4\rceil\},\qquad
 T_\epsilon^{(2)}
 =\max\left\{1,\left\lceil\left(\frac{K_2+\sqrt{K_2^2+4\epsilon K_0}}{2\epsilon}\right)^2\right\rceil\right\}.
\end{equation}
The one-query method uses $T_\epsilon^{(1)}$ noisy value responses and $T_\epsilon^{(1)}$ projections; the two-query method uses $2T_\epsilon^{(2)}$ exact value responses and $T_\epsilon^{(2)}$ projections. Dividing \eqref{eq:value-online-stochastic} and \eqref{eq:value-online-exact} by $T$ as in \eqref{eq:offline} bounds the average deficit by $K_1T^{-1/4}$ and $K_2T^{-1/2}+K_0T^{-1}$, respectively. The first is at most $\epsilon$ when $T\ge(K_1/\epsilon)^4$. For the second, put $u=\sqrt T$; the deficit is at most $\epsilon$ exactly when $\epsilon u^2-K_2u-K_0\ge0$, that is, when $u$ is at least the positive root $(K_2+\sqrt{K_2^2+4\epsilon K_0})/(2\epsilon)$, which holds for $T=T_\epsilon^{(2)}$. Hence $\E f(p)\ge\alow(\zeta)\OPT-\epsilon$ in either case, so the same coefficient is available offline from scalar values alone. A supplied scale $v\le\OPT$ again converts additive accuracy $\epsilon=\varepsilon v$ into multiplicative accuracy.

The quadratic and cubic online results extend to offline optimization with the same coefficients. For a fixed polynomial objective, one coefficient reconstruction serves every iteration, so all gradient or value queries occur at the start and the iterations use only arithmetic and projections.
\begin{corollary}[Offline quadratic and cubic achievability]\label{cor:polynomial-offline}
Let $m\in\{2,3\}$ and use the assumptions of Proposition~\ref{prop:polynomial-oracle}. Choose gradient or value feedback, with $(q_m,k_m,\sigma)$ as in \eqref{eq:polynomial-feedback-parameters}. The offline achievable coefficients are
\begin{equation}\label{eq:offline-polynomial-coefficients}
 \alpha_2(\zeta)=\frac12,\qquad
 \alpha_3(\zeta)=\max\!\left\{\frac8{17},
        \frac{2\sqrt{3-4r}}{(1+\sqrt{3-4r})^2}\right\},
 \qquad r=\min\{\zeta,1/2\}.
\end{equation}
Let $\alpha=\alpha_m(\zeta)$. For additive accuracy $\epsilon>0$, set
\begin{equation}\label{eq:polynomial-offline-batch}
 n=\max\!\left\{1,\left\lceil
       \left(\frac{8P_m k_m\sigma}{\epsilon}\right)^2
       \right\rceil\right\},
\end{equation}
and
\begin{equation}\label{eq:polynomial-otb-horizon}
 T=\max\!\left\{1,\left\lceil
       \left(\frac{\alpha D\Lambda_m}{\epsilon}
             (B\sqrt d+\epsilon)\right)^2
       \right\rceil\right\},
\end{equation}
with $\Lambda_m$ from \eqref{eq:polynomial-field-bound}. One initial reconstruction using $q_mn$ queries of the chosen type, followed by $T$ projected iterations with its field, returns a feasible point satisfying
\begin{equation}\label{eq:polynomial-offline-guarantee}
 \E f(p)\ge\alpha_m(\zeta)\OPT-\epsilon.
\end{equation}
With exact responses, $n=1$, the total query count is $Q_m$ for gradients or $\overline Q_m=\binom{d+m}{m}$ for values, and the output can be chosen deterministically. In this case the horizon may be reduced to
\begin{equation}\label{eq:polynomial-offline-exact-horizon}
 T=\max\!\left\{1,\left\lceil
       \left(\frac{\alpha D\Lambda_m B\sqrt d}{2\epsilon}\right)^2
       \right\rceil\right\}.
\end{equation}
With noisy responses, the total query count is
$O\!\left(q_m[1+P_m^2 k_m^2\sigma^2/\epsilon^2]\right)$.
There are $T$ projections and polynomially many arithmetic operations in $d,n,T$. For a supplied $0<v\le\OPT$, setting $\epsilon=\varepsilon v$ gives
$\E f(p)\ge(\alpha_m(\zeta)-\varepsilon)\OPT$ with the coefficients in \eqref{eq:offline-polynomial-coefficients}.
\end{corollary}
\begin{proof}
The coefficients in \eqref{eq:offline-polynomial-coefficients} are those of \eqref{eq:cubic-geometric-profile}, since $A_3(1/2)=1/2\ge8/17$. Apply the interpolation and shape-constrained reconstruction of Proposition~\ref{prop:polynomial-oracle} once, using $q_mn$ queries; the query nodes do not depend on the state. Let $q=f-f(\zero)$, let $\widetilde q$ be the reconstructed DR polynomial, and let
$\delta=\|\widetilde a-a\|_1$ be the coefficient error, so that
$\E\delta\le\ell_{m,n}=2P_m k_m\sigma/\sqrt n$.
Normalize the reconstruction as
$\widetilde F=\widetilde q-\min_Q\widetilde q$, a nonnegative DR-submodular polynomial of degree at most $m$, and compute its exact field $s_{\widetilde F}$. Condition on the reconstruction throughout the projected iterations. Then \eqref{eq:polynomial-perturbed-comparison} holds simultaneously for every state and comparator:
\[
 \alpha f(o)-\E_H f(p_\zeta(x))
 \le\alpha\ip{s_{\widetilde F}(x)}{o-x}+2\delta ,
\]
where $\E_H$ averages over the action law of Theorem~\ref{thm:polynomial}. Since $\min_Q\widetilde F=0$, \eqref{eq:polynomial-field-bound} applied to $\widetilde F$ shows that the computable number
$G=\Lambda_m\widetilde F(\tfrac12\one)$ bounds every field norm, and \eqref{eq:polynomial-surrogate-field} gives
$G\le\Lambda_m(B\sqrt d/2+2\delta)$.

If $G>0$, use projected linear ascent with step size $\eta=D/(G\sqrt T)$; if $G=0$, keep the state fixed. The update vectors $s_{\widetilde F}(x_t)$ are determined by the reconstruction, and \eqref{eq:projected-ascent-regret} bounds their linear regret by $D^2/(2\eta)+\eta TG^2/2=DG\sqrt T$. Let $o\in K$ attain $\OPT$, and let $p$ be the action of a uniformly sampled iteration, drawn from the action law at that state. Summing the displayed comparison over the $T$ states and dividing by $T$ gives
\[
 \E[f(p)\mid\widetilde q]
 \ge\alpha\OPT-2\delta-\frac{\alpha DG}{\sqrt T}
 \ge\alpha\OPT-2\delta-\frac{\alpha D\Lambda_m}{\sqrt T}\Bigl(\frac{B\sqrt d}2+2\delta\Bigr).
\]
Taking expectations over the reconstruction and using $\E\delta\le\ell_{m,n}$ gives
\begin{equation}\label{eq:polynomial-offline-error}
 \E f(p)\ge\alpha\OPT-2\ell_{m,n}
 -\frac{\alpha D\Lambda_m}{\sqrt T}
                (B\sqrt d/2+2\ell_{m,n}).
\end{equation}
If $\sigma>0$, the batch size \eqref{eq:polynomial-offline-batch} gives $\sqrt n\ge8P_mk_m\sigma/\epsilon$ and hence
$\ell_{m,n}\le\epsilon/4$; if $\sigma=0$, then $\ell_{m,n}=0$. With $\ell_{m,n}\le\epsilon/4$, the last term in \eqref{eq:polynomial-offline-error} is at most
$(\alpha D\Lambda_m/(2\sqrt T))(B\sqrt d+\epsilon)$, which is at most $\epsilon/2$ by \eqref{eq:polynomial-otb-horizon}. Together with $2\ell_{m,n}\le\epsilon/2$, this proves \eqref{eq:polynomial-offline-guarantee}.

Exact responses of either type allow $n=1$ and give $\delta=0$, so \eqref{eq:polynomial-offline-error} becomes $\E f(p)\ge\alpha\OPT-\alpha D\Lambda_mB\sqrt d/(2\sqrt T)$, and the horizon \eqref{eq:polynomial-offline-exact-horizon} makes the last term at most $\epsilon$. With noisy responses, $n\le1+64P_m^2k_m^2\sigma^2/\epsilon^2$, which gives the stated total query count $q_mn$. There is one projection per iteration, and the reconstruction, the linear program, and the fields use polynomially many arithmetic operations by Proposition~\ref{prop:polynomial-oracle} and Theorem~\ref{thm:polynomial}.

For the deterministic assertion with exact responses, fix the initial state and the tie rules, so that all states are deterministic. At every state, form the quadratic action or the analytic cubic action; when using the $8/17$ construction, form both $A(x_t)$ and $B(x_t)$ from \eqref{eq:cubic-baseline-maps}. Return the candidate maximizing the reconstructed polynomial, which with exact responses equals $q=f-f(\zero)$. This candidate also maximizes $f$ among the candidates, and its value is at least the average over iterations and branch probabilities that was bounded above. Evaluating the reconstructed polynomial uses arithmetic alone, so the query count is unchanged.
\end{proof}

\subsection{Transferring the oracle upper bounds to online optimization}\label{sec:online-upper-transfer}
Fix $\zeta$. An online instance consists of a set $K$ as in Section~\ref{sec:model-geometry} with $\zeta\one\in K$, an oblivious sequence of nonnegative continuously differentiable DR-submodular objectives satisfying \eqref{eq:regularity}, exact value-and-gradient queries, and exact projection onto $K$. A uniform online algorithm is a single algorithm that receives $d,T,D,M,B$ and is run on every such instance; it may depend on $\zeta$. Let $\aon(\zeta)$ be the supremum of coefficients $a$ for which there are a uniform online algorithm, a constant $\kappa<1$, and fixed polynomials $P_Q,P_R$ such that, on every instance and for every $T$,
\begin{equation}\label{eq:online-polynomial-convention}
 \begin{split}
 Q(T)&\le P_Q(d,T,1+D,1+M,1+B),\\
 \Reg_a(T)&\le P_R(d,1+D,1+M,1+B)\,T^\kappa ,
 \end{split}
\end{equation}
where $Q(T)$ is the total number of objective queries and projections on every execution. Objective queries on a round occur after the action for that round has been committed. Each query returns the full value-and-gradient response, as in the definition of the offline coefficient $\aor(\zeta)$.

The constant-sequence reduction applies to every online learner, including one that uses many queries on a round.
\begin{proposition}[Online coefficients satisfy the offline upper bound]\label{lem:oracle-online-transfer}
For every $\zeta\in[0,1]$,
\begin{equation}\label{eq:online-offline-order}
 \aon(\zeta)\le\aor(\zeta).
\end{equation}
\end{proposition}
\begin{proof}
Fix a coefficient $a$ attained by a uniform online algorithm as in \eqref{eq:online-polynomial-convention}, and consider an admissible offline instance from Definition~\ref{def:oracle-coefficient} with accuracy $\varepsilon\in(0,1)$. The function $\widetilde f=f/v$ is nonnegative, continuously differentiable, and DR-submodular, with value bound $\widetilde M=M/v$, gradient bound $\widetilde B=B/v$, and optimum $\OPT/v\ge1$.

Run the online algorithm with inputs $d,T,D,\widetilde M,\widetilde B$ on the constant sequence $\widetilde f_t=\widetilde f$. Each response is obtained by querying $f$ at the requested point and dividing the value and gradient by $v$. If the learner also observes its earned reward, one additional value query at the played action supplies it, for at most $T$ additional queries. Set
$C=P_R(d,1+D,1+M/v,1+B/v)$ and choose
\begin{equation}\label{eq:transfer-horizon}
 T=\max\!\left\{1,
  \left\lceil(C/\varepsilon)^{1/(1-\kappa)}\right\rceil\right\}.
\end{equation}
Return a uniformly sampled played action $p$. For the constant sequence, $\Reg_a(T)=aT\OPT/v-\E\sum_t\widetilde f(p_t)$, so averaging as in \eqref{eq:offline} gives
\[
 \E\widetilde f(p)
 \ge a\,\OPT/v-C T^{\kappa-1}
 \ge a\,\OPT/v-\varepsilon,
\]
where the second inequality holds because $T^{1-\kappa}\ge C/\varepsilon$. Multiplying by $v$ and using $v\le\OPT$ gives
$\E f(p)\ge a\OPT-\varepsilon v\ge(a-\varepsilon)\OPT$. The point $p$ is feasible because every online action is.

For the query count, put $r=1/(1-\kappa)$, a fixed number. Then $(C/\varepsilon)^r\le1+(C\varepsilon^{-1})^{\lceil r\rceil}$, so $T$ is bounded by a fixed polynomial in $C$ and $\varepsilon^{-1}$, and hence in the normalized arguments of \eqref{eq:oracle-polynomial-budget}. The number of calls to each oracle is at most $P_Q(d,T,1+D,1+M/v,1+B/v)+T$, which is again polynomial in the same arguments. Thus $a$ is an achievable offline coefficient, and taking suprema proves \eqref{eq:online-offline-order}.
\end{proof}

\subsection{The resulting approximation intervals}\label{sec:approximation-intervals}
Combining the general structural coefficient, the transfer above, and the offline upper function yields one interval for each model. The upper function concerns the full DR-submodular class, whereas the polynomial coefficients apply to the subclasses of Theorem~\ref{thm:polynomial}.

\begin{theorem}[Online and offline approximation intervals]\label{thm:intervals}
For every $\zeta\in[0,1]$,
\begin{equation}\label{eq:approximation-interval}
 \alow(\zeta)\le\aon(\zeta)\le\aor(\zeta)\le\aup(\zeta).
\end{equation}
At zero diagonal level \(\zeta=0\),
\begin{equation}\label{eq:endpoint-interval}
 \frac49\le\aon(0)\le\aor(0)\le
 \beta_*.
\end{equation}
For $\zeta\ge1/2$,
\begin{equation}\label{eq:matched-half}
 \aon(\zeta)=\aor(\zeta)=\frac12.
\end{equation}
If $1/4\le\zeta<1/2$, put $\Delta=1/2-\zeta$. For either
$a_\star(\zeta)=\aon(\zeta)$ or $a_\star(\zeta)=\aor(\zeta)$,
\begin{equation}\label{eq:deficit-bounds}
 \frac{2\Delta^2}{(1+2\sqrt\zeta)^4}
 \le\frac12-a_\star(\zeta)
 \le\frac{\Delta^2}{2(1+\Delta)^2}.
\end{equation}
In particular, the deficit is $\Theta(\Delta^2)$ as $\Delta\downarrow0$.
\end{theorem}
\begin{proof}
Algorithm~\ref{alg:gradient} with exact gradients is a uniform online algorithm that makes $Q(T)=2T$ objective queries and projections, and since $f_t(\zero)\ge0$, Corollary~\ref{thm:first-order} gives
\[
 \Reg_{\alow(\zeta)}(T)\le\alow(\zeta)DB\sqrt{2dT}\le2d(1+D)(1+B)\,T^{1/2}.
\]
Thus \eqref{eq:online-polynomial-convention} holds with $a=\alow(\zeta)$ and $\kappa=1/2$, which gives the first inequality of \eqref{eq:approximation-interval}. Proposition~\ref{lem:oracle-online-transfer} gives the second, and Theorem~\ref{thm:upper} gives the third. Corollary~\ref{cor:offline} also gives $\alow(\zeta)\le\aor(\zeta)$ through a constructive offline algorithm.

At $\zeta=0$, Corollary~\ref{thm:geometry} gives $\alow(0)=4/9$, and Corollary~\ref{cor:upper-endpoint} gives $\aor(0)\le\beta_*$. For $\zeta\ge1/2$, \eqref{eq:geometry-special} and Theorem~\ref{thm:upper} give $\alow(\zeta)=\aup(\zeta)=1/2$, so \eqref{eq:approximation-interval} proves \eqref{eq:matched-half}.

Let $1/4\le\zeta<1/2$. By \eqref{eq:strengthened-profile}, $\aup(\zeta)\le C(\zeta)$, with $C$ given by Lemma~\ref{prop:upper-near}. Since $1-2\zeta=2\Delta$, its formula gives
\[
 \frac12-C(\zeta)=\frac{(2\Delta)^2}{2(1+2\sqrt\zeta)^4}=\frac{2\Delta^2}{(1+2\sqrt\zeta)^4}.
\]
Since $3-2\zeta=2(1+\Delta)$, \eqref{eq:geometry-profile} gives
\[
 \frac12-\alow(\zeta)=\frac{(2\Delta)^2}{2\cdot4(1+\Delta)^2}=\frac{\Delta^2}{2(1+\Delta)^2}.
\]
By \eqref{eq:approximation-interval}, $\alow(\zeta)\le a_\star(\zeta)\le C(\zeta)$ for either choice of $a_\star$, which yields \eqref{eq:deficit-bounds}. As $\Delta\downarrow0$, $\sqrt\zeta\to1/\sqrt2$, so the two coefficients of $\Delta^2$ converge to
$2/(1+\sqrt2)^4=2/(17+12\sqrt2)=34-24\sqrt2>0$ and to $1/2$, respectively.
\end{proof}

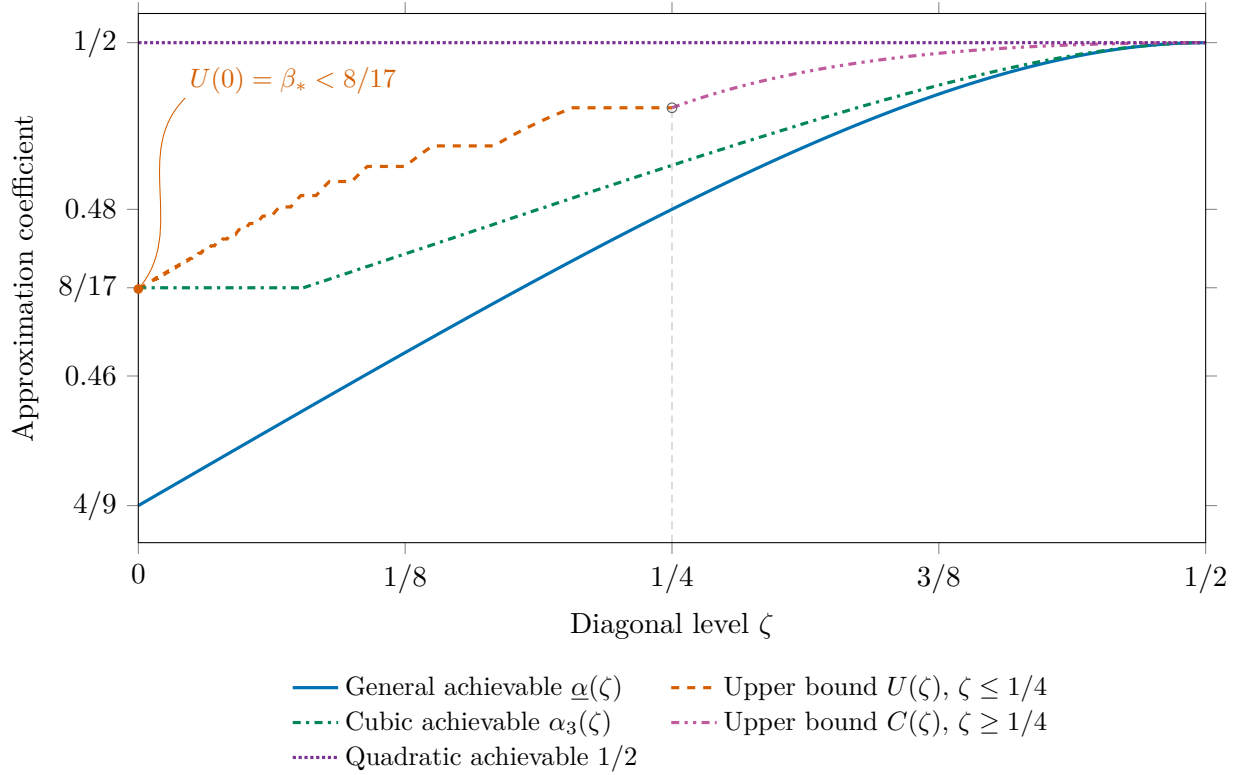
\begin{figure}[H]
\centering
\begingroup
\definecolor{profileBlue}{HTML}{0072B2}
\definecolor{profileGreen}{HTML}{00865A}
\definecolor{profilePurple}{HTML}{7B3294}
\definecolor{profileOrange}{HTML}{D55E00}
\definecolor{profilePink}{HTML}{C05A9D}
\pgfmathsetmacro{\cubicswitch}{(47-9*sqrt(17))/128}
\def\endpointbeta{0.47043868138089445}
\begin{tikzpicture}
\begin{axis}[
 width=.95\textwidth,height=.52\textwidth,
 xmin=0,xmax=.5,ymin=.440,ymax=.5035,
 xlabel={Diagonal level $\zeta$},ylabel={Approximation coefficient},
 xtick={0,.125,.25,.375,.5},
 xticklabels={$0$,$1/8$,$1/4$,$3/8$,$1/2$},
 ytick={.444444444444,.46,.470588235294,.48,.5},
 yticklabels={$4/9$,$0.46$,$8/17$,$0.48$,$1/2$},
 tick align=outside,axis on top,
 legend columns=2,legend cell align=left,
 legend style={at={(.5,-.23)},anchor=north,draw=none,fill=none,font=\small,
 /tikz/every even column/.append style={column sep=10pt}},
 clip=false]
\addplot[black!25,thin,densely dashed,forget plot]
 coordinates {(.25,.440) (.25,.4921875)};
\addplot[profileBlue,very thick,solid,domain=0:.5,samples=201]
 {1/2-(1-2*x)^2/(2*(3-2*x)^2)};
\addlegendentry{General achievable $\alow(\zeta)$}
\addplot[profileOrange,very thick,dashed]
 table[x=zeta,y=upper,col sep=comma]{figures/upper_profile.csv};
\addlegendentry{Upper bound $U(\zeta)$, $\zeta\le1/4$}
\addplot[profileGreen,very thick,dashdotted,domain=0:\cubicswitch,samples=2]
 {8/17};
\addlegendentry{Cubic achievable $\alpha_3(\zeta)$}
\addplot[profileGreen,very thick,dashdotted,domain=\cubicswitch:.5,samples=201,forget plot]
 {2*sqrt(3-4*x)/(1+sqrt(3-4*x))^2};
\addplot[profilePink,very thick,dashdotdotted,domain=.25:.5,samples=201]
 {1/2-(1-2*x)^2/(2*(1+2*sqrt(x))^4)};
\addlegendentry{Upper bound $C(\zeta)$, $\zeta\ge1/4$}
\addplot[profilePurple,very thick,densely dotted,domain=0:.5,samples=2]
 {1/2};
\addlegendentry{Quadratic achievable $1/2$}
\addplot[profileOrange,only marks,mark=*,mark size=1.6pt,forget plot]
 coordinates {(0,\endpointbeta)};
\addplot[black!65,only marks,mark=o,mark size=1.8pt,
 mark options={fill=white},forget plot]
 coordinates {(.25,.4921875)};
\node[anchor=west,font=\small,text=profileOrange,fill=white,inner sep=1.5pt]
 (endpointLabel) at (axis cs:.022,.4955) {$U(0)=\beta_*<8/17$};
\draw[profileOrange,thin,->]
 (endpointLabel.south west) to[out=225,in=50] (axis cs:0,\endpointbeta);
\end{axis}
\end{tikzpicture}
\endgroup
\caption{Geometry-dependent approximation coefficients. The complete upper curve is $\aup(\zeta)=U(\zeta)$ for $0\le\zeta\le1/4$ and $\aup(\zeta)=C(\zeta)$ for $1/4\le\zeta\le1/2$, with the two portions distinguished by color and line style. For $0<\zeta<1/8$, the plotted values of $U$ are obtained numerically from the finite minimization in Proposition~\ref{prop:upper-evaluation}; its remaining branch and all other curves use the displayed formulas. The annotated endpoint satisfies $U(0)=\beta_*<8/17$. The achievable profiles are $\alow(\zeta)$ for general objectives, $\max\{8/17,A_3(\zeta)\}$ for cubics, and $1/2$ for quadratics. All profiles equal $1/2$ for $1/2\le\zeta\le1$.}\label{fig:profiles}
\end{figure}

\section{Extension to Discrete Submodular Optimization}\label{sec:discrete}
For offline maximization of a nonnegative submodular set function under one matroid or a fixed number of knapsack constraints, the multilinear extension and established rounding procedures turn the continuous guarantee into coefficient $4/9-\varepsilon$ for both families. This improves the $0.401$ guarantee of \citet[Corollary~1.3]{buchbinder2024} in the oracle model below, and for a matroid the coefficient is $\alow(\zeta)-\varepsilon$ at its largest diagonal level $\zeta$. In the other direction, Section~\ref{sec:discrete-upper} shows that the upper bound $\beta_*$ of Corollary~\ref{cor:upper-endpoint} also holds for set functions with value queries under a single cardinality constraint, and hence under one matroid.

\subsection{Multilinear relaxation and set-value access}\label{sec:discrete-model}
Let $N$ be a finite ground set, and let $f:2^N\to\R_{\ge0}$ be submodular:
$f(A)+f(B)\ge f(A\cup B)+f(A\cap B)$ for all $A,B\subseteq N$.
Let $\mathcal I\subseteq2^N$ be the feasible family. For both constraint families, an element whose singleton is infeasible lies in no feasible set, so we delete such elements and continue to write $N$ for the remaining ground set, with $d=|N|$. If no element remains, $\varnothing$ is optimal.

For $y\in[0,1]^N$, let $R_y$ contain each element $i$ independently with probability $y_i$. Define
\begin{equation}\label{eq:multilinear-extension}
 F(y)=\E f(R_y),\qquad
 O=\max_{S\in\mathcal I}f(S),\qquad
 O_{\rm frac}=\max_{y\in K}F(y),
\end{equation}
where $K$ is the relaxation specified below; it contains every feasible indicator vector $\one_S$, so $O_{\rm frac}\ge O$. The multilinear polynomial $F$ is nonnegative on the cube, its diagonal second derivatives vanish, and each mixed second derivative $\partial_{ij}F(y)$ is the expectation of the second difference $f(R+i+j)-f(R+i)-f(R+j)+f(R)\le0$ over $R=R_y\setminus\{i,j\}$, so $F$ is DR-submodular.

The oracle model consists of a set-value oracle, defined on every subset of the retained ground set including infeasible ones, and a Euclidean projection oracle for each relaxation to which the continuous algorithm is applied. For a matroid it also contains an independence oracle and an oracle that writes a requested point of the independence polytope as a convex combination of at most $d+1$ independent-set indicators; knapsack costs and budgets are given explicitly. Complexity bounds count calls to all of these oracles.

A positive objective scale is obtained from the empty set and the singletons:
\begin{equation}\label{eq:discrete-scale}
 s=\max\left\{f(\varnothing),\max_{i\in N}f(\{i\})\right\}\le O.
\end{equation}
Submodularity and nonnegativity imply $0\le f(A)\le ds$ for every $A\subseteq N$, as shown in the proof below; in particular $f\equiv0$ when $s=0$, and then $\varnothing$ is optimal. The same preprocessing and scale computation apply to every residual instance used below.
\begin{corollary}[Fractional approximation from set-value queries]\label{cor:discrete-fractional}
Suppose $d\ge1$, $s>0$, and $K\subseteq[0,1]^N$ is compact, convex, and down-closed, contains all feasible indicator vectors, and satisfies $\zeta\one\in K$. For $0<\eta<1/16$, the OTB procedure returns $X\in K$ satisfying
\begin{equation}\label{eq:discrete-fractional-mean}
 \E F(X)\ge(\alow(\zeta)-\eta)O_{\rm frac}
\end{equation}
using $O(d^5\eta^{-2})$ set-value queries and projections. Each iteration uses two set-value queries. For any $0<\rho<1$, independent repetition and sampled evaluation of $F$ give a feasible point $\widehat X$ such that
\begin{equation}\label{eq:discrete-fractional-probability}
 \Pr\left\{F(\widehat X)\ge(\alow(\zeta)-4\eta)O_{\rm frac}\right\}
 \ge1-\rho,
\end{equation}
with a number of these oracle calls polynomial in
$d,\eta^{-1},\log(\rho^{-1})$.
\end{corollary}
\begin{proof}
For a nonempty $A=\{i_1,\ldots,i_q\}$, write $f(A)-f(\varnothing)$ as the sum of the marginal gains $f(\{i_1,\ldots,i_j\})-f(\{i_1,\ldots,i_{j-1}\})$ and bound each gain by $f(\{i_j\})-f(\varnothing)$ using submodularity:
\[
 f(A)\le f(\varnothing)+
   \sum_{j=1}^q\bigl[f(\{i_j\})-f(\varnothing)\bigr]
 \le qs\le ds.
\]
At the empty set, $f(\varnothing)\le s\le ds$.

At a current state $x$, draw $I,S,Y$ according to \eqref{eq:query-law} for the objective $F$. Conditional on these choices, sample
$R\subseteq N\setminus\{I\}$ with independent inclusion probabilities $Y_j$, and return
\begin{equation}\label{eq:set-value-field}
 \widehat g=d\bigl[f(R\cup\{I\})-f(R)\bigr]e_I.
\end{equation}
The identity
\[
 \E\bigl[f(R\cup\{I\})-f(R)\mid I,Y\bigr]
 =\partial_I F(Y)
\]
and the averaging in Lemma~\ref{lem:sampler} give
$\E[\widehat g\mid x]=g_\zeta(F,x)$.
Moreover, both queried values belong to $[0,ds]$, so
\begin{equation}\label{eq:set-value-field-moment}
 \|\widehat g\|_2^2\le d^4s^2.
\end{equation}
The proof of Corollary~\ref{cor:offline} therefore applies with update moment bound $d^4s^2$, diameter bound $D^2\le d$, and additive accuracy $\eta s$. Its horizon is at most
$\max\{1,\lceil \alow(\zeta)^2d^5/\eta^2\rceil\}$, and
$s\le O_{\rm frac}$ gives \eqref{eq:discrete-fractional-mean}.

For the probability guarantee, write $\alpha=\alow(\zeta)$ and let $\pi$ be the probability that a single run gives $F(X)\ge(\alpha-2\eta)O_{\rm frac}$. Every candidate $X$ lies in $K$, so $0\le F(X)\le O_{\rm frac}$, and $O_{\rm frac}\ge s>0$. Bounding $F(X)$ by $(\alpha-2\eta)O_{\rm frac}$ off this event and by $O_{\rm frac}$ on it, and comparing with \eqref{eq:discrete-fractional-mean}, gives
\[
 (\alpha-\eta)O_{\rm frac}
 \le(\alpha-2\eta)O_{\rm frac}(1-\pi)+O_{\rm frac}\,\pi,
 \qquad\text{hence}\qquad
 \pi\ge\frac{\eta}{1-\alpha+2\eta}\ge\eta,
\]
where the last step uses $2\eta<1/8<4/9\le\alpha$. After $r=\lceil\eta^{-1}\log(2/\rho)\rceil$ independent runs, the probability that no candidate meets this threshold is at most $(1-\eta)^r\le e^{-r\eta}\le\rho/2$.
Estimate $F$ at each candidate by the mean of
\begin{equation}\label{eq:extension-evaluation-samples}
 n=\left\lceil\frac{d^2}{2\eta^2}
             \log\frac{4r}{\rho}\right\rceil
\end{equation}
independent values $f(R_X)$, each in $[0,ds]$. By Hoeffding's inequality \citep{hoeffding1963} and a union bound over the $r$ candidates, all $r$ means are within $\eta s$ of the corresponding values of $F$ with probability at least $1-2r\exp(-2n\eta^2/d^2)\ge1-\rho/2$. On this event, the candidate with the largest estimate has value at least that of the best candidate minus $2\eta s\le2\eta O_{\rm frac}$, and a final union bound over the two failure events proves \eqref{eq:discrete-fractional-probability}. The repetition count $r=O(\eta^{-1}\log(2/\rho))$ does not depend on $d$, and the total number of oracle calls is $r$ times $O(d^5\eta^{-2}+n)$, which is polynomial in $d$, $\eta^{-1}$, and $\log(\rho^{-1})$.
\end{proof}

\subsection{One matroid}\label{sec:discrete-matroid}
Let $\mathcal M=(N,\mathcal I)$ be a matroid with rank function $r_{\mathcal M}$. Its independence polytope and largest diagonal level are
\begin{equation}\label{eq:matroid-diagonal}
 \begin{split}
 P(\mathcal M)
 &=\left\{x\ge\zero:
      \sum_{i\in A}x_i\le r_{\mathcal M}(A)
      \text{ for every }A\subseteq N\right\},\\
 \zeta_{\mathcal M}
 &=\min_{\varnothing\ne A\subseteq N}
          \frac{r_{\mathcal M}(A)}{|A|}.
 \end{split}
\end{equation}
The formula for $\zeta_{\mathcal M}$ follows by substituting $x=t\one$ into every rank inequality, and $\zeta_{\mathcal M}>0$ because loops have been deleted. Swap rounding preserves the multilinear value in expectation for every submodular objective, monotone or not \citep[Theorem~II.1]{chekuri2010swap}.

\begin{corollary}[A geometry-dependent matroid approximation]\label{cor:matroid-discrete}\label{cor:discrete-offline}
For every $0<\varepsilon<4/9$, there is an algorithm returning $S\in\mathcal I$ with
\begin{equation}\label{eq:matroid-discrete-guarantee}
 \E f(S)\ge
 \bigl(\alow(\zeta_{\mathcal M})-\varepsilon\bigr)O.
\end{equation}
It uses $O(d^5\varepsilon^{-2})$ set-value queries and projections, together with polynomially many supplied decomposition and independence operations for swap rounding. Any supplied lower bound on $\zeta_{\mathcal M}$ gives the corresponding guarantee with that level. For a uniform matroid of rank $r$, the level is $\zeta_{\mathcal M}=r/d$.
\end{corollary}
\begin{proof}
Apply Corollary~\ref{cor:discrete-fractional} to $K=P(\mathcal M)$ with $\eta=\min\{\varepsilon/2,1/32\}$, and swap-round its output. To use the base form of the procedure in \citet[Section~IV]{chekuri2010swap}, add a set $Z$ of $r_{\mathcal M}(N)$ dummy elements, form the direct sum of $\mathcal M$ with the free matroid on $Z$, and truncate it to rank $r_{\mathcal M}(N)$; every independent set of $\mathcal M$ extends by dummy elements to a base of this matroid. Pad each set in the convex decomposition in this way, extend the objective so that it ignores $Z$, apply base swap rounding, and delete $Z$ from the rounded base. Padding does not change the marginals on $N$, and every independent set of the truncated sum restricts to an independent set of $\mathcal M$. The result is an independent set, and Theorem~II.1 of that paper gives
$\E[f(S)\mid X]\ge F(X)$. Thus
\[
 \E f(S)\ge\E F(X)
 \ge\bigl(\alow(\zeta_{\mathcal M})-\varepsilon\bigr)O_{\rm frac}
 \ge\bigl(\alow(\zeta_{\mathcal M})-\varepsilon\bigr)O.
\]
The rounding uses polynomially many operations in the decomposition size and $d$. If $\zeta_{\mathcal M}$ is not supplied, binary search on $t$ with the test $t\one\in P(\mathcal M)$, which the projection oracle decides, approximates it from below to accuracy $\varepsilon/4$; by the derivative bound in \eqref{eq:geometry-profile}, this changes the coefficient by at most $\varepsilon/4$, and together with $\eta\le\varepsilon/2$ the loss stays within the stated accuracy. For a uniform matroid of rank $r$, the minimum in \eqref{eq:matroid-diagonal} is attained at $A=N$ and equals $r/d$.
\end{proof}

\subsection{A fixed number of knapsack constraints}\label{sec:discrete-knapsack}
Let $k\ge1$ be a fixed integer, $A\in\R_{\ge0}^{k\times d}$, and $b\in\R_{\ge0}^k$, and call $S$ feasible when $A\one_S\le b$. After deleting infeasible singletons and redundant zero-budget rows, the relaxation is
\begin{equation}\label{eq:knapsack-diagonal}
 K=\{x\in[0,1]^N:Ax\le b\}.
\end{equation}
The reduction of \citet{kulik2013knapsack} enumerates a bounded number of elements and rounds a fractional solution of each resulting residual instance. We apply it with the coefficient $\alow(0)=4/9$, because the residual relaxations can lose the original diagonal point (Remark~\ref{rem:residual-geometry}).
\begin{corollary}[Fixed-number knapsack constraints]\label{cor:knapsack-discrete}
Fix $k$ and $0<\varepsilon<4/9$. There is an algorithm returning a feasible set $S$ with
\begin{equation}\label{eq:knapsack-discrete-guarantee}
 \E f(S)\ge(4/9-\varepsilon)O.
\end{equation}
Its set-value, projection, and arithmetic counts are bounded by
\begin{equation}\label{eq:knapsack-complexity}
 (d+1)^{O_k(\varepsilon^{-4})}
       \operatorname{poly}(d,\varepsilon^{-1}).
\end{equation}
In particular, the bound is polynomial for fixed $k$ and fixed accuracy.
\end{corollary}
\begin{proof}
The rounding lemma of \citet[Section~2.3, Lemma~2]{kulik2013knapsack} gives constants $C_k>0$ and $\theta_k>0$, depending only on $k$, such that, for $0<\theta\le\theta_k$, a fractional point $x$ with $F(x)\ge O/5$ and costs $A_{ji}\le\theta^3 b_j$ can be rounded to a feasible set with
\begin{equation}\label{eq:kulik-rounding}
    \mathbb E f(S)\ge(1-C_k\theta)F(x),
\end{equation}
using polynomially many set-value queries and arithmetic operations.
Choose
\[
 \theta=\min\left\{\theta_k,\frac{\varepsilon}{8(C_k+1)}\right\}.
\]
Enumerate feasible $T\subseteq N$ with $|T|\le H_\theta:=\lceil k\theta^{-4}\rceil$. Section~2.4 of \citet{kulik2013knapsack} reserves cost $A\one_T$, uses the remaining budget $b^T=b-A\one_T$, retains the elements of $T$ at zero residual cost, and retains other elements only if $A_{\cdot i}\le\theta^3b^T$. Write $K_T$ for the resulting relaxation. The objective on each residual is the restriction of $f$, which is nonnegative and submodular, and every residual-feasible set is feasible for the original instance.

Let $O_T$ be the optimum of the residual instance for $T$. By \citet[Section~2.4, Lemma~3]{kulik2013knapsack}, either an optimal set is among the enumerated sets, or some enumerated $T$ has $O_T\ge(1-\theta)O$; call this residual the specified one. The algorithm returns the best of the enumerated sets and the rounded residual solutions, comparing them by exact set values, so the guarantee holds trivially in the first case and when $O=0$.

For each residual with a nonempty ground set and positive scale in \eqref{eq:discrete-scale}, apply \eqref{eq:discrete-fractional-probability} with $\zeta=0$, $\eta=\varepsilon/16$, and $\rho=\varepsilon/8$. For the specified residual, the selected point has value at least $(4/9-4\eta)O_T=(4/9-\varepsilon/4)O_T$ with probability at least $1-\rho$, and this value is at least $O_T/5$ because $\varepsilon<4/9$, so \eqref{eq:kulik-rounding} applies on that event. On the complementary event the rounding still returns a feasible set, whose value is nonnegative. Consequently,
\begin{align}
 \E f(S)
 &\ge(1-\rho)(1-C_k\theta)
          (4/9-\varepsilon/4)(1-\theta)O\nonumber\\
 &\ge(1-\varepsilon/4)(4/9-\varepsilon/4)O
 \ge(4/9-\varepsilon)O.
 \label{eq:knapsack-loss-accounting}
\end{align}
The second inequality uses $(1-\rho)(1-C_k\theta)(1-\theta)\ge1-\rho-(C_k+1)\theta\ge1-\varepsilon/4$, which holds because $\rho=\varepsilon/8$ and $(C_k+1)\theta\le\varepsilon/8$; the third uses $(1-\varepsilon/4)(4/9-\varepsilon/4)\ge4/9-\varepsilon/4-\varepsilon/9\ge4/9-\varepsilon$. Empty residuals and residuals with zero scale in \eqref{eq:discrete-scale} are solved directly. There are at most $(d+1)^{H_\theta}$ residuals, and since $C_k$ and $\theta_k$ depend only on $k$, the choice of $\theta$ gives
\[
    \theta^{-4}=O_k(\varepsilon^{-4}),
    \qquad
    H_\theta=O_k(\varepsilon^{-4}).
\]
Corollary~\ref{cor:discrete-fractional} and the rounding procedure of \citet{kulik2013knapsack} use polynomially many oracle calls on each residual, and multiplying by the number of residuals gives \eqref{eq:knapsack-complexity}.
\end{proof}

\begin{remark}[Geometry of the residual constraints]\label{rem:residual-geometry}
Because the enumeration reserves part of each budget and removes expensive elements, a residual relaxation $K_T$ generally has a smaller diagonal level than the original relaxation \eqref{eq:knapsack-diagonal}. The coefficient $4/9$ needs only $\zero\in K_T$, which every residual satisfies.
\end{remark}

\subsection{The upper bound under a cardinality constraint}\label{sec:discrete-upper}
The feasible set $P_n$ of the endpoint instances in Corollary~\ref{cor:upper-endpoint} is a partition-matroid polytope with blocks $\{u,v\}$, $\{a_1,\ldots,a_{2n}\}$, and $\{b_1,\ldots,b_{2n}\}$. Following the replication used by \citet[Section~4]{qi2022} for the instance of \citet{oveis2011}, we replace each $b_i$ by many copies, so that a single cardinality bound forces $\sum_ib_i\le1$ up to a vanishing error while the other two blocks become unconstrained.

Fix $n\ge2$, $\omega\in\Delta_4$, and an integer $\ell\ge1$. The ground set $X$ consists of $u$, $v$, $a_1,\ldots,a_{2n}$, and copies $b_{i,1},\ldots,b_{i,\ell}$ of each $b_i$, so $|X|=2n(\ell+1)+2$. For $y\in[0,1]^X$, let $\bar b_i(y)=\ell^{-1}\sum_{c=1}^{\ell}y_{b_{i,c}}$, and define
\begin{equation}\label{eq:cardinality-instance}
 F_\ell(y)=F_{n,\omega}\bigl(y_u,y_v,(y_{a_i})_{i\le2n},(\bar b_i(y))_{i\le2n}\bigr),
 \qquad f_\ell(S)=F_\ell(\one_S),
 \qquad k_\ell=\ell+2n+2,
\end{equation}
with $F_{n,\omega}$ from \eqref{eq:four-objective}. The polytope of the family $\{S\subseteq X:|S|\le k_\ell\}$ is $K_\ell=\{y\in[0,1]^X:\sum_{p\in X}y_p\le k_\ell\}$. Let $\mathcal G_\ell$ be the group generated by the permutations of the copies of each $b_i$ and by the action of the group $\mathcal G$ of Section~\ref{sec:upper-base} in which a permutation that maps $b_i$ to $b_{i'}$ maps $b_{i,c}$ to $b_{i',c}$ for every $c$.

\begin{lemma}[Cardinality constraint]\label{lem:cardinality-upper}
The set function $f_\ell$ is nonnegative and submodular, with multilinear extension $F_\ell$. The instance $\max\{f_\ell(S):|S|\le k_\ell\}$ has optimum at least one, and the symmetry gap of $(F_\ell,K_\ell)$ with respect to $\mathcal G_\ell$ satisfies
\begin{equation}\label{eq:cardinality-gap}
 \gamma\le\max_{j\in\{0,1\}}\max_{0\le x\le1/2}Q_{j,n,0}(x;\omega)
   +\frac1{2n-1}+\frac{2n+2}{\ell},
\end{equation}
with $Q_{j,n,0}$ from \eqref{eq:parameterized-branches}. Consequently, for every $\varepsilon>0$, attaining expected coefficient $\beta_*+\varepsilon$ for nonnegative submodular maximization under a cardinality constraint requires a number of value queries exponential in the size of the ground set.
\end{lemma}
\begin{proof}
Every monomial of $F_{n,\omega}$ in \eqref{eq:four-components} has degree at most one in each $b_i$, and the averages $\bar b_i$ for distinct $i$ involve disjoint sets of copies, so $F_\ell$ is a multilinear polynomial on $[0,1]^X$. It agrees with $f_\ell$ at the vertices of the cube and is therefore the multilinear extension of $f_\ell$. The map $y\mapsto(y_u,y_v,(y_{a_i}),(\bar b_i(y)))$ is linear with an entrywise nonnegative matrix and maps the cube into the cube, so Lemma~\ref{lem:affine-dr}, applied to the nonnegative DR-submodular $F_{n,\omega}$ of Lemma~\ref{lem:upper-u0}, shows that $F_\ell$ is nonnegative and DR-submodular. For distinct $p,r\in X$ and $S\subseteq X\setminus\{p,r\}$, multilinearity gives $f_\ell(S+p+r)-f_\ell(S+p)-f_\ell(S+r)+f_\ell(S)=\partial_{pr}F_\ell(\one_S)\le0$, so $f_\ell$ is submodular.

The function $F_{n,\omega}$ is $\mathcal G$-invariant and each $\bar b_i$ is unchanged by permuting its copies, so $F_\ell$ is $\mathcal G_\ell$-invariant. The orbits of $\mathcal G_\ell$ are $\{u,v\}$, $\{a_1,\ldots,a_{2n}\}$, and the set of all copies, so a symmetrized point has $y_u=y_v=q$, $y_{a_i}=s$, and $y_{b_{i,c}}=t$, hence $\bar b_i=t$, and $F_\ell$ takes the value \eqref{eq:upper-symmetric-four} with $N=2n$; that expression equals $F_{n,\omega}$ at $(q,q,s\one,t\one)$ for all $(q,s,t)\in[0,1]^3$. Permutations preserve $|S|$, so feasibility depends only on the symmetrized indicator vector, and the instance is strongly symmetric in the sense of \citet[Definition~1.5]{vondrak2013}.

Choose $j\in I_R$ and let $S^\star=\{u\}\cup\{a_i:i\ne j\}\cup\{b_{j,1},\ldots,b_{j,\ell}\}$. Then $|S^\star|=\ell+2n\le k_\ell$, and $\one_{S^\star}$ is mapped to the asymmetric witness of Lemma~\ref{lem:upper-u0}, so $f_\ell(S^\star)=1$ and the optimum is at least one.

A symmetrized point $(q,s,t)$ of $K_\ell$ satisfies $2q+2ns+2n\ell t\le k_\ell$, so
\[
 t\le\frac{k_\ell}{2n\ell}=r_n(1+\delta_\ell),
 \qquad \delta_\ell=\frac{2n+2}{\ell},
\]
with $r_n=1/(2n)$ as in \eqref{eq:hard-parameters}. Replacing $q$ by $1-q$ when $q>1/2$ keeps $2q(1-q)$ and the bound on $t$, and does not decrease $1-q^2$ or $1-q$; by \eqref{eq:upper-symmetric-four}, it does not decrease the value. We may therefore assume $q\le1/2$ and write $x=1/2-q\in[0,1/2]$. The value is affine in $s\in[0,1]$, so it suffices to consider $s=0$ and $s=1$. Concavity of $\Phi_n$, which gives the first inequality in \eqref{eq:upper-coverage-tangent} for every $t\in[0,1]$, together with $\ell_n\ge0$ and $\ell_nr_n=(1-r_n)^{n-1}/2\le1/2$, gives
\[
 1-(1-t)^n\le\chi_n+\ell_n(t-r_n)\le\chi_n+\delta_\ell/2,
\]
and $2(1-q)=1+2x\le2$, so the coverage term of \eqref{eq:upper-symmetric-four} is at most $\mu(1+2x)\chi_n+\mu\delta_\ell$. At $s=0$, the matched term is $\theta\cdot2nt\le\theta(1+\delta_\ell)$ and the $J_n$ term vanishes, so the value is at most $Q_{0,n,0}(x;\omega)+(\theta+\mu)\delta_\ell$. At $s=1$, the matched term vanishes, which is at most the term $\theta r_n$ of $Q_{1,n,0}$, and the $J_n$ term is
\[
 \lambda\frac{2n}{2n-1}(1-q^2)
 =\lambda\Bigl(\frac34+x-x^2\Bigr)+\frac{\lambda(1-q^2)}{2n-1}
 \le\lambda\Bigl(\frac34+x-x^2\Bigr)+\frac{\lambda}{2n-1},
\]
so the value is at most $Q_{1,n,0}(x;\omega)+\lambda/(2n-1)+\mu\delta_\ell$. Since $\theta+\mu\le1$, $\lambda\le1$, and the optimum is at least one, these bounds prove \eqref{eq:cardinality-gap}.

For the consequence, fix $\varepsilon\in(0,1)$, which suffices because attaining a coefficient implies attaining every smaller one, and take the weights $\omega_*$ of \eqref{eq:endpoint-weights}. The proof of Corollary~\ref{cor:upper-endpoint} shows that the double maximum in \eqref{eq:cardinality-gap} is at most the right-hand side of \eqref{eq:finite-endpoint-gap}. Choose $n$ and then $\ell$ so that the resulting bound on $\gamma$ is at most $\beta_*+\varepsilon/2$. By \citet[Theorem~1.8]{vondrak2013}, every algorithm with expected coefficient $(1+\varepsilon/2)\gamma$ on the refinements of this instance requires exponentially many value queries. A refinement replaces each element by a cluster of $L$ copies and keeps a set feasible when its normalized cluster counts lie in $K_\ell$ \citep[Definition~1.7]{vondrak2013}; here this means $|\tilde S|\le Lk_\ell$, which is again a cardinality constraint. Finally, $(1+\varepsilon/2)\gamma\le\beta_*+\varepsilon/2+(\varepsilon/2)(\beta_*+\varepsilon/2)<\beta_*+\varepsilon$ because $\beta_*+\varepsilon/2<1$.
\end{proof}

\begin{corollary}[Upper bound under one matroid]\label{cor:matroid-upper}
For every $\varepsilon>0$, attaining expected coefficient $\beta_*+\varepsilon$ for a nonnegative submodular set function under one matroid constraint requires a number of value queries exponential in the size of the ground set, already for uniform matroids. Together with Corollary~\ref{cor:matroid-discrete}, the supremum of the coefficients attainable for one matroid with polynomially many oracle calls in the model of Section~\ref{sec:discrete-model} lies in $[4/9,\beta_*]$.
\end{corollary}
\begin{proof}
The hard instances in the proof of Lemma~\ref{lem:cardinality-upper} are refinements with ground sets of size $L|X|$ and constraint $|\tilde S|\le Lk_\ell$, whose feasible sets are the independent sets of the uniform matroid of rank $Lk_\ell$. These constraints are known to the algorithm, so independence tests, projections onto the corresponding cardinality polytopes, and convex decompositions of their points use no objective queries, and the lower bound, which counts value queries only, holds in the oracle model of Section~\ref{sec:discrete-model}. The lower end of the interval is Corollary~\ref{cor:matroid-discrete} with $\zeta_{\mathcal M}\ge0$ and $\alow(0)=4/9$.
\end{proof}

For a cardinality constraint, the previous value-query bounds were $0.491$ \citep{oveis2011} and $0.478$ \citep{qi2022}, and $\beta_*\approx0.4704$ is smaller than both. The pair $(F_\ell,K_\ell)$ also satisfies the hypotheses of Lemma~\ref{lem:continuous-refinement}: $F_\ell$ is a multilinear DR-submodular polynomial invariant under $\mathcal G_\ell$, the symmetrization preserves $\sum_py_p$, so $y\in K_\ell$ if and only if its symmetrization is in $K_\ell$, and the symmetric point with $q=1/4$, $s=1/2$, and $t=0$ lies in $K_\ell$, where $F_\ell\ge2\eta q(1-q)=3\eta/8>0$ for $\omega=\omega_*$. The replicated polytopes of \eqref{eq:refinement-symmetry} are then $\{z:\sum z\le mk_\ell\}$, so the bound $\aor(0)\le\beta_*$ of Corollary~\ref{cor:upper-endpoint} holds already for instances whose feasible sets are cardinality polytopes.

\section{Conclusions}\label{sec:conclusion}
A supplied feasible diagonal point $\zeta\one\in K$ yields geometry-dependent achievable coefficients and oracle upper bounds for non-monotone DR-submodular maximization over the class of instances that contain it. An upper-linearizable comparison with coefficient $\alow(\zeta)$, implemented from one gradient or from values, gives online regret through OMBQ and offline query complexity through OTB. In the other direction, two families of hard instances give the upper bound $\aup(\zeta)=\min\{U(\zeta),C(\zeta)\}$ against exact values and full gradients, with $U$ reduced exactly to finitely many convex minimization problems below $1/8$ and given in closed form above it.

At $\zeta=0$ the best online and offline coefficients therefore lie in $[4/9,\beta_*]$, both equal $1/2$ for $\zeta\ge1/2$, and their shortfall from $1/2$ is of order $(1/2-\zeta)^2$. For quadratic objectives we attain coefficient $1/2$, and for cubic objectives $\max\{8/17,A_3(\min\{\zeta,1/2\})\}$. For submodular set functions, the coefficient under one matroid lies in $[4/9,\beta_*]$, with $\alow(\zeta)-\varepsilon$ attainable when the matroid polytope contains $\zeta\one$; the same upper bound $\beta_*$ holds under a cardinality constraint, and $(4/9-\varepsilon)$ is attainable under a fixed number of knapsack constraints. The optimal coefficient between $\alow(\zeta)$ and $\aup(\zeta)$ remains open for $0\le\zeta<1/2$.

We note that the previous version of the paper \citep{aggarwal2026geometryv1} has improved achievability of cubic functions and upper bound of general function for general $\zeta$, while we have removed the computer-aided parts in this version and kept the analytical components. Concurrent independent work of \citet{feldman2026stronger} gives the upper bound $8/17\approx0.47059$ under a matroid or cardinality constraint, slightly above $\beta_*\approx0.47044$ of Corollary~\ref{cor:matroid-upper}. The second version of their paper remarks, following an observation of Siyuan Chen, that our earlier bound $\beta_*$ transfers to cardinality and matroid constraints; Lemma~\ref{lem:cardinality-upper} and Corollary~\ref{cor:matroid-upper} were added to this version after that remark, and we thank Moran Feldman, Alan Kuhnle, and Siyuan Chen for bringing it to our attention, and Moran Feldman for suggesting that we prove the matroid upper bound formally in this version.

\acks{The authors thank Mohammad Pedramfar and Christopher J. Quinn for helpful discussions, and Hillclimb for compute and generative AI credits. GPT and Claude assisted with checking the proofs and with writing and revising the manuscript.}
\label{main:end}

\appendix
\section{Standard Facts and Elementary Consequences}\label{app:base}
Differentiability on the cube follows the boundary convention in Definition~\ref{def:dr}: a continuously differentiable function on $[0,1]^d$ is continuous on the cube, differentiable in its interior, and has a gradient that extends continuously to the boundary.

\subsection{Integration on segments and finite summation}\label{app:calculus}
\begin{lemma}[Elementary calculus identities]\label{lem:calculus}
The following identities hold.
\begin{enumerate}[label=(\roman*)]
\item Let $f:[0,1]^d\to\R$ be continuously differentiable, with gradient continuously extended to the boundary. For any $a,b\in[0,1]^d$,
\begin{equation}\label{eq:segment-ftc}
 f(b)-f(a)
 =\int_0^1\ip{\nabla f(a+t(b-a))}{b-a}\,dt.
\end{equation}
\item For scalars $w_1,\ldots,w_n$ and $a_0,\ldots,a_n$,
\begin{equation}\label{eq:finite-parts}
 \sum_{k=1}^n w_k(a_k-a_{k-1})
 =w_na_n-w_1a_0-
   \sum_{k=1}^{n-1}(w_{k+1}-w_k)a_k.
\end{equation}
If $0\le w_1\le\cdots\le w_n\le W$ and $a_k\ge0$ for all $k$, the sum on the left is at most $Wa_n$.
\end{enumerate}
\end{lemma}
\begin{proof}
(i) First let $a,b$ lie in the open cube $(0,1)^d$. The segment between them lies in the open cube, where $f$ is differentiable, so $\phi(t)=f(a+t(b-a))$ has the continuous derivative $\phi'(t)=\ip{\nabla f(a+t(b-a))}{b-a}$ by the chain rule, and the fundamental theorem of calculus gives \eqref{eq:segment-ftc}. For general $a,b$, put $c=\tfrac12\one$, $a_n=(1-1/n)a+c/n$, and $b_n=(1-1/n)b+c/n$. These points lie in the open cube, so the identity holds for $a_n,b_n$, and $a_n\to a$, $b_n\to b$. As $n\to\infty$, the left side converges to $f(b)-f(a)$ by continuity of $f$. The points $a_n+t(b_n-a_n)$ converge to $a+t(b-a)$ uniformly in $t$, and $\nabla f$ is uniformly continuous on the compact cube, so the integrands converge uniformly on $[0,1]$. Passing to the limit proves \eqref{eq:segment-ftc}.

(ii) Shifting the index in the second part of the left side gives
\[
 \sum_{k=1}^n w_ka_k-\sum_{k=0}^{n-1}w_{k+1}a_k
 =w_na_n-w_1a_0+\sum_{k=1}^{n-1}(w_k-w_{k+1})a_k,
\]
which is \eqref{eq:finite-parts}. Under the stated monotonicity and sign conditions, $w_1a_0\ge0$ and $(w_{k+1}-w_k)a_k\ge0$, so the left side is at most $w_na_n\le Wa_n$.
\end{proof}

\subsection{Consequences of gradient antitonicity}\label{app:dr-base}
Concavity along a nonnegative or nonpositive direction is Proposition~4 of \citet{bian2017guaranteed}. We prove the form used here, for a segment in the cube, together with the first-order inequality it implies.

\begin{lemma}[Concavity on comparable segments]\label{lem:directional-dr}
Let $f:[0,1]^d\to\R$ be continuously differentiable and DR-submodular. If $a,a+v\in[0,1]^d$ and either $v\ge\zero$ or $v\le\zero$, then $t\mapsto f(a+tv)$ is concave on $[0,1]$. In particular, for comparable $a,b\in[0,1]^d$,
\begin{equation}\label{eq:directional-dr-tangent}
 f(b)-f(a)\le\ip{\nabla f(a)}{b-a}.
\end{equation}
\end{lemma}
\begin{proof}
Put $\phi(t)=f(a+tv)$ and $\psi(t)=\ip{\nabla f(a+tv)}{v}$. Applying Lemma~\ref{lem:calculus}(i) to the segment from $a+sv$ to $a+tv$ gives $\phi(t)-\phi(s)=\int_s^t\psi(r)\,dr$ for $0\le s\le t\le1$, and $\psi$ is continuous. If $v\ge\zero$, then $a+sv\le a+tv$, so \eqref{eq:dr} gives $\nabla f(a+sv)\ge\nabla f(a+tv)$, and the inner product with $v\ge\zero$ gives $\psi(s)\ge\psi(t)$. If $v\le\zero$, then $a+sv\ge a+tv$, so $\nabla f(a+sv)\le\nabla f(a+tv)$, and the inner product with $v\le\zero$ again gives $\psi(s)\ge\psi(t)$. Thus $\phi$ is the integral of a nonincreasing continuous function, hence concave.

For comparable $a,b$, take $v=b-a$, which is either nonnegative or nonpositive. Since $\psi$ is nonincreasing,
\[
 f(b)-f(a)=\phi(1)-\phi(0)=\int_0^1\psi(r)\,dr\le\psi(0)=\ip{\nabla f(a)}{b-a},
\]
which is \eqref{eq:directional-dr-tangent}.
\end{proof}

The refinement construction in Section~\ref{sec:upper} composes objectives with nonnegative affine maps.
\begin{lemma}[Nonnegative affine composition]\label{lem:affine-dr}
Let $f:[0,1]^d\to\R$ be continuously differentiable and DR-submodular. Let $\mathsf R\in\R^{d\times m}$ be entrywise nonnegative, and suppose $b+\mathsf R z\in[0,1]^d$ for every $z\in[0,1]^m$. Then $F(z)=f(b+\mathsf R z)$ is DR-submodular on $[0,1]^m$. If $f\ge0$, then also $F\ge0$.
\end{lemma}
\begin{proof}
For $z,z+h\in[0,1]^m$, Lemma~\ref{lem:calculus}(i) gives
\[
 F(z+h)-F(z)
 =\int_0^1\ip{\nabla f(b+\mathsf R z+t\mathsf R h)}{\mathsf R h}\,dt
 =\ip{\mathsf R^\top\nabla f(b+\mathsf R z)}{h}+o(\norm h_2),
\]
where the error bound follows from uniform continuity of $\nabla f$ on the cube. Hence $F$ is continuously differentiable, with $\nabla F(z)=\mathsf R^\top\nabla f(b+\mathsf R z)$, including the boundary convention. If $z\le z'$, then $b+\mathsf R z\le b+\mathsf R z'$ because $\mathsf R$ is entrywise nonnegative, so \eqref{eq:dr} gives $\nabla f(b+\mathsf R z)\ge\nabla f(b+\mathsf R z')$. Multiplication by the entrywise nonnegative matrix $\mathsf R^\top$ preserves this order:
\[
 \nabla F(z)=\mathsf R^\top\nabla f(b+\mathsf R z)
 \ge\mathsf R^\top\nabla f(b+\mathsf R z')
 =\nabla F(z').
\]
This is DR-submodularity of $F$. Every value of $F$ is a value of $f$ on the cube, so $f\ge0$ implies $F\ge0$.
\end{proof}

Each partial derivative of a DR-submodular function lies between its values at the two cube endpoints, and summing the squared coordinate bounds gives an estimate without a dimension factor.
\begin{lemma}[Endpoint bounds for partial derivatives]\label{lem:gradient-envelope}\label{lem:base-dr}
Let $f:[0,1]^d\to\R$ be continuously differentiable and DR-submodular, and define
\begin{equation}\label{eq:coordinate-envelope}
 m_i(f)=\max\{|\partial_i f(\zero)|,|\partial_i f(\one)|\}.
\end{equation}
For every $y\in[0,1]^d$, $|\partial_i f(y)|\le m_i(f)$. If
$\sup_{y\in[0,1]^d}\norm{\nabla f(y)}_2\le B$, then
\begin{equation}\label{eq:envelope-sum}
 \sum_{i=1}^d m_i(f)^2\le2B^2.
\end{equation}
\end{lemma}
\begin{proof}
For $y\in[0,1]^d$, $\zero\le y\le\one$, so \eqref{eq:dr} gives
\[
 \partial_i f(\one)\le\partial_i f(y)\le\partial_i f(\zero),
\]
and hence $|\partial_i f(y)|\le\max\{|\partial_i f(\zero)|,|\partial_i f(\one)|\}=m_i(f)$. Since the maximum of two nonnegative numbers is at most their sum, $m_i(f)^2\le\partial_i f(\zero)^2+\partial_i f(\one)^2$, and summing over $i$ gives
\[
 \sum_{i=1}^d m_i(f)^2
 \le\norm{\nabla f(\zero)}_2^2+
     \norm{\nabla f(\one)}_2^2
 \le2B^2.
\]
\end{proof}

\section{Proof of Theorem~\ref{thm:polynomial}}\label{app:polynomial}
The proof writes the objective as a nonnegative combination of simple components, bounds each component by a concave majorant, and proves that the scalar actions retain a fixed fraction of every majorant. All constructions are applied to
\begin{equation}\label{eq:polynomial-minimum-shift}
 F=f-\min_{[0,1]^d}f .
\end{equation}
A comparison for $F$ implies the same comparison for $f$ with the same field, since
$\alpha f(o)-\E f(p)=\alpha F(o)-\E F(p)-(1-\alpha)\min f$, where $\alpha=\alpha_m(\zeta)\le1$ and $\min f\ge0$. We fix an order of all vertices and arcs and a deterministic maximum-flow algorithm throughout. The arc capacities and supergradient choices below then depend only on the nonconstant coefficients, and the constant affects only the residual minimum.

\subsection{Proof of the quadratic part}\label{app:quadratic}
Let $\mathcal Q_d$ denote the nonnegative DR-submodular polynomials on $[0,1]^d$ of total degree at most two. For a quadratic polynomial, DR-submodularity means that every Hessian entry, including the diagonal ones, is nonpositive. Hence every $f\in\mathcal Q_d$ can be written as
\begin{equation}\label{eq:quad-form}
 f(z)=c_0+a^\top z-\sum_{i<j}w_{ij}z_i z_j-\sum_i d_i z_i^2,
 \qquad w_{ij},d_i\ge0.
\end{equation}

Following the graph-cut representation of pairwise submodular binary energies \citep{kolmogorov2004}, with the residual-flow step written out, we express $f$ as its vertex minimum plus a nonnegative sum of unary terms, directed cuts, and concave diagonal terms.

\begin{lemma}[Nonnegative representation]\label{lem:quadratic-representation}
For every $f\in\mathcal Q_d$, there are nonnegative coefficients $c_*,u_i,v_i,W_{ij}$ such that
\begin{equation}\label{eq:quad-representation}
 f(z)=c_*+\sum_i\bigl(u_i z_i+v_i(1-z_i)\bigr)
      +\sum_{i\ne j}W_{ij}z_i(1-z_j)
      +\sum_i d_i z_i(1-z_i),
\end{equation}
where the $d_i$ are those in \eqref{eq:quad-form}. The representation is computable by one source--sink maximum-flow calculation on a graph with $d+2$ vertices and $O(d^2)$ arcs.
\end{lemma}
\begin{proof}
Subtracting the concave diagonal terms gives
\[
 F(z):=f(z)-\sum_i d_i z_i(1-z_i)
      =c_0+\sum_i(a_i-d_i)z_i-\sum_{i<j}w_{ij}z_i z_j .
\]
The polynomial $F$ is multilinear and agrees with $f$ at every vertex of the cube, so its vertex values are nonnegative. A multilinear polynomial at $z\in[0,1]^d$ equals the average of its vertex values under independent Bernoulli$(z_i)$ coordinates, so $F\ge0$ on the cube.

Introduce a source $s$ and a sink $t$ with fixed labels $z_s=1$ and $z_t=0$, and say that an arc $i\to j$ is cut when $z_i=1$ and $z_j=0$. For each $i<j$, add an arc $i\to j$ of capacity $w_{ij}$; its cut polynomial is $w_{ij}z_i(1-z_j)$. Put
\[
 \lambda_i=a_i-d_i-\sum_{j>i}w_{ij}.
\]
If $\lambda_i\ge0$, add the arc $i\to t$ with capacity $\lambda_i$, whose cut polynomial is $\lambda_iz_i$; if $\lambda_i<0$, add $s\to i$ with capacity $-\lambda_i$, whose cut polynomial is $-\lambda_i(1-z_i)$. Let $E(z)$ be the sum of these cut polynomials and
\[
 C=c_0+\sum_{i:\lambda_i<0}\lambda_i .
\]
Expanding, the coefficient of $z_i$ in $E$ is $\sum_{j>i}w_{ij}+\lambda_i=a_i-d_i$, the coefficient of $z_iz_j$ is $-w_{ij}$, and the constant is $-\sum_{i:\lambda_i<0}\lambda_i$. Hence $F(z)=C+E(z)$ as polynomials.

Take a maximum $s$--$t$ flow of value $\nu$, and let $E_{\rm res}$ be the cut polynomial of the residual graph. For a binary $z$, the arcs leaving $S=\{s\}\cup\{i:z_i=1\}$ in the residual graph have total capacity equal to the original capacity $E(z)$, minus the flow leaving $S$, plus the flow entering $S$. The net flow out of every $s$--$t$ cut is $\nu$, so $E(z)=\nu+E_{\rm res}(z)$ at every vertex. Since $E_{\rm res}\ge0$ and a minimum cut has residual capacity zero,
\[
 c_*:=C+\nu=\min_{z\in\{0,1\}^d}F(z)\ge0 .
\]
Residual arcs into the source or out of the sink are never cut, because $z_s=1$ and $z_t=0$. Every other residual arc of positive capacity is an original arc or the reversal of an original arc that carries positive flow; in particular, positive residual capacity joins two variable vertices only if an original arc of positive capacity joins them. Residual arcs $i\to t$, $s\to i$, and $i\to j$ contribute terms $u_i z_i$, $v_i(1-z_i)$, and $W_{ij}z_i(1-z_j)$ with nonnegative coefficients. Therefore $F(z)=c_*+\sum_i(u_iz_i+v_i(1-z_i))+\sum_{i\ne j}W_{ij}z_i(1-z_j)$ at every vertex, and since both sides are multilinear, the identity holds on the whole cube. Adding back $\sum_i d_i z_i(1-z_i)$ gives \eqref{eq:quad-representation}. The graph has $d+2$ vertices and at most $\binom d2+d$ arcs.
\end{proof}

For a fixed representation \eqref{eq:quad-representation}, define
\begin{equation}\label{eq:quad-majorant}
 C_f(x)=c_*+\sum_i\bigl(u_i x_i+v_i(1-x_i)\bigr)
       +\sum_{i\ne j}W_{ij}\min\{x_i,1-x_j\}
       +\sum_i d_i x_i(1-x_i).
\end{equation}
The function $C_f$ is concave, and $f\le C_f$ on the cube because $z_i(1-z_j)\le\min\{z_i,1-z_j\}$ for $z_i,z_j\in[0,1]$. A supergradient of $C_f$ at $x$ is
\begin{equation}\label{eq:quad-supergradient}
 s_f(x)=u-v+\sum_{i\ne j}W_{ij}\xi_{ij}(x)
             +\sum_i d_i(1-2x_i)e_i,
 \qquad
 \xi_{ij}(x)=
 \begin{cases}
 e_i,&x_i\le1-x_j,\\
 -e_j,&x_i>1-x_j,
 \end{cases}
\end{equation}
since $\xi_{ij}(x)$ is the gradient of an affine function attaining $\min\{x_i,1-x_j\}$ at $x$.

For the quadratic action $h_Q(x)=x\wedge\tfrac12\one$ we prove
\begin{equation}\label{eq:quad-half-comparison}
 \tfrac12 f(o)-f(h_Q(x))\le\tfrac12\ip{s_f(x)}{o-x}
 \qquad(x,o\in[0,1]^d).
\end{equation}
\begin{proof}[Completion of the quadratic part of Theorem~\ref{thm:polynomial}]
Write $p=h_Q(x)$. Since $\zero\le p\le x$, down-closedness gives $p\in K$ whenever $x\in K$. For every directed term in \eqref{eq:quad-representation},
\begin{equation}\label{eq:quad-directed-bound}
 p_i(1-p_j)\ge\frac12\min\{x_i,1-x_j\}.
\end{equation}
If $x_i\le1/2$, then $p_i=x_i$ and $1-p_j\ge1/2$, so the left side is at least $x_i/2$. If $x_i>1/2$, then $p_i=1/2$ and $1-p_j\ge1-x_j$, so it is at least $(1-x_j)/2$. The remaining terms satisfy
\[
 p_i\ge\frac{x_i}{2},\qquad
 1-p_i\ge1-x_i,\qquad
 p_i(1-p_i)\ge x_i(1-x_i);
\]
the last holds with equality if $x_i\le1/2$, and otherwise $p_i(1-p_i)=1/4\ge x_i(1-x_i)$. Since all coefficients in \eqref{eq:quad-representation} are nonnegative, summing these bounds gives
\begin{equation}\label{eq:quad-action-majorant}
 f(h_Q(x))\ge\frac12 C_f(x).
\end{equation}
Concavity of $C_f$ and the supergradient inequality give
\[
 f(o)\le C_f(o)\le C_f(x)+\ip{s_f(x)}{o-x}.
\]
Multiplying by $1/2$ and using \eqref{eq:quad-action-majorant} proves \eqref{eq:quad-half-comparison}, which is \eqref{eq:polynomial-comparison} with $m=2$.

For the field bound, apply the construction to $F=f-\min f$, which has the same coefficients $u,v,W,d$. By \eqref{eq:quad-supergradient},
\[
 \norm{s_f(x)}_2\le\norm{s_f(x)}_1
 \le\sum_i(u_i+v_i)+\sum_{i\ne j}W_{ij}+\sum_i d_i .
\]
Evaluating the representation of $F$ at $\tfrac12\one$ gives
$F(\tfrac12\one)=c_*+\tfrac12\sum_i(u_i+v_i)+\tfrac14\sum_{i\ne j}W_{ij}+\tfrac14\sum_id_i$, which is at least one quarter of the right side. Hence $\norm{s_f(x)}_2\le4F(\tfrac12\one)$. The field is computed from one maximum flow and the active-branch rule in \eqref{eq:quad-supergradient}. Since $F=f-\min f$ gives $F(\tfrac12\one)=f(\tfrac12\one)-\min f$, this bound reads $\norm{s_f(x)}_2\le\Lambda_2\bigl(f(\tfrac12\one)-\min f\bigr)$ with $\Lambda_2=4$, which is the first inequality in \eqref{eq:polynomial-field-bound} for $m=2$, and the second follows as at the end of Appendix~\ref{app:cubic-assembly}.
\end{proof}

\subsection{Cubic components}\label{app:cubic}\label{app:cubic-rounding}
For the cubic part of Theorem~\ref{thm:polynomial}, the concave majorant is built from the components
\begin{align}
 D_{ij}(z)&=z_i(1-z_j),&C_{D_{ij}}(x)&=\min\{x_i,1-x_j\},\label{eq:cubic-D}\\
 H_{i;jk}(z)&=(1-z_i)(z_j+z_k-z_jz_k),&C_{H_{i;jk}}(x)&=\min\{1-x_i,x_j+x_k\},\label{eq:cubic-H}\\
 J_{i;jk}(z)&=z_i(1-z_jz_k),&C_{J_{i;jk}}(x)&=\min\{x_i,2-x_j-x_k\},\label{eq:cubic-J}
\end{align}
together with constants and the unary terms $z_i$ and $1-z_i$, each of which is its own majorant. We call $D$, $H$, and $J$ the directed, coverage, and product components. The displayed minima are concave and majorize the corresponding components: for $z\in[0,1]^d$, $uv\le\min\{u,v\}$ whenever $u,v\in[0,1]$,
$z_j+z_k-z_jz_k\le\min\{1,z_j+z_k\}$, and
$1-z_jz_k\le2-z_j-z_k$ because $(1-z_j)(1-z_k)\ge0$.
The decomposition below combines the classical cubic quadratization with graph-cut representations \citep{billionnet1985,kolmogorov2004}, and tracks the nonnegativity of the coefficients and the support of the auxiliary variables.

\subsection{Nonnegative decomposition of multilinear cubics}\label{app:cubic-decomposition}
We first decompose functions of three binary variables and then reduce an arbitrary multilinear cubic to such functions.

For a nonnegative multilinear quadratic, Lemma~\ref{lem:quadratic-representation} gives a nonnegative sum of a constant, unary terms, and directed cuts, with all $d_i=0$. Its proof shows that positive residual capacity joins two variable vertices only if they were joined in the original graph, and this support property keeps the auxiliary variables of the cubic construction separate.

\Needspace{8\baselineskip}
\begin{lemma}[Three-variable decomposition]\label{lem:cubic-ternary}
The multilinear extension of every nonnegative submodular function on three binary variables is a nonnegative sum of constants, unary terms, directed cuts, and the terms $H$ and $J$ in \eqref{eq:cubic-H}--\eqref{eq:cubic-J}.
\end{lemma}
\begin{proof}
Write $G_S$ $(S\subseteq\{1,2,3\})$ for the eight vertex values and $\tau$ for the cubic coefficient of the multilinear extension. If $\tau=0$, the extension is a nonnegative multilinear quadratic whose quadratic coefficients $G_{\{i,j\}}-G_{\{i\}}-G_{\{j\}}+G_\varnothing$ are nonpositive by submodularity, and Lemma~\ref{lem:quadratic-representation} applies. Suppose $\tau=-\beta<0$, and put
\[
 b=G_\varnothing,\qquad v=G_{\{1,2,3\}},\qquad
 s_i=G_{\{i\}},\qquad p_i=G_{\{1,2,3\}\setminus\{i\}}.
\]
By M\"obius inversion, $\tau=\sum_S(-1)^{3-|S|}G_S$, so
\begin{equation}\label{eq:cubic-ternary-beta}
 \beta=b-\sum_i s_i+\sum_i p_i-v.
\end{equation}
For distinct $i,j,k$, submodularity gives $p_i=G_{\{j,k\}}\le s_j+s_k-b$. Together with nonnegativity of the vertex values, these bounds imply
\begin{equation}\label{eq:cubic-ternary-intervals}
 \beta\le s_i+p_i,\qquad
 \beta\le\sum_i s_i,\qquad
 \beta\le p_i+p_j\ (i\ne j),\qquad
 2\beta\le\sum_i p_i.
\end{equation}
For the first inequality, substituting the bounds for $p_j$ and $p_k$ into \eqref{eq:cubic-ternary-beta} gives $\beta\le s_i+p_i-b-v$. For the third, substituting the bound for $p_k$ gives $\beta\le p_i+p_j-s_k-v$. Substituting all three bounds gives $\beta\le\sum_is_i-2b-v$, which is the second; and bounding $\sum_is_i$ below by $(\sum_ip_i+3b)/2$ gives $2\beta\le\sum_ip_i-b-2v$, which is the fourth.

Next we choose numbers $\lambda_i$ with
\begin{equation}\label{eq:cubic-lambda}
 \max\{0,\beta-p_i\}\le\lambda_i\le s_i,
 \qquad\sum_i\lambda_i=\beta.
\end{equation}
Each interval is nonempty by the first inequality in \eqref{eq:cubic-ternary-intervals}, and the upper endpoints sum to at least $\beta$ by the second. The lower endpoints sum to at most $\beta$: if at most one of them is positive this follows from $p_i\ge0$, if two are positive from $\beta\le p_i+p_j$, and if three are positive from $2\beta\le\sum_i p_i$. Starting at the lower endpoints and distributing the remaining mass within the intervals in a fixed order gives \eqref{eq:cubic-lambda}.

Subtract $\sum_i\lambda_iJ_{i;jk}$, where $\{j,k\}=\{1,2,3\}\setminus\{i\}$. Since $J_{i;jk}=z_i-z_1z_2z_3$, the cubic coefficient of the residual is $\tau+\beta=0$, and its quadratic coefficients are unchanged and nonpositive. At the vertex $\{i\}$ only $J_{i;jk}$ is nonzero, so the residual value is $s_i-\lambda_i\ge0$. At the opposite pair $\{j,k\}$, $J_{j;ik}$ and $J_{k;ij}$ equal one and $J_{i;jk}$ vanishes, so the residual value is $p_i-(\beta-\lambda_i)\ge0$. All $J$ terms vanish at the empty and full sets, where the residual values remain $b$ and $v$. The residual is therefore a nonnegative multilinear DR-submodular quadratic, to which Lemma~\ref{lem:quadratic-representation} applies.

If $\tau>0$, apply the preceding case to $z\mapsto G(\one-z)$. Complementing all coordinates preserves nonnegativity and submodularity and changes the sign of the cubic coefficient. Undoing the complementation turns each $J_{i;jk}$ into $H_{i;jk}$, exchanges the unary terms $z_i$ and $1-z_i$, and reverses the directed cuts, which proves the claim.
\end{proof}

The auxiliary variables of the reduction can then be eliminated one at a time.
\begin{lemma}[Cubic decomposition]\label{lem:cubic-multilinear}
Every nonnegative multilinear DR-submodular polynomial of degree at most three is a nonnegative sum of the components in \eqref{eq:cubic-D}--\eqref{eq:cubic-J}, together with constants and unary terms. The decomposition can be constructed from its coefficients in polynomially many arithmetic operations.
\end{lemma}
\begin{proof}
Write
\[
 F(z)=b_0+\sum_i a_i z_i+\sum_{i<j}b_{ij}z_i z_j
       +\sum_{|T|=3}t_T\prod_{i\in T}z_i .
\]
For $i\ne j$, the mixed Hessian entry $\partial_i\partial_jF(z)=b_{ij}+\sum_{T\supseteq\{i,j\}}t_Tz_{T\setminus\{i,j\}}$ is affine, and its maximum over the cube is attained by setting $z_{T\setminus\{i,j\}}=1$ exactly when $t_T>0$. Its nonpositivity on the cube is therefore equivalent to
\[
 b_{ij}+\sum_{\substack{T\supseteq\{i,j\}\\t_T>0}}t_T\le0 .
\]
Set $w_{ij}=-b_{ij}-\sum_{T\supseteq\{i,j\},\,t_T>0}t_T\ge0$.
For binary $z$, the following identities replace a cubic term by a quadratic in one auxiliary binary variable:
\begin{align}
 -\prod_{i\in T}z_i
 &=\min_{u_T\in\{0,1\}}u_T\left(2-\sum_{i\in T}z_i\right),\label{eq:cubic-quadratization-negative}\\
 \prod_{i\in T}z_i-\sum_{\{i,j\}\subset T}z_i z_j
 &=\min_{u_T\in\{0,1\}}u_T\left(1-\sum_{i\in T}z_i\right).\label{eq:cubic-quadratization-positive}
\end{align}
If $\kappa\in\{0,1,2,3\}$ coordinates of $z_T$ equal one, both sides of \eqref{eq:cubic-quadratization-negative} equal $-1$ for $\kappa=3$ and $0$ otherwise, and both sides of \eqref{eq:cubic-quadratization-positive} equal $0,0,-1,-2$ for $\kappa=0,1,2,3$. These are the classical cubic quadratization identities \citep{billionnet1985,kolmogorov2004}.

Write $t_T\prod_Tz_i=t_T\sum_{\{i,j\}\subset T}z_iz_j+t_T(\prod_Tz_i-\sum_{\{i,j\}\subset T}z_iz_j)$ when $t_T>0$ and $t_T\prod_Tz_i=|t_T|(-\prod_Tz_i)$ when $t_T<0$. The pair terms produced in the first case combine with $b_{ij}$ into $-w_{ij}$, so the two identities give $F(z)=\min_u Q(z,u)$ on binary vectors, where
\begin{align}\label{eq:cubic-auxiliary-Q}
 Q(z,u)={}&b_0+\sum_i a_i z_i-\sum_{i<j}w_{ij}z_i z_j\nonumber\\
 &+\sum_{T:t_T>0}t_Tu_T\left(1-\sum_{i\in T}z_i\right)
  +\sum_{T:t_T<0}|t_T|u_T\left(2-\sum_{i\in T}z_i\right).
\end{align}
Every quadratic coefficient of $Q$ is nonpositive, so $Q$ is a multilinear DR-submodular quadratic. Each auxiliary variable interacts only with the three original variables in its set $T$, and no two auxiliary variables interact. For every binary $(z,u)$,
\[
 Q(z,u)\ge\min_{u'}Q(z,u')=F(z)\ge0,
\]
and multilinearity makes $Q$ nonnegative on its whole cube. Applying Lemma~\ref{lem:quadratic-representation} to $Q$ gives a residual representation with nonnegative coefficients that, by the support property, contains no term joining two auxiliary variables.

Group the residual terms involving a fixed auxiliary $u=u_T$ into
\begin{equation}\label{eq:cubic-auxiliary-block}
 E_T(z_T,u)=Uu+V(1-u)
       +\sum_{j\in T}\bigl[A_jz_j(1-u)+B_ju(1-z_j)\bigr],
 \quad U,V,A_j,B_j\ge0.
\end{equation}
On binary $z_T$, define $\psi_T(z_T)=\min_{u\in\{0,1\}}E_T(z_T,u)$, which is nonnegative. To see that it is submodular, identify binary vectors with sets and, for $A,B\subseteq T$, take minimizers $u_A,u_B$. Since $E_T$ is a nonnegative combination of unary terms and directed cuts, it is submodular on binary $(z_T,u)$, so
\begin{align}\label{eq:ternary-partial-minimum}
 \psi_T(A)+\psi_T(B)
 &\ge E_T(A\cup B,u_A\vee u_B)+E_T(A\cap B,u_A\wedge u_B)\\
 &\ge\psi_T(A\cup B)+\psi_T(A\cap B).
\end{align}
Since the auxiliaries do not interact, $\min_uQ(z,u)$ separates into one minimization per auxiliary. On binary vectors, $F$ is therefore the sum of the residual constant, the residual unary and directed-cut terms on original variables, and the functions $\psi_T$. Replacing each $\psi_T$ by its multilinear extension, both sides are multilinear and agree at every vertex, so the identity holds on the whole cube. Applying Lemma~\ref{lem:cubic-ternary} to each $\psi_T$ completes the decomposition.

There are at most $\binom d3$ auxiliaries, each with three edges to original variables, in addition to $O(d^2)$ edges between original variables. The residual calculation uses one graph of size $O(d^3)$, and each ternary decomposition has constant size, which proves the construction claim.
\end{proof}

\subsection{Polarization and preservation of the minimum}\label{app:cubic-polarization}
Monomials with repeated factors are removed by replacing each coordinate with three interchangeable copies, called clones. For total degree at most three, all Hessian entries are affine, and the DR property passes to the polarized polynomial.

For $f(x)=\sum_{|\nu|\le3}a_\nu x^\nu$, define
\begin{equation}\label{eq:cubic-polarization}
 \widetilde f(z)=\sum_{|\nu|\le3}a_\nu
 \prod_{i=1}^d\frac{e_{\nu_i}(z_{i,1},z_{i,2},z_{i,3})}{\binom3{\nu_i}}.
\end{equation}
Here $\nu\in\mathbb N_0^d$, $|\nu|=\sum_i\nu_i$, and $e_k$ is the elementary symmetric polynomial of degree $k$, with $e_0=1$; thus $e_1$ is the sum of the three variables, $e_2$ the sum of their pairwise products, and $e_3$ their product. The diagonal map $L:[0,1]^d\to[0,1]^{3d}$ is $(Lx)_{i,r}=x_i$. Since $e_k(t,t,t)=\binom3kt^k$,
\begin{equation}\label{eq:cubic-diagonal-identity}
 \widetilde f(Lx)=f(x).
\end{equation}
\begin{lemma}[Polarization]\label{lem:cubic-polarization}
If $f$ is a DR-submodular polynomial of total degree at most three, then $\widetilde f$ is multilinear, has total degree at most three, is DR-submodular, and satisfies
\begin{equation}\label{eq:cubic-minimum-preservation}
 \min_{z\in[0,1]^{3d}}\widetilde f(z)
 =\min_{x\in[0,1]^d}f(x).
\end{equation}
In particular, nonnegativity of $f$ implies nonnegativity of $\widetilde f$.
\end{lemma}
\begin{proof}
Multilinearity and the degree bound are immediate from \eqref{eq:cubic-polarization}. For distinct blocks $i,j$ and clones $r,s$,
\begin{equation}\label{eq:cubic-polarization-cross}
 \partial_{z_{i,r}}\partial_{z_{j,s}}\widetilde f(z)
 =\frac19\partial_i\partial_j f(y),
\end{equation}
where $y_i$ is the average of the two other clones in block $i$, $y_j$ is the analogous average in block $j$, and every other coordinate of $y$ is the average of its three clones. For distinct clones $r,s$ of the same block,
\begin{equation}\label{eq:cubic-polarization-diagonal}
 \partial_{z_{i,r}}\partial_{z_{i,s}}\widetilde f(z)
 =\frac16\partial_i^2 f(y'),
\end{equation}
where $y'_i$ is the remaining clone of block $i$ and every other coordinate is its three-clone average. Both identities are linear in $f$, so it suffices to check them on monomials $x^\nu$ with $|\nu|\le3$. For \eqref{eq:cubic-polarization-cross} we may assume $\nu_i,\nu_j\ge1$, so $\nu_i,\nu_j\le2$ and every other exponent is at most one. For $k\in\{1,2\}$, the derivative of $e_k(z_{i,1},z_{i,2},z_{i,3})/\binom3k$ in $z_{i,r}$ is $e_{k-1}$ of the other two clones divided by $\binom3k$, which equals $(k/3)y_i^{k-1}=\frac13\frac{d}{dt}t^k\big|_{t=y_i}$. An exponent $\nu_l\le1$ in another block contributes $y_l^{\nu_l}$. Multiplying the factors gives \eqref{eq:cubic-polarization-cross}. For \eqref{eq:cubic-polarization-diagonal} we may assume $\nu_i\in\{2,3\}$, so every other exponent is at most one. The second derivative of $e_{\nu_i}/\binom3{\nu_i}$ in two distinct clones is $1/3$ for $\nu_i=2$ and the remaining clone $y'_i$ for $\nu_i=3$; these equal $\frac16\frac{d^2}{dt^2}t^{\nu_i}$ at $t=y'_i$, which gives \eqref{eq:cubic-polarization-diagonal}. All evaluation points $y,y'$ lie in the cube, so both derivatives are nonpositive because $f$ is DR-submodular. Second derivatives twice in the same clone vanish by multilinearity. Hence every Hessian entry of $\widetilde f$ is nonpositive, and $\widetilde f$ is DR-submodular.

For \eqref{eq:cubic-minimum-preservation}, consider a binary clone vector. Since $\widetilde f$ is symmetric within each block, its value depends only on the counts $k_i\in\{0,1,2,3\}$ of clones equal to one. With all other counts fixed, write the four values for block $i$ as $v_0,v_1,v_2,v_3$. For $k\in\{0,1\}$, start from a vector with $k$ clones of block $i$ equal to one and switch on two further clones $r,s$; because $\widetilde f$ is multilinear, the resulting second difference equals $\partial_{z_{i,r}}\partial_{z_{i,s}}\widetilde f\le0$. Thus $v_{k+2}-2v_{k+1}+v_k\le0$ for $k=0,1$, so $(v_k)$ is concave and lies above its chord:
\begin{equation}\label{eq:clone-chord}
 v_k\ge\left(1-\frac{k}{3}\right)v_0+\frac{k}{3}v_3
 \qquad(k=0,1,2,3).
\end{equation}
Applying \eqref{eq:clone-chord} to one block after another, every binary value of $\widetilde f$ is at least a convex combination of values with count vectors in $\{0,3\}^d$. By \eqref{eq:cubic-diagonal-identity}, these are values of $f$ at vertices of $[0,1]^d$, so the lower bound is at least $\min_{x\in\{0,1\}^d}f(x)$. Since $\widetilde f$ is multilinear, its values on the clone cube are averages of its binary values, and the same lower bound holds there.

By Lemma~\ref{lem:directional-dr} applied to single-coordinate segments, $f$ is concave in each coordinate, so a coordinate can be moved to zero or one without increasing the value. Repeating this in every coordinate shows that the minimum of $f$ over the cube equals its minimum over the vertices. Hence $\min\widetilde f\ge\min f$, and \eqref{eq:cubic-diagonal-identity} gives $\min\widetilde f\le\min f$, which proves \eqref{eq:cubic-minimum-preservation}.
\end{proof}

\subsection{Two scalar maps with coefficient \texorpdfstring{$8/17$}{8/17}}\label{app:cubic-action}
The baseline comparison uses the two maps of Remark~\ref{rem:cubic-history}:
\begin{equation}\label{eq:cubic-baseline-maps}
\begin{aligned}
 A(t)&=
 \begin{cases}
  t,&0\le t\le1/2,\\
  t/4+3/8,&1/2\le t\le3/4,\\
  \frac94(1-t),&3/4\le t\le1,
 \end{cases}
&
 B(t)&=
 \begin{cases}
  t,&0\le t\le1/4,\\
  t/3+1/6,&1/4\le t\le3/4,\\
  (7t-4)/3,&3/4\le t\le1.
 \end{cases}
\end{aligned}
\end{equation}
The branches agree at the breakpoints ($A(1/2)=1/2$, $A(3/4)=9/16$, $B(1/4)=1/4$, $B(3/4)=5/12$), and each branch satisfies
$0\le A(t),B(t)\le t$ on its interval. Put $h_0=A$, $h_1=B$, and draw one random variable $Z$ with
\begin{equation}\label{eq:cubic-baseline-action}
 \Pr(Z=0)=\frac8{17},\qquad
 \Pr(Z=1)=\frac9{17},\qquad
 p_i(x)=h_Z(x_i)\quad(1\le i\le d).
\end{equation}
The same draw is used for every coordinate, and since $\zero\le p(x)\le x$, the action belongs to every down-closed set containing $x$. We prove the coverage and product inequalities on the whole cube; the directed and unary inequalities are special cases of the coverage inequality.

In the coverage inequality, the head $a$ and the tails $b,c$ can be moved to the plane $a+b+c=1$, where the two branches of the majorant agree.
\begin{lemma}[Coverage inequality for the two-map action]\label{lem:cubic-baseline-coverage}
For $a,b,c\in[0,1]$, define
\begin{align}
 R_0(b,c)&=A(b)+A(c)-A(b)A(c),\nonumber\\
 R_1(b,c)&=B(b)+B(c)-B(b)B(c),\label{eq:cubic-baseline-coverage-slack}\\
 W(a,b,c)&=8(1-A(a))R_0(b,c)+9(1-B(a))R_1(b,c),\nonumber\\
 \Delta_H(a,b,c)&=W(a,b,c)-8\min\{1-a,b+c\}.\nonumber
\end{align}
Then $\Delta_H(a,b,c)\ge0$.
\end{lemma}
\begin{proof}
Both maps are nondecreasing on $[0,3/4]$ and affine on $[3/4,1]$, and their values lie in $[0,1]$, so $R_0=1-(1-A(b))(1-A(c))$ and $R_1$ lie in $[0,1]$ and are nondecreasing in each tail on $[0,3/4]$. Since $B(a)\le a$,
\begin{equation}\label{eq:cubic-baseline-head-bound}
 9(1-B(a))\ge8(1-a).
\end{equation}
Since $A(1)=0$ and $B(1)=1$, we have $R_0(1,c)=A(c)$ and $R_1(1,c)=1$, so a tail equal to one satisfies
\begin{equation}\label{eq:cubic-baseline-coverage-endpoint}
 W(a,1,c)=8(1-A(a))A(c)+9(1-B(a))\ge8(1-a)
\end{equation}
by \eqref{eq:cubic-baseline-head-bound}, and hence $\Delta_H(a,1,c)\ge0$.

\emph{Reduction to $a+b+c=1$.}
Write $s=b+c$. Suppose first that $s<1-a$, and increase $a$ to $1-s$ with the tails fixed. Along this motion the subtracted term remains $8s$, and while $a\le3/4$, the factors $1-A(a)$ and $1-B(a)$ do not increase, so $W$ does not increase. If the motion enters $a>3/4$, then $s<1/4$, so both tails lie in $[0,1/4]$, where $A$ and $B$ are the identity, and $R_0=R_1=s-bc$. On that part of the motion,
\begin{equation}\label{eq:cubic-baseline-coverage-head-derivative}
 \frac{\partial W}{\partial a}=18R_0-21R_1=-3R_0\le0,
\end{equation}
using $A'=-9/4$ and $B'=7/3$ on $(3/4,1)$. The motion therefore reaches the plane $a+b+c=1$ with no larger slack.

Suppose next that $s>1-a$, so that the subtracted term is $8(1-a)$. If both tails are at most $3/4$, decrease them within $[0,3/4]$ until their sum is $1-a$. Monotonicity of $R_0,R_1$ shows that $W$ does not increase, and the subtracted term stays $8(1-a)$. If $b>3/4$, instead hold $a,c$ fixed and vary $b$ in
\begin{equation}\label{eq:cubic-baseline-coverage-tail-interval}
 b\in[\max\{3/4,1-a-c\},1].
\end{equation}
On this interval $A(b)$ and $B(b)$ are affine in $b$, $R_0$ and $R_1$ are affine in $A(b)$ and $B(b)$, and the subtracted term equals $8(1-a)$, so the slack is affine in $b$ and its minimum occurs at an endpoint. The endpoint $b=1$ is covered by \eqref{eq:cubic-baseline-coverage-endpoint}; at the other endpoint, either $b=3/4$ or the point lies on the plane. After the same argument is applied to $c$ if it exceeds $3/4$, both tails are at most $3/4$, and the preceding monotonicity argument reaches the plane. It therefore suffices to prove $\Delta_H\ge0$ on the plane.

\emph{The plane inequality.}
Order the tails as $b\ge c$ and set $s=b+c=1-a$, $p=bc$. Then
\begin{equation}\label{eq:cubic-baseline-tail-product}
 0\le p\le\frac{s^2}{4}\le\frac s4,\qquad 0\le s\le1.
\end{equation}
Substituting $a=1-s$ into \eqref{eq:cubic-baseline-maps}, the coefficients $P=8(1-A(a))$ and $Q=9(1-B(a))$ are
\begin{equation}\label{eq:cubic-baseline-coverage-heads}
 (P,Q)=
 \begin{cases}
  (8-18s,21s),&0\le s\le1/4,\\
  (3+2s,9/2+3s),&1/4\le s\le1/2,\\
  (8s,9/2+3s),&1/2\le s\le3/4,\\
  (8s,9s),&3/4\le s\le1.
 \end{cases}
\end{equation}
In each of the following ranges, the displayed identity for $\Delta_H=PR_0+QR_1-8s$ follows by substituting \eqref{eq:cubic-baseline-maps} and \eqref{eq:cubic-baseline-coverage-heads} and expanding.

\emph{(i) $0\le s\le1/4$.}
Both tails are at most $1/4$, so $R_0=R_1=s-p$ and
\begin{equation}\label{eq:cubic-baseline-coverage-small}
 \Delta_H=s^2\left(1-\frac{3s}{4}\right)
          +(8+3s)\left(\frac{s^2}{4}-p\right)\ge0.
\end{equation}

\emph{(ii) $1/4\le s\le1/2$.}
If $b\le1/4$, both maps are the identity at both tails. With $r=s-1/4\in[0,1/4]$,
\begin{equation}\label{eq:cubic-baseline-coverage-second-low}
 \Delta_H=s\left[\frac{13}{64}+\frac54r(2-r)\right]
          +\left(\frac{15}{2}+5s\right)\left(\frac{s^2}{4}-p\right)\ge0.
\end{equation}
If $b\ge1/4$, then $b\le s\le1/2$ and $c\le s/2\le1/4$, and
\begin{equation}\label{eq:cubic-baseline-coverage-second-high}
 \Delta_H=3\left(s-\frac12\right)^2
       +\left(\frac92+3s\right)c\left(\frac12-s+c\right)\ge0,
\end{equation}
because $1/2-s+c=1/2-b\ge0$.

\emph{(iii) $1/2\le s\le3/4$.}
Here $c\le s/2\le3/8$. If $b\ge1/2$, then $b\le3/4$ and $c=s-b\le1/4$, and
\begin{equation}\label{eq:cubic-baseline-coverage-third-high}
 \Delta_H=3\left(s-\frac12\right)^2
       +\left(3s+\frac32\right)c\left(\frac32-s+c\right)\ge0.
\end{equation}
If $b\le1/2$ and $c\le1/4$, put $d=c-(s-1/2)=1/2-b\ge0$. Expanding around $b=1/2$ gives
\begin{equation}\label{eq:cubic-baseline-coverage-third-low}
 \Delta_H=3s(2s-1)+\frac34(2s-1)(6s-1)d
                    +\left(9s+\frac32\right)d^2\ge0.
\end{equation}
In the remaining case $b,c\in[1/4,1/2]$, and
\begin{equation}\label{eq:cubic-baseline-central-coverage}
 R_0=s-p,\qquad R_1=\frac{5s}{18}+\frac{11}{36}-\frac p9.
\end{equation}
With $t=s-1/2\ge0$, substitution gives
\begin{equation}\label{eq:cubic-baseline-coverage-third-central}
 \Delta_H=\frac{2+19t+162t^2}{24}
        +\left(\frac{25s}{3}+\frac12\right)\left(\frac s4-p\right)\ge0.
\end{equation}

\emph{(iv) $3/4\le s\le1$.}
Now $W=sF$ with $F=8R_0+9R_1$, so $\Delta_H=s(F-8)$. Since $c\le s/2\le1/2$, the possible tail regions give
\begin{equation}\label{eq:cubic-baseline-coverage-large}
 F-8=
 \begin{cases}
  \dfrac{33s-21}{4}+9\left(\dfrac s4-p\right),
       &b\le1/2,\\[4pt]
  5s-\dfrac72+5c\left(\dfrac32-s+c\right),
       &1/2\le b\le3/4,\ c\le1/4,\\[4pt]
  \dfrac{9s}{2}-\dfrac94+3c(1-s+c),
       &1/2\le b\le3/4,\ c\ge1/4,\\[4pt]
  3s-2+c(8-3s)+3c^2,&b\ge3/4.
 \end{cases}
\end{equation}
In the first region both tails lie in $[1/4,1/2]$, since $c=s-b\ge1/4$; in the last, $c\le1/4$. For $3/4\le s\le1$, the constant terms satisfy $33s-21\ge15/4$, $5s-7/2\ge1/4$, $9s/2-9/4\ge9/8$, and $3s-2\ge1/4$, while $3/2-s+c$, $1-s+c$, and $8-3s$ are nonnegative and $s/4-p\ge0$ by \eqref{eq:cubic-baseline-tail-product}. Every expression in \eqref{eq:cubic-baseline-coverage-large} is therefore nonnegative, which proves the plane inequality in all cases; the reduction extends it to the whole cube.
\end{proof}

\begin{corollary}[Directed and unary inequalities]\label{cor:cubic-baseline-directed}
For all $a,b\in[0,1]$,
\begin{align}
 8A(a)(1-A(b))+9B(a)(1-B(b))&\ge8\min\{a,1-b\},\label{eq:cubic-baseline-directed}\\
 8A(a)+9B(a)&\ge8a.\label{eq:cubic-baseline-unary}
\end{align}
\end{corollary}
\begin{proof}
Since $A(0)=B(0)=0$, we have $R_0(a,0)=A(a)$ and $R_1(a,0)=B(a)$. Lemma~\ref{lem:cubic-baseline-coverage} with head $b$ and tails $a,0$ therefore gives \eqref{eq:cubic-baseline-directed}. Setting $b=0$ gives \eqref{eq:cubic-baseline-unary}.
\end{proof}

In the product inequality, the minimum over the head coordinate occurs at $a=1$ or on the plane where the two branches of the majorant agree, and both sets are checked separately.
\begin{lemma}[Product inequality for the two-map action]\label{lem:cubic-baseline-product}
For all $a,b,c\in[0,1]$,
\begin{equation}\label{eq:cubic-baseline-product-slack}
\begin{split}
 \Delta_J(a,b,c):={}&8A(a)(1-A(b)A(c))
                    +9B(a)(1-B(b)B(c))\\
                  &-8\min\{a,2-b-c\}\ge0.
\end{split}
\end{equation}
\end{lemma}
\begin{proof}
We first establish three bounds for the maps:
\begin{equation}\label{eq:cubic-baseline-map-bounds}
 B(t)\ge\frac23A(t),\qquad A(t)\le\frac9{16},\qquad
 8A(t)^2+6B(t)^2\le6\qquad(0\le t\le1).
\end{equation}
For the first bound, subtracting branch by branch gives
\begin{equation}\label{eq:cubic-baseline-map-difference}
 B(t)-\frac23A(t)=
 \begin{cases}
  t/3,&0\le t\le1/4,\\
  (1-2t)/6,&1/4\le t\le1/2,\\
  (2t-1)/12,&1/2\le t\le3/4,\\
  (23t-17)/6,&3/4\le t\le1,
 \end{cases}
\end{equation}
and each branch is nonnegative on its interval. The map $A$ increases on $[0,3/4]$ and decreases on $[3/4,1]$, so its maximum is $A(3/4)=9/16$. For $t\le3/4$, also $B(t)\le B(3/4)=5/12$, so
\begin{equation}\label{eq:cubic-baseline-square-small}
 8A(t)^2+6B(t)^2\le8\left(\frac9{16}\right)^2
                   +6\left(\frac5{12}\right)^2=\frac{343}{96}<6.
\end{equation}
For $t\ge3/4$, put $u=1-t\in[0,1/4]$, so that $A(t)=9u/4$ and $B(t)=1-7u/3$. Then
\begin{equation}\label{eq:cubic-baseline-square-large}
 8A(t)^2+6B(t)^2=6-u\left(28-\frac{439}{6}u\right)\le6,
\end{equation}
because $28-439u/6\ge28-439/24=233/24>0$. This proves \eqref{eq:cubic-baseline-map-bounds}.

\emph{Reduction in the head coordinate.}
Fix $b,c$ and write $u=A(b)A(c)$, $v=B(b)B(c)$, and $d=2-b-c$. The Cauchy--Schwarz inequality for the vectors $(\sqrt8A(b),\sqrt6B(b))$ and $(\sqrt8A(c),\sqrt6B(c))$, together with the third bound in \eqref{eq:cubic-baseline-map-bounds}, gives $8u+6v\le6$. Since $v\le1$ and $9B(a)\ge6A(a)$,
\begin{equation}\label{eq:cubic-baseline-product-small-head}
 8A(a)(1-u)+9B(a)(1-v)
 \ge8A(a)+A(a)(6-8u-6v)\ge8A(a).
\end{equation}
For $a\le1/2$, $A(a)=a\ge\min\{a,d\}$, which proves the product inequality on that range.

On $[1/2,1]$, the slack $\Delta_J(\cdot,b,c)$ is continuous and piecewise affine, with possible breakpoints $3/4$ and $d$. On $(1/2,3/4)$, where $A'=1/4$ and $B'=1/3$, its slope is $-3-2u-3v<0$ when $a<d$ and $2(1-u)+3(1-v)\ge0$ when $a>d$. On $(3/4,1)$, where $A'=-9/4$ and $B'=7/3$, its slope is constant when $a>d$, and when $a<d$ it satisfies
\begin{equation}\label{eq:cubic-baseline-product-head-derivative}
 \frac{\partial\Delta_J}{\partial a}
 =-5+18u-21v
 \le-5+\frac{26}{3}u
 \le-5+\frac{26}{3}\left(\frac9{16}\right)^2
 =-\frac{289}{128}<0,
\end{equation}
where $v\ge4u/9$ follows from the first bound in \eqref{eq:cubic-baseline-map-bounds} and $u\le(9/16)^2$ from the second. Thus the slack decreases on $[1/2,\min\{d,1\}]$. On $[\max\{d,1/2\},1]$ it is nondecreasing up to $3/4$ and affine after $3/4$, so there its minimum is attained at $\max\{d,1/2\}$ or at $1$. Consequently the minimum over $[1/2,1]$ is attained either at $a=1/2$, which happens only when $d\le1/2$ and is covered by \eqref{eq:cubic-baseline-product-small-head}, or at
\begin{equation}\label{eq:cubic-baseline-product-reduction}
 a=1\quad\text{or}\quad a=2-b-c\in[1/2,1].
\end{equation}

\emph{The endpoint $a=1$.}
Here $A(1)=0$ and $B(1)=1$, so the inequality to prove is
\begin{equation}\label{eq:cubic-baseline-product-endpoint-target}
 B(b)B(c)\le1-\frac89\min\{1,2-b-c\}.
\end{equation}
Order $b\le c$ and put $s=b+c$. On $[0,3/4]$, $B(t)\le(t+1/2)/3$: equality holds on $[1/4,3/4]$, and $t\le t/3+1/6$ on $[0,1/4]$. If $s\le1$ and $c\le3/4$, the arithmetic--geometric mean inequality gives
\begin{equation}\label{eq:cubic-baseline-product-endpoint-central}
 B(b)B(c)\le\frac{(b+1/2)(c+1/2)}{9}\le\frac{(s+1)^2}{36}\le\frac19.
\end{equation}
If $s\le1$ and $c\ge3/4$, then $b\le1/4$ and $c\le1-b$, so monotonicity of $B$ gives $B(c)\le B(1-b)=(3-7b)/3$ and
\begin{equation}\label{eq:cubic-baseline-product-endpoint-low}
 B(b)B(c)\le\frac{b(3-7b)}3\le\frac3{28}<\frac19,
\end{equation}
where the quadratic $b(3-7b)/3$ attains its maximum $3/28$ at $b=3/14$. Thus \eqref{eq:cubic-baseline-product-endpoint-target} holds for $s\le1$.

For $s\ge1$, the right side of \eqref{eq:cubic-baseline-product-endpoint-target} is $(8s-7)/9$. If $c\le3/4$, then $s\le3/2$, and the bound $(s+1)^2/36$ in \eqref{eq:cubic-baseline-product-endpoint-central} suffices because
\begin{equation}\label{eq:cubic-baseline-product-endpoint-sum}
 (s+1)^2-4(8s-7)=(s-1)(s-29)\le0.
\end{equation}
If $b\le1/4$ and $c\ge3/4$, substituting $B(b)=b$, $B(c)=(7c-4)/3$, and $c=s-b$ gives
\begin{align}
 9B(b)B(c)-(8s-7)
 &=(21b-8)s-21b^2-12b+7\nonumber\\
 &\le9b-21b^2-1\le-\frac1{28}.\label{eq:cubic-baseline-product-endpoint-low-high}
\end{align}
The first inequality uses $21b-8<0$ and $s\ge1$; the quadratic in the second attains its maximum at $b=3/14$.
If $1/4\le b\le3/4$ and $c\ge3/4$, then
\begin{equation}\label{eq:cubic-baseline-product-endpoint-middle-high}
 9B(b)B(c)-(8s-7)=c\left(7b-\frac92\right)-12b+5.
\end{equation}
The right side decreases in $b$, since its derivative is $7c-12<0$, so it is at most its value $2-11c/4\le-1/16$ at $b=1/4$.
Finally, if $b,c\ge3/4$, then $3/2\le s\le2$ and, by $bc\le s^2/4$,
\begin{align}
 9B(b)B(c)-(8s-7)
 &=49bc-36s+23\nonumber\\
 &\le\frac{49s^2}{4}-36s+23
 =\frac{(s-2)(49s-46)}4\le0.\label{eq:cubic-baseline-product-endpoint-high}
\end{align}
This proves the endpoint inequality in every tail region.

\emph{The switching plane.}
Now let $a=2-b-c\in[1/2,1]$, so that $\min\{a,2-b-c\}=a$. Order $b\le c$ and put $s=b+c=2-a\in[1,3/2]$; in particular, $b\le3/4$. We distinguish the positions of the tails relative to $1/4$ and $3/4$.

If $c\le3/4$, then $b=s-c\ge1/4$, so $b,c\in[1/4,3/4]$, an interval on which $A$ is nonnegative and concave and $B$ is affine. The arithmetic--geometric mean inequality followed by concavity of $A$ (respectively, affinity of $B$) gives
\begin{equation}\label{eq:cubic-baseline-product-central-tails}
 A(b)A(c)\le A(s/2)^2=\frac{(5-a)^2}{64},\qquad
 B(b)B(c)\le B(s/2)^2=\frac{(3-a)^2}{36},
\end{equation}
where $s/2=1-a/2\in[1/2,3/4]$. Substituting these upper bounds for the tail products into \eqref{eq:cubic-baseline-product-slack} gives
\begin{equation}\label{eq:cubic-baseline-product-central-slack}
 \Delta_J(a,b,c)\ge
 \begin{cases}
  \dfrac{567-732a+139a^2-22a^3}{192},&1/2\le a\le3/4,\\[5pt]
  \dfrac{(1-a)(29a^2-42a+189)}{96},&3/4\le a\le1.
 \end{cases}
\end{equation}
The first expression decreases on $[1/2,3/4]$, because the numerator of its derivative, $-732+278a-66a^2$, has negative discriminant and is therefore negative; its value at $a=3/4$ is $927/2048>0$. In the second expression, $29a^2-42a+189$ has negative discriminant and is positive, so the expression is nonnegative.

If $b\le1/4$ and $c\ge3/4$, then $s\le5/4$, so $a\ge3/4$. Put $t=1-a$ and $x=b$, so that $c=1+t-x$ and $c\le1$ gives $0\le t\le x\le1/4$. Expanding and completing the square,
\begin{align}
 \Delta_J
 &=1-9x+21x^2+5t-\frac{179}{2}xt(x-t)\nonumber\\
 &=21\left(x-\frac3{14}\right)^2+\frac1{28}
   +\left(5-\frac{179}{2}x^2\right)t+\frac{179}{2}xt^2.
 \label{eq:cubic-baseline-product-low-high}
\end{align}
For $x\le1/8$, the coefficient of $t$ is at least $5-179/128=461/128>0$, so every term in the last line is nonnegative. For $1/8\le x\le1/4$, the bounds
\begin{equation}\label{eq:cubic-baseline-product-low-high-coefficients}
 5-\frac{179}{2}x^2\ge-\frac{19}{32}\ge-1,\qquad
 \frac{179}{2}x\ge\frac{179}{16}\ge8
\end{equation}
give
\begin{equation}\label{eq:cubic-baseline-product-low-high-positive}
 \Delta_J\ge\frac1{28}-t+8t^2
 =8\left(t-\frac1{16}\right)^2+\frac1{224}>0.
\end{equation}

In the last case, $1/4\le b\le3/4$ and $c\ge3/4$. For fixed $a$, define
\begin{equation}\label{eq:cubic-baseline-product-middle-interval}
 P_a(b)=\Delta_J(a,b,2-a-b),\qquad
 \max\{1/4,1-a\}\le b\le5/4-a,
\end{equation}
where the bounds on $b$ express $b\ge1/4$, $c\le1$, and $c\ge3/4$. Since $c\in[3/4,1]$ throughout, $P_a$ is quadratic on each side of $b=1/2$, where the branch of $A(b)$ changes.

First let $1/2\le a\le3/4$, so that the interval is $[1-a,5/4-a]\ni1/2$. On $[1-a,1/2]$, the quadratic coefficient of $P_a$ is $-(26a+67)/12<0$, so the minimum occurs at an endpoint. The endpoint values satisfy
\begin{align}
 P_a(1-a)&=a^2-4a+\frac{15}{4}\ge\frac{21}{16},\label{eq:cubic-baseline-product-middle-endpoints}\\
 P_a(1/2)&=\frac{245-300a+4a^2}{48}\ge\frac{89}{192},\nonumber
\end{align}
since both expressions decrease on $[1/2,3/4]$ (their derivatives are $2a-4$ and $(8a-300)/48$) and the lower bounds are their values at $a=3/4$. On $[1/2,5/4-a]$, the quadratic coefficient is $(58a-25)/48>0$, and
\begin{equation}\label{eq:cubic-baseline-product-middle-derivative}
 P_a'(5/4-a)=-\frac{58a^2-39a+155}{48}<0,
\end{equation}
because $58a^2-39a+155$ has negative discriminant. Since the derivative of a convex quadratic increases, $P_a$ decreases on this interval, and its minimum is at $b=5/4-a$, where $c=3/4$ and \eqref{eq:cubic-baseline-product-central-slack} applies.

Finally, let $3/4\le a\le1$ and put $t=1-a$. The interval in \eqref{eq:cubic-baseline-product-middle-interval} becomes $1/4\le b\le1/4+t\le1/2$, and expansion gives
\begin{equation}\label{eq:cubic-baseline-product-middle-quadratic}
 P_a(b)=\frac{42-341t}{6}b^2
       +\frac{3-49t+341t^2}{6}b
       +\frac{-3+30t+49t^2}{6}.
\end{equation}
If the quadratic coefficient is nonpositive, the minimum occurs at an endpoint. The endpoint $b=1/4$ is covered by \eqref{eq:cubic-baseline-product-low-high} with $x=1/4$, and the endpoint $b=1/4+t$ has $c=3/4$ and is covered by \eqref{eq:cubic-baseline-product-central-slack}. If the coefficient is positive, then $0\le t<42/341<1/8$, and
\begin{equation}\label{eq:cubic-baseline-product-middle-increasing}
 P_a'(1/4)=\frac{48-439t+682t^2}{12}\ge\frac{121}{384}>0,
\end{equation}
since the numerator has derivative $-439+1364t<0$ on $[0,1/8]$ and equals $121/32$ at $t=1/8$. The derivative of this convex quadratic is therefore positive on the whole interval, and its minimum occurs at the endpoint $b=1/4$, which is already covered.

These cases exhaust the switching plane. Together with the endpoint case and \eqref{eq:cubic-baseline-product-reduction}, they prove the product inequality on the whole cube.
\end{proof}

\begin{lemma}[The $8/17$ scalar comparison]\label{lem:cubic-rounding}
For the common scalar action in \eqref{eq:cubic-baseline-action}, every component $\phi$ in \eqref{eq:cubic-D}--\eqref{eq:cubic-J}, every unary component, and every nonnegative constant satisfy
\begin{equation}\label{eq:cubic-component-rounding}
 \E\phi(h_Z(x))\ge\frac8{17}C_\phi(x)\qquad(x\in[0,1]^d).
\end{equation}
Moreover, $h_Z(x)\in K$ whenever $x\in K$ and $K$ is down-closed.
\end{lemma}
\begin{proof}
Because the same $Z$ is used in every coordinate, $\E\phi(h_Z(x))=\tfrac8{17}\phi(A(x))+\tfrac9{17}\phi(B(x))$, where $A$ and $B$ act coordinatewise. Dividing Lemmas~\ref{lem:cubic-baseline-coverage} and~\ref{lem:cubic-baseline-product} and Corollary~\ref{cor:cubic-baseline-directed} by $17$ therefore proves \eqref{eq:cubic-component-rounding} for the coverage, product, directed, and positive unary components. For a complemented unary term, $h_Z(t)\le t$ gives
$1-h_Z(t)\ge1-t\ge(8/17)(1-t)$. Constants satisfy the comparison because $8/17\le1$, and feasibility follows from $\zero\le h_Z(x)\le x$.
\end{proof}

\subsection{Geometry-dependent cubic comparison}\label{app:cubic-geometry}
A deterministic scalar map attains the coefficient $A_3(r)$ of \eqref{eq:cubic-analytic-coefficient}. Set
\begin{equation}\label{eq:cubic-geometry-parameters}
 r=\min\{\zeta,1/2\},\qquad b=\sqrt{3/4-r},\qquad
 k=\frac{4}{(2b+1)^2},\qquad \lambda=1-kr,\qquad a_b=bk.
\end{equation}
Then $0\le r\le1/2$, $1/2\le b\le\sqrt3/2$, and $a_b=A_3(r)$. Since $b^2=3/4-r$, we have $b+1-r=(b+1/2)^2$, so
\begin{equation}\label{eq:cubic-geometry-k-identity}
 k(b+1-r)=1,
\end{equation}
and $kr=r/(b+1-r)\le r/(b+1/2)\le1/2$, so $1/2\le\lambda\le1$. Define
\begin{equation}\label{eq:cubic-geometric-action}
 R(t)=\frac{t}{b+t},\qquad
 L(t)=\lambda t+(1-\lambda)r,\qquad
 H_r(t)=\min\{R(t),L(t)\}\quad(0\le t\le1).
\end{equation}
In the products below, the scalar map is applied to distinct component slots, which may be different clones of one original coordinate.

\begin{lemma}[Geometric cubic component inequalities]\label{lem:cubic-geometry}
For every $r\in[0,1/2]$ and all $u,v,w\in[0,1]$, the map in \eqref{eq:cubic-geometric-action} satisfies
\begin{gather}
 H_r(u)\ge a_bu,\qquad 1-H_r(u)\ge a_b(1-u),\label{eq:cubic-geometry-unaries}\\
 H_r(u)(1-H_r(v))\ge a_b\min\{u,1-v\},\label{eq:cubic-geometry-directed}\\
 (1-H_r(u))(H_r(v)+H_r(w)-H_r(v)H_r(w))
 \ge a_b\min\{1-u,v+w\},\label{eq:cubic-geometry-coverage}\\
 H_r(u)(1-H_r(v)H_r(w))\ge a_b\min\{u,2-v-w\}.\label{eq:cubic-geometry-product}
\end{gather}
Moreover, $H_r(x)\in K$ whenever $x\in K$ and $r\one\in K$.
\end{lemma}
\begin{proof}
Set
\[
 T(t)=kt(b+1-t),\qquad d_0=(b+1)/2.
\]
By \eqref{eq:cubic-geometry-k-identity} and $(b+1/2)^2-(b+1-t)(b+t)=(t-1/2)^2$,
\begin{equation}\label{eq:cubic-geometric-slacks}
 L(t)-T(t)=k(t-r)^2,\qquad
 R(t)-T(t)=\frac{kt(t-1/2)^2}{b+t}.
\end{equation}
Consequently $T\le H_r\le R$ on $[0,1]$; in particular $0\le T\le H_r\le1$, and $T(t)\le\max_{s\in\R}ks(b+1-s)=k(b+1)^2/4=(b+1)^2/(2b+1)^2\le9/16$ for $b\ge1/2$. Since $R(0)=0$, also $H_r(0)=0$. Both $R$ and $L$ are nondecreasing, and $R(t)/t=1/(b+t)$ and $L(t)/t=\lambda+(1-\lambda)r/t$ are nonincreasing on $(0,1]$, so $H_r$ is nondecreasing and $H_r(t)/t$ is nonincreasing on $(0,1]$. Also
\begin{equation}\label{eq:cubic-geometric-upperhalf}
 L(1/2)-R(1/2)=(b-1/2)^2,\qquad
 L'(1/2)-R'(1/2)=\frac{2(b-1/2)}{b+1/2}\ge0.
\end{equation}
Since $R$ is concave, $L'-R'$ is nondecreasing, so $L-R$ is nondecreasing on $[1/2,1]$ and remains nonnegative there. Hence $H_r(t)=R(t)$ for $t\ge1/2$; in particular $H_r(1)=R(1)$.

\emph{Unary terms.} Since $T(t)=kt(b+1-t)\ge kbt=a_bt$, the first inequality in \eqref{eq:cubic-geometry-unaries} follows. For the second, $1-H_r(t)\ge1-R(t)=b/(b+t)$ and $(b+t)(1-t)\le(b+1)^2/4$, so
\[
 \frac{1-H_r(t)}{1-t}\ge\frac{b}{(b+t)(1-t)}
 \ge\frac{4b}{(b+1)^2}\ge\frac{4b}{(2b+1)^2}=a_b\qquad(t<1).
\]
At $t=1$ the inequality reduces to $1-H_r(1)\ge0$.

\emph{Directed terms.} Fix $v$ and put $s=1-v$. If $s>0$ and $u>0$, then $H_r(u)/\min\{u,s\}\ge H_r(s)/s$: for $u\le s$ this is the monotonicity of $H_r(t)/t$, and for $u\ge s$ it is the monotonicity of $H_r$. Using $H_r(s)\ge T(s)$, $H_r(v)\le R(v)$, and $b+1-s=b+v$,
\[
 \frac{H_r(s)(1-H_r(v))}{s}
 \ge k(b+1-s)\frac{b}{b+v}=bk=a_b .
\]
Multiplying the two bounds proves \eqref{eq:cubic-geometry-directed}. If $s=0$ or $u=0$, the right side vanishes and the left side is nonnegative.

\emph{Product terms.} Put $s=2-v-w$. For fixed $v+w$, the product
\[
 R(v)R(w)=\frac{vw}{b^2+b(v+w)+vw}
\]
is increasing in $vw$, which is largest when $v=w$; it is also nondecreasing in each coordinate. Suppose $s\ge1$, so $v+w\le1$ and $\min\{u,s\}=u$. Then $H_r(v)H_r(w)\le R(v)R(w)\le R(1/2)^2$, and the monotonicity of $H_r(t)/t$ and $H_r(1)=R(1)$ give, for $u>0$,
\[
 \frac{H_r(u)(1-H_r(v)H_r(w))}{u}
 \ge R(1)\bigl(1-R(1/2)^2\bigr)
 =\frac{1}{b+1}\cdot\frac{4b(b+1)}{(2b+1)^2}=a_b .
\]

For $0<s<1$, the head satisfies $H_r(u)\ge(H_r(s)/s)\min\{u,s\}$ as in the directed case, and the tails have sum $2-s$, so $H_r(v)H_r(w)\le R(1-s/2)^2$ and
\[
 1-H_r(v)H_r(w)\ge1-R(1-s/2)^2=\frac{b(b+2-s)}{(b+1-s/2)^2}.
\]
It suffices to show that $(H_r(s)/s)\,b(b+2-s)/(b+1-s/2)^2\ge a_b$. If $H_r(s)=R(s)$, this reads
\[
 \frac{b(b+2-s)}{(b+s)(b+1-s/2)^2}\ge\frac{4b}{(2b+1)^2},
\]
which, after clearing positive denominators, follows from
\begin{equation}\label{eq:cubic-geometric-J-rational}
 (b+2-s)(2b+1)^2-4(b+s)(b+1-s/2)^2
 =(1-s)\bigl[4b^2+b(5-3s)+(1-s)(2-s)\bigr]\ge0.
\end{equation}
Otherwise $H_r(s)=L(s)<R(s)$, so $s<1/2$ by \eqref{eq:cubic-geometric-upperhalf}, and $H_r(s)/s\ge T(s)/s=k(b+1-s)$. Since $a_b=bk$, it then suffices that
\[
 (b+1-s)(b+2-s)\ge(b+1-s/2)^2 .
\]
The difference equals $b(1-s)+1-2s+3s^2/4$, which is nonnegative for $s\le1/2$. If $s=0$, the right side of \eqref{eq:cubic-geometry-product} vanishes.

\emph{Coverage terms.} Write $s=v+w$ and
\[
 M(s)=\max_{0\le u\le1}(b+u)\min\{1-u,s\}
 =\begin{cases}s(b+1-s),&0\le s\le d_0,\\(b+1)^2/4,&d_0\le s\le2.\end{cases}
\]
The formula holds because $(b+u)s$ increases in $u$ on $u\le1-s$, while $(b+u)(1-u)$ has its unconstrained maximum $(b+1)^2/4$ at $u=(1-b)/2\in[0,1/4]$, which satisfies $u\ge1-s$ exactly when $s\ge d_0$. Since $1-H_r(u)\ge b/(b+u)$,
\[
 (1-H_r(u))\,kM(s)\ge\frac{b}{b+u}\,k(b+u)\min\{1-u,s\}=a_b\min\{1-u,s\},
\]
so it is enough to prove
\begin{equation}\label{eq:cubic-geometric-H-reduction}
 H_r(v)+H_r(w)-H_r(v)H_r(w)\ge kM(v+w).
\end{equation}
Its left side equals $1-(1-H_r(v))(1-H_r(w))$ and is therefore nondecreasing in each of the values $H_r(v)$ and $H_r(w)$, which lie in $[0,1]$.

If both tails are at least $1/2$, then $H_r=R$ there, and the left side is at least $2R(1/2)-R(1/2)^2=(4b+1)/(2b+1)^2$. Dividing by $k$ gives $b+1/4$, which exceeds $(b+1)^2/4\ge M(s)$ by $b(2-b)/4\ge0$.

Next suppose that one tail, denoted $v$, has $H_r(v)=R(v)$ and the other tail $w$ is at most $1/2$. Since $1-R(v)=b/(b+v)$ and $H_r(w)\ge T(w)$, the left side of \eqref{eq:cubic-geometric-H-reduction} is at least
\[
 R(v)+T(w)(1-R(v))=\frac{v+bT(w)}{b+v}.
\]
For $s=v+w\le d_0$, multiplying the required inequality by $(b+v)/k$ and using $1/k=(2b+1)^2/4$, the difference of the two sides is
\begin{equation}\label{eq:cubic-geometric-H-mixed}
 \frac vk+bw(b+1-w)-(b+v)s(b+1-s)
 =\frac v4\bigl[(2s-1)^2+4bw\bigr]\ge0.
\end{equation}
For $s\ge d_0$, replace $M(s)$ by its constant value. The resulting difference $v/k+bw(b+1-w)-(b+v)(b+1)^2/4$ is affine and increasing in $v$, with slope
\[
 \frac1k-\frac{(b+1)^2}{4}=\frac{b(3b+2)}4>0.
\]
Since $w\le1/2\le d_0$, the allowed values satisfy $v\ge d_0-w\ge0$, and at $v=d_0-w$ the difference is nonnegative by \eqref{eq:cubic-geometric-H-mixed} with $s=d_0$. This settles the case, including two tails with $H_r=R$ when one of them is at most $1/2$.

In the remaining case $H_r(v)<R(v)$ and $H_r(w)<R(w)$, and both tails are below $1/2$ because $H_r=R$ on $[1/2,1]$. As $H_r\ge T$ and $T\in[0,1]$, the left side of \eqref{eq:cubic-geometric-H-reduction} is at least $T(v)+T(w)-T(v)T(w)$, and expanding with $v(b+1-v)+w(b+1-w)=s(b+1-s)+2vw$ and $(b+1-v)(b+1-w)=(b+1)^2-(b+1)s+vw$ shows that this quantity divided by $k$ equals
\begin{equation}\label{eq:cubic-geometric-H-double}
 s(b+1-s)+vw\left[2-k\bigl((b+1)^2-(b+1)s+vw\bigr)\right].
\end{equation}
Suppose $s\le d_0$. Since $vw\le s^2/4$, the bracket is at least $2-k(b+1-s/2)^2$, which is nonnegative if and only if $b+1-s/2\le\sqrt2(b+1/2)$. When $b\ge1/\sqrt2$, this holds for every $s\ge0$, because $b+1\le\sqrt2(b+1/2)$ is equivalent to $b\ge1/\sqrt2$. If $1/2\le b<1/\sqrt2$, define
\[
 e_0=b+1-\sqrt2(b+1/2)>0 .
\]
For a clipped tail $t$, the inequalities $R(t)>L(t)\ge(1-\lambda)r=kr^2$ and $R(t)\le t/b$ imply $t>bkr^2=a_br^2$. The two identities
\begin{align}
 a_br^2-\frac{1/2-b^2}{2}
 &=\frac{(b+1)(2b-1)^2(4b^2+2b-1)}{4(2b+1)^2}\ge0,\label{eq:cubic-geometric-small-tail}\\
 \frac{1/2-b^2}{2}-e_0
 &=-\frac12(b-1/\sqrt2)(b+2-3/\sqrt2)\ge0\label{eq:cubic-geometric-small-tail-two}
\end{align}
give $a_br^2\ge e_0$; the signs follow from $4b^2+2b-1\ge1$, $b<1/\sqrt2$, and $b+2-3/\sqrt2>0$ for $b\ge1/2$. Therefore $s>2a_br^2\ge2e_0$, so $b+1-s/2<\sqrt2(b+1/2)$, and the bracket is again nonnegative. In both cases \eqref{eq:cubic-geometric-H-double} is at least $s(b+1-s)=M(s)$.

If instead $s\ge d_0$, decrease the two tails within $[0,1/2]$ until their sum is $d_0$. Since $T'(t)=k(b+1-2t)>0$ on $[0,1/2]$ and $T$ takes values in $[0,1]$, the quantity $T(v)+T(w)-T(v)T(w)$ does not increase. At sum $d_0$, $vw\le d_0^2/4$, so the subtracted term in the bracket of \eqref{eq:cubic-geometric-H-double} is at most
\[
 k(b+1)^2\left(1-\frac12+\frac1{16}\right)
 =\frac94\left(\frac{b+1}{2b+1}\right)^2
 \le\frac94\cdot\frac9{16}=\frac{81}{64}<2,
\]
since $(b+1)/(2b+1)\le3/4$ for $b\ge1/2$. Thus the value at sum $d_0$ is at least $d_0(b+1-d_0)=(b+1)^2/4=M(s)$. This proves \eqref{eq:cubic-geometric-H-reduction}, and hence the coverage inequality, in every case.

\emph{Feasibility.} Since $r\le\zeta$ and $K$ is down-closed, $r\one\in K$ whenever $\zeta\one\in K$. The bound $\zero\le H_r(x)\le\lambda x+(1-\lambda)r\one$, with $0\le\lambda\le1$, places the action below a convex combination of two points of $K$, and down-closedness gives $H_r(x)\in K$.
\end{proof}

\subsection{Completion of the cubic comparison}\label{app:cubic-assembly}
Apply the preceding constructions to $F=f-\min f$ from \eqref{eq:polynomial-minimum-shift}. By Lemma~\ref{lem:cubic-polarization}, $\widetilde F$ is a nonnegative multilinear DR-submodular polynomial of degree at most three, and Lemma~\ref{lem:cubic-multilinear} gives
\begin{equation}\label{eq:cubic-nonnegative-sum}
 \widetilde F(z)=b_0+\sum_{\ell=1}^N\rho_\ell\phi_\ell(z),
 \qquad b_0\ge0,\quad\rho_\ell\ge0,
\end{equation}
where the $\phi_\ell$ are unary, directed, coverage, or product components. Define
\begin{equation}\label{eq:cubic-majorant}
 C_F(x)=b_0+\sum_{\ell=1}^N\rho_\ell C_{\phi_\ell}(Lx).
\end{equation}
Each summand is concave and majorizes its component, so $C_F$ is concave and, by \eqref{eq:cubic-diagonal-identity}, $C_F(x)\ge\widetilde F(Lx)=F(x)$ on the cube.

For each minimum of affine functions, select its first active branch in a fixed order, and let $v_\ell(z)$ be the gradient of that branch; for unary terms take the gradient of the affine term. Define
\begin{equation}\label{eq:cubic-supergradient}
 s_f(x)=L^\top\sum_{\ell=1}^N\rho_\ell v_\ell(Lx).
\end{equation}
If the selected branch of $C_{\phi_\ell}$ at $z$ is the affine function $\mathsf M$, then $C_{\phi_\ell}(z')\le\mathsf M(z')=C_{\phi_\ell}(z)+\ip{v_\ell(z)}{z'-z}$ for all $z'$. Multiplying by $\rho_\ell$, summing, setting $z=Lx$ and $z'=Lo$, and using $\ip{v}{Lo-Lx}=\ip{L^\top v}{o-x}$ gives
\begin{equation}\label{eq:cubic-majorant-supergradient}
 C_F(o)\le C_F(x)+\ip{s_f(x)}{o-x}.
\end{equation}

For $0\le\zeta\le1/2$, use the two-map action of Lemma~\ref{lem:cubic-rounding} if $8/17\ge A_3(\zeta)$ and the map $H_r$ of Lemma~\ref{lem:cubic-geometry} with $r=\zeta$ otherwise. For $\zeta\ge1/2$, use $H_r$ with $r=1/2$, for which $b=1/2$ and $a_b=1/2=\alpha_3(\zeta)$. In every case the selected action is feasible and satisfies $\E\phi(p)\ge\alpha_3(\zeta)C_\phi(z)$ when it is applied coordinatewise to $z$, for every component $\phi$. A scalar map applied to every coordinate commutes with $L$, so \eqref{eq:cubic-diagonal-identity}, \eqref{eq:cubic-nonnegative-sum}, the component inequalities at $z=Lx$, and $b_0\ge\alpha_3(\zeta)b_0$ give
\begin{equation}\label{eq:cubic-rounding-sum}
 \E F(p_\zeta(x))
 =\E\widetilde F(Lp_\zeta(x))
 \ge\alpha_3(\zeta)\left(b_0+\sum_{\ell=1}^N\rho_\ell C_{\phi_\ell}(Lx)\right)
 =\alpha_3(\zeta)C_F(x).
\end{equation}
Combining $F(o)\le C_F(o)$, \eqref{eq:cubic-majorant-supergradient}, and \eqref{eq:cubic-rounding-sum} proves
\begin{equation}\label{eq:cubic-final-comparison}
 \alpha_3(\zeta)F(o)-\E F(p_\zeta(x))
 \le\alpha_3(\zeta)\ip{s_f(x)}{o-x},
\end{equation}
and the argument after \eqref{eq:polynomial-minimum-shift} transfers this inequality to $f$.

For the field bound, an active branch of a unary or directed component has $\ell_1$ norm one, and an active branch of a coverage or product component has $\ell_1$ norm at most two. At the all-half point the component values are $1/2$ (unary), $1/4$ (directed), and $3/8$ (coverage and product). Hence, for every component,
\begin{equation}\label{eq:cubic-component-field-size}
 \norm{v_\ell(z)}_1\le\Lambda_3\,\phi_\ell(\tfrac12\one),
 \qquad\Lambda_3=\frac{16}{3},
\end{equation}
with equality for coverage and product components. The map $L^\top$ sums the clone coordinates and does not increase the $\ell_1$ norm. Therefore, using $b_0\ge0$ and $L\tfrac12\one=\tfrac12\one$,
\begin{align}
 \norm{s_f(x)}_2
 &\le\norm{s_f(x)}_1
 \le\Lambda_3\sum_{\ell=1}^N\rho_\ell\phi_\ell(\tfrac12\one)\nonumber\\
 &\le\Lambda_3\widetilde F(\tfrac12\one)
 =\Lambda_3\left(f(\tfrac12\one)-\min f\right).
 \label{eq:cubic-field-bound}
\end{align}
For a minimizer $y$ of $f$ on the cube, the gradient bound and \eqref{eq:segment-ftc} on the segment from $y$ to $\tfrac12\one$ give
$f(\tfrac12\one)-f(y)\le B\norm{\tfrac12\one-y}_2\le B\sqrt d/2$, so $\norm{s_f(x)}_2\le(\Lambda_3/2)B\sqrt d$. This proves \eqref{eq:polynomial-field-bound} for $m=3$.

Polarization has $3d$ variables and $O(d^3)$ monomials, and the auxiliary graph has $O(d^3)$ vertices and arcs. One maximum-flow calculation and $O(d^3)$ constant-size ternary decompositions produce \eqref{eq:cubic-nonnegative-sum}, after which every coefficient and active branch can be evaluated in polynomial time. This proves all assertions of Theorem~\ref{thm:polynomial}.

\section{Proof of the smoothing step in Lemma~\ref{lem:continuous-refinement}}\label{app:upper-smoothing}\label{app:upper}

\begin{lemma}[Smooth symmetrization with local first-order agreement]
\label{lem:upper-smoothing}
Let $F:[0,1]^{d_0}\to\mathbb R_+$ be a multilinear
DR-submodular polynomial invariant under a finite group
$\mathcal G$ of coordinate permutations, and let
\[
    \mathsf T y
    =\frac{1}{|\mathcal G|}\sum_{\sigma\in\mathcal G}\sigma(y)
\]
denote the corresponding averaging operator.

For every $\nu>0$, there exist $\delta>0$ and nonnegative,
twice continuously differentiable DR-submodular functions
$\widehat F,\widehat G:[0,1]^{d_0}\to\mathbb R_+$
such that
\begin{equation}\label{eq:appendix-smoothing}
\begin{gathered}
    \widehat G(y)=\widehat F(\mathsf T y),
    \qquad
    \|\widehat F-F\|_\infty\le\nu,\\
    \bigl(\widehat F(y),\nabla\widehat F(y)\bigr)
    =
    \bigl(\widehat G(y),\nabla\widehat G(y)\bigr)
    \quad\text{if }
    \|y-\mathsf T y\|_2^2\le\delta.
\end{gathered}
\end{equation}
\end{lemma}

\begin{proof}
Write $k=d_0$ and identify $\mathsf T$ with the matrix $|\mathcal G|^{-1}\sum_{\sigma}\Pi_\sigma$, where $\Pi_\sigma$ is the permutation matrix of $\sigma$. Since $\Pi_\sigma^\top=\Pi_{\sigma^{-1}}$ and $\mathsf T\Pi_\sigma=\mathsf T$, the matrix $\mathsf T$ is symmetric and idempotent, so it is the orthogonal projection onto the subspace of invariant vectors. It is also entrywise nonnegative with unit row and column sums, so it maps the cube into itself and preserves $\sum_iy_i$. Set
\begin{equation}\label{eq:smoothing-auxiliaries}
 G_0(y)=F(\mathsf T y),\qquad R_0(y)=F(y)-G_0(y),\qquad
 D_0(y)=\|y-\mathsf T y\|_2^2.
\end{equation}
We use the interpolation and invariant concave correction from the proof of \citet[Lemma~3.2]{vondrak2013}, with a twice continuously differentiable cutoff.

Differentiating $F(\Pi_\sigma y)=F(y)$ gives $\Pi_\sigma^\top\nabla F(\Pi_\sigma y)=\nabla F(y)$, so the gradient of $F$ at an invariant point is invariant and is fixed by $\mathsf T$. Since $\mathsf Ty$ is invariant and $\mathsf T^2=\mathsf T$, this gives
\[
 \nabla G_0(y)=\mathsf T\nabla F(\mathsf T y)=\nabla F(\mathsf T y),
 \qquad R_0(\mathsf T y)=0,\qquad \nabla R_0(\mathsf T y)=0.
\]
Choose
\[
 \Lambda=\max\left\{1,\max_{y\in[0,1]^k}
                         \|\nabla^2R_0(y)\|_{\mathrm{op}}\right\}.
\]
The segment from $\mathsf T y$ to $y$ lies in the cube and has length $\sqrt{D_0(y)}$. Taylor's formula along it, with the vanishing value and gradient of $R_0$ at $\mathsf Ty$, gives
\begin{equation}\label{eq:smoothing-remainder-bounds}
 |R_0(y)|\le\frac{\Lambda}{2}D_0(y),\qquad
 \|\nabla R_0(y)\|_2\le\Lambda\sqrt{D_0(y)}.
\end{equation}
Since $I-\mathsf T$ is a symmetric idempotent matrix whose diagonal entries lie in $[0,1]$ and whose off-diagonal entries lie in $[-1,0]$,
\begin{equation}\label{eq:smoothing-distance-bounds}
 \nabla D_0(y)=2(y-\mathsf T y),\qquad
 \nabla^2D_0=2(I-\mathsf T),\qquad
 |\partial_{ij}D_0|\le2.
\end{equation}
In particular $|\partial_iD_0(y)|\le2\|y-\mathsf Ty\|_2=2\sqrt{D_0(y)}$.

For the required error $\nu>0$, choose
\begin{equation}\label{eq:smoothing-scales}
 b=\min\{1,\nu/\Lambda\},\qquad
 L=\max\{1,132\Lambda k^2/\nu\},\qquad
 \delta=b\exp(-L).
\end{equation}
Let $\psi(s)=1-10s^3+15s^4-6s^5$ for $0\le s\le1$ and define
\begin{equation}\label{eq:smoothing-cutoff}
 \phi(t)=
 \begin{cases}
 1,&0\le t\le\delta,\\
 \psi\!\left(\log(t/\delta)/L\right),&\delta<t<b,\\
 0,&t\ge b.
 \end{cases}
\end{equation}
The argument $\log(t/\delta)/L$ increases from $0$ at $t=\delta$ to $1$ at $t=b$. Since $\psi(0)=1$, $\psi(1)=0$, $\psi'(s)=-30s^2(1-s)^2$, and
$\psi''(s)=-60s(1-s)(1-2s)$, the function $\psi$ decreases from one to zero and its first two derivatives vanish at both endpoints. Thus $\phi$ is twice continuously differentiable on $[0,\infty)$, takes values in $[0,1]$, and satisfies
\begin{equation}\label{eq:smoothing-cutoff-bounds}
 |t\phi'(t)|\le\frac2L,\qquad
 |t^2\phi''(t)|\le\frac{17}{L}.
\end{equation}
On the transition interval, $t\phi'=\psi'/L$ and $t^2\phi''=\psi''/L^2-\psi'/L$, so these bounds follow from $|\psi'|\le15/8$, $|\psi''|\le15$, and $L\ge1$; outside it, $\phi'$ and $\phi''$ vanish.

Write $s(y)=\sum_i y_i$ and set
\begin{equation}\label{eq:smoothing-construction}
 \begin{split}
 J_0(y)&=k^2+3ks(y)-s(y)^2,\qquad a=22\Lambda/L,\\
 \widetilde F(y)&=F(y)-\phi(D_0(y))R_0(y),\\
 \widehat F(y)&=\widetilde F(y)+aJ_0(y),\qquad
 \widehat G(y)=G_0(y)+aJ_0(y).
 \end{split}
\end{equation}
Both functions are nonnegative on the cube, because $\widetilde F=(1-\phi(D_0))F+\phi(D_0)G_0$ is a convex combination of $F\ge0$ and $G_0\ge0$, and $s(3k-s)\ge0$ for $0\le s\le k$ gives $k^2\le J_0\le3k^2$. The same formulas define twice continuously differentiable functions on $\R^k$, so the gradients on the cube boundary are the continuous extensions required by Definition~\ref{def:dr}. Since $\mathsf T$ preserves $s(y)$, we have $J_0(\mathsf Ty)=J_0(y)$; together with $D_0(\mathsf Ty)=0$, $R_0(\mathsf Ty)=0$, and $F(\mathsf Ty)=G_0(y)$, this gives
\begin{equation}\label{eq:smoothing-composition}
 \widehat G(y)=\widehat F(\mathsf T y).
\end{equation}
Moreover, $\widehat F-\widehat G=(1-\phi(D_0))R_0$ has gradient $(1-\phi(D_0))\nabla R_0-\phi'(D_0)R_0\nabla D_0$. Whenever $D_0(y)\le\delta$, we have $\phi(D_0(y))=1$ and $\phi'(D_0(y))=0$, so
\begin{equation}\label{eq:smoothing-exact-derivatives}
 \widehat F(y)=\widehat G(y),\qquad
 \nabla\widehat F(y)=\nabla\widehat G(y).
\end{equation}

Since $\phi(D_0)=0$ when $D_0\ge b$, the bounds \eqref{eq:smoothing-remainder-bounds}, $J_0\le3k^2$, and \eqref{eq:smoothing-scales} give
\begin{equation}\label{eq:smoothing-uniform-error}
 \|\widehat F-F\|_\infty
 \le\frac{\Lambda b}{2}+3ak^2
 =\frac{\Lambda b}{2}+\frac{66\Lambda k^2}{L}\le\frac\nu2+\frac\nu2=\nu.
\end{equation}

For DR-submodularity, the Hessian of $F$ is entrywise nonpositive, with zero diagonal by multilinearity, and
$\nabla^2G_0(y)=\mathsf T\nabla^2F(\mathsf T y)\mathsf T$ is entrywise nonpositive because $\mathsf T$ is entrywise nonnegative. Differentiating $\widetilde F=F-\phi(D_0)R_0$ twice gives, with $D=D_0(y)$ and all derivatives evaluated at $y$,
\begin{align}\label{eq:smoothing-hessian}
 \partial_{ij}\widetilde F
 ={}&(1-\phi(D))\partial_{ij}F+\phi(D)\partial_{ij}G_0\nonumber\\
 &-\phi'(D)\{(\partial_iD_0)(\partial_jR_0)
                  +(\partial_jD_0)(\partial_iR_0)
                  +(\partial_{ij}D_0)R_0\}\nonumber\\
 &-\phi''(D)(\partial_iD_0)(\partial_jD_0)R_0.
\end{align}
By $|\partial_iD_0|\le2\sqrt D$, \eqref{eq:smoothing-remainder-bounds}, and \eqref{eq:smoothing-distance-bounds}, the two products in braces are each at most $2\sqrt D\cdot\Lambda\sqrt D=2\Lambda D$ in absolute value and the third is at most $2\cdot\Lambda D/2$, so the braces are at most $5\Lambda D$; the final product is at most $4D\cdot\Lambda D/2=2\Lambda D^2$. The first line of \eqref{eq:smoothing-hessian} is nonpositive, so \eqref{eq:smoothing-cutoff-bounds} gives
\[
 \partial_{ij}\widetilde F
 \le5\Lambda D|\phi'(D)|+2\Lambda D^2|\phi''(D)|
 \le\frac{10\Lambda}{L}+\frac{34\Lambda}{L}=\frac{44\Lambda}{L}.
\]
Every second partial derivative of $J_0$ equals $-2$, so $a=22\Lambda/L$ yields
\[
 \partial_{ij}\widehat F=\partial_{ij}\widetilde F-2a\le0,\qquad
 \partial_{ij}\widehat G=\partial_{ij}G_0-2a\le0
 \quad\text{on }[0,1]^k,
 \quad\text{for all }i,j.
\]
If $u\le w$ in the cube, integrating $\partial_i\widehat F$ along the segment from $u$ to $w$ gives $\partial_i\widehat F(w)-\partial_i\widehat F(u)=\int_0^1\sum_j\partial_{ij}\widehat F(u+\tau(w-u))(w_j-u_j)\,d\tau\le0$, and similarly for $\widehat G$. Hence both functions are DR-submodular. Equations~\eqref{eq:smoothing-composition}, \eqref{eq:smoothing-exact-derivatives}, and \eqref{eq:smoothing-uniform-error} give \eqref{eq:appendix-smoothing}.
\end{proof}

\section{Proof of Proposition~\ref{prop:upper-evaluation}}\label{app:upper-evaluation}
\begin{proof}
The equality at zero is Corollary~\ref{cor:upper-endpoint}. We prove the finite reduction for $0<\zeta<1/8$ and the rational formula for $1/8\le\zeta<1/2$; for $\zeta\ge1/2$ the formula is the definition of $U$. Throughout the proof, abbreviate
\[
 h=\frac12-\zeta,\qquad \xi_n=\xi_{n,\zeta},\qquad
 z_n=\ell_n\xi_n,\qquad \upsilon_n=1-2\chi_n.
\]
For $0\le x\le1/2$, define the component contributions in the quadratic branches by
\begin{equation}\label{eq:upper-dual-components}
 c(x)=\frac12-2x^2,\qquad
 j(x)=\frac34+x-x^2,\qquad
 \mathcal H_n(x)=(1+2x)(\chi_n+z_nx).
\end{equation}
If the first branch is selected with probability $w$ and the second with probability $1-w$, both at the same $x$, their expected contributions from $F_{\mathrm{cut}},M_n,J_n,H_n$ are
\begin{equation}\label{eq:upper-common-contact}
 c(x),\qquad
 \mathcal M_n(x,w):=(r_n+\xi_nx)(w/r_n+1-w),\qquad
 (1-w)j(x),\qquad \mathcal H_n(x),
\end{equation}
where the second uses $1+N\xi_nx=(r_n+\xi_nx)/r_n$. Suppose all four numbers are at least $\vartheta$. Then for every $\omega=(\eta,\theta,\lambda,\mu)\in\Delta_4$,
\[
 \begin{split}
 \max_{j,x}Q_{j,n,\zeta}(x;\omega)
 &\ge wQ_{0,n,\zeta}(x;\omega)+(1-w)Q_{1,n,\zeta}(x;\omega)\\
 &=\eta c(x)+\theta\mathcal M_n(x,w)+\lambda(1-w)j(x)+\mu\mathcal H_n(x)\ge\vartheta,
 \end{split}
\]
so $V_n(\zeta)\ge\vartheta$.

For the upper bounds, the weights in \eqref{eq:upper-two-feasible-weights} and \eqref{eq:upper-strict-half} give
\begin{equation}\label{eq:upper-contact-values}
 x_n=\frac{\upsilon_n}{2(1+z_n)},\qquad
 v_n=c(x_n)=\frac12-\frac{\upsilon_n^2}{2(1+z_n)^2},
 \qquad V_n(\zeta)\le v_n.
\end{equation}
We call $x_n$ the contact point; it satisfies $\mathcal H_n(x_n)=c(x_n)$, since
\[
 \mathcal H_n(x)-c(x)
 =(1+2x)\{(1+z_n)x-\upsilon_n/2\}.
\]
This identity also shows that $\mathcal H_n(x)\ge c(x)$ for $x\ge x_n$.

For the finite reduction, put $b_n=2n\xi_n$. Whenever $b_n\ge1$, use \eqref{eq:upper-common-contact} with $x=1/(2b_n)\le1/2$ and $w=1/3$. Since $r_n+\xi_nx=r_n+1/(4n)=3/(4n)$ and $j$ increases on $[0,1/2]$, the first three expectations are
\[
 \frac12-\frac1{2b_n^2},\qquad
 \frac3{4n}\cdot\frac{2n+2}3=\frac12+r_n,\qquad
 \frac23j(x)\ge\frac23j(0)=\frac12.
\]
The fourth is
\[
 \mathcal H_n(x)
 =\left(1+\frac1{b_n}\right)
       \left(\chi_n+\frac{\ell_n}{4n}\right)>\frac12,
\]
because $\ell_n/n=(1-r_n)^{n-1}\ge(1-r_n)^n=1-\chi_n$ and $\chi_n>1/3$ (shown in the proof of Lemma~\ref{lem:upper-u0}) make the second factor at least $1/4+3\chi_n/4>1/2$. Consequently
\begin{equation}\label{eq:upper-tail-bound}
 V_n(\zeta)\ge\frac12-\frac1{2(2n\xi_n)^2}
 \qquad\text{when }2n\xi_n\ge1.
\end{equation}
Now let $0<\zeta<1/8$. Then $\xi_2=0$, and \eqref{eq:upper-contact-values} for $n=2$ gives $V_2(\zeta)\le v_2=63/128$. If $n\ge5/(2\zeta)-4$, then $n\ge1/(2\zeta)$ because $\zeta<1/2$, so $\zeta\ge r_n$ and
\[
 b_n=\frac{2n\zeta-1}{h}\ge8,
 \qquad\text{since}\qquad
 \frac{2n\zeta-1}{h}\ge8
 \quad\Longleftrightarrow\quad
 n\ge\frac5{2\zeta}-4.
\]
By \eqref{eq:upper-tail-bound}, every $n\ge N_\zeta$ satisfies $V_n(\zeta)\ge63/128\ge V_2(\zeta)$.

For $n\le m_\zeta$, we have $\zeta\le r_n$, so $\xi_n=0$, $P_{n,\zeta}=P_n$, and $V_n(\zeta)=V_n(0)$. Both $r_n$ and $\chi_n$ decrease with $n$. For $\chi_n=1-(1-1/(2n))^n$, set $r=1/(2t)$ for real $t\ge2$ and differentiate:
\[
 \frac{d}{dt}\log(1-1/(2t))^t
 =\log(1-r)+\frac{r}{1-r}
 =\sum_{k\ge2}\frac{k-1}{k}r^k>0.
\]
At $\xi_n=0$, the branch quadratics are $\eta c(x)+\theta+\mu(1+2x)\chi_n$ and $\eta c(x)+\theta r_n+\lambda j(x)+\mu(1+2x)\chi_n$, which decrease pointwise in $n$. Hence $V_n(0)$ is nonincreasing in $n$, and every $n\le m_\zeta$ satisfies $V_n(\zeta)=V_n(0)\ge V_{m_\zeta}(0)=V_{m_\zeta}(\zeta)$. Since $m_\zeta\ge4>2$, this applies to $n=2$, and with the preceding paragraph every index outside $[m_\zeta,N_\zeta]$ has $V_n(\zeta)\ge V_{m_\zeta}(\zeta)$. With $m_\zeta\le N_\zeta$, this proves \eqref{eq:upper-finite-evaluation}.

The next comparison of coverage contributions applies on every interval of \eqref{eq:upper-rational-evaluation}. For integers $n>m\ge2$, we claim
\begin{equation}\label{eq:upper-coverage-index-comparison}
 \zeta\ge r_m,\quad x\ge\frac{h}{4m}
 \quad\Longrightarrow\quad
 \mathcal H_n(x)\ge\mathcal H_m(x).
\end{equation}
Since $\mathcal H_n(x)-\mathcal H_m(x)=(1+2x)\{(z_n-z_m)x-(\chi_m-\chi_n)\}$, it suffices to show $(z_n-z_m)x\ge\chi_m-\chi_n$. For real $t\ge2$, let
\[
 A(t)=(1-1/(2t))^t,\qquad
 B(t)=(1-1/(2t))^{t-1},\qquad \ell(t)=tB(t).
\]
With $r=1/(2t)$, the inequality $\log(1-r)\ge-r/(1-r)$ gives
\[
 \frac{d}{dt}\log\ell(t)
 =\frac1t+\log(1-r)+\frac{(t-1)r}{t(1-r)}
 \ge\frac{1-2r}{t(1-r)}>0,
\]
so $\ell_n\ge\ell_m$. Also, since $(k-1)/k\le(k-1)/2$ for $k\ge2$,
\[
 \frac{d}{dt}\log A(t)
 =\sum_{k\ge2}\frac{k-1}{k}r^k
 \le\frac12\sum_{k\ge2}(k-1)r^k
 =\frac{r^2}{2(1-r)^2}
 =\frac1{8(t-1/2)^2}.
\]
Integrating from $m$ to $n$ gives $a:=\log\{A(n)/A(m)\}\le(n-m)/[8(m-1/2)(n-1/2)]$, and $1-e^{-a}\le a$ yields
\[
 \chi_m-\chi_n=A(n)-A(m)
 =A(n)\bigl(1-e^{-a}\bigr)
 \le A(n)a
 \le\frac{B(n)(n-m)}{8n(m-1/2)},
\]
where the last step uses $A(n)=B(n)(n-1/2)/n$. Since $\ell_n(r_m-r_n)=B(n)(n-m)/(2m)$, we obtain
\begin{equation}\label{eq:upper-coverage-ratio}
 \frac{\chi_m-\chi_n}{\ell_n(r_m-r_n)}
 \le\frac{m}{4n(m-1/2)}\le\frac1{4m}.
\end{equation}
The last inequality uses the integer condition $n\ge m+1$ and $(m+1)(m-1/2)\ge m^2$. Finally, $\zeta\ge r_m>r_n$ gives $\xi_n=(\zeta-r_n)/h$ and $\xi_m=(\zeta-r_m)/h$, so
\[
 z_n-z_m
 =\frac{(\ell_n-\ell_m)(\zeta-r_m)+\ell_n(r_m-r_n)}h
 \ge\frac{\ell_n(r_m-r_n)}h.
\]
Together with $x/h\ge1/(4m)$ and \eqref{eq:upper-coverage-ratio}, this gives $(z_n-z_m)x\ge\chi_m-\chi_n$, proving \eqref{eq:upper-coverage-index-comparison}.

For the matched contributions, fix $\zeta,x,w$ and put $\kappa=x/h$. For $r_n\le\zeta$ we have $r_n+\xi_nx=\kappa\zeta+(1-\kappa)r_n$, so $\mathcal M_n(x,w)=\mathcal M(r_n)$ for the function of a real variable $r\in(0,\zeta]$ given by
\begin{equation}\label{eq:upper-matched-index-derivative}
 \begin{split}
 \mathcal M(r)
 &=\frac{w\kappa\zeta}{r}+w(1-\kappa)
       +(1-w)\kappa\zeta+(1-w)(1-\kappa)r,\\
 \mathcal M'(r)
 &=-\frac{w\kappa\zeta}{r^2}+(1-w)(1-\kappa).
 \end{split}
\end{equation}
If $\mathcal M'<0$ on $(0,r_m]$ and $r_m\le\zeta$, then $\mathcal M_n(x,w)\ge\mathcal M_m(x,w)$ for every $n\ge m$. The data for the three relevant contact points are
\begin{equation}\label{eq:upper-small-index-data}
 \begin{array}{c|ccc}
 n&2&3&4\\\hline
 \chi_n&7/16&91/216&1695/4096\\[2pt]
 \ell_n&3/2&25/12&343/128\\[2pt]
 x_n(0)&1/16&17/216&353/4096,
 \end{array}
\end{equation}
where $x_n(0)=\upsilon_n/2$ is the contact point when $\xi_n=0$. Substituting $z_n=\ell_n(\zeta-r_n)/h$ into \eqref{eq:upper-contact-values} gives, on the ranges $\zeta\ge r_n$,
\begin{equation}\label{eq:upper-explicit-contacts}
 x_2=\frac{h}{2(1+4\zeta)},\qquad
 x_3=\frac{17h}{3(11+78\zeta)},\qquad
 x_4=\frac{353h}{4(169+1720\zeta)}.
\end{equation}
Each of these decreases in $\zeta$. Substituting them into $c(x)$ gives the three nonconstant rational expressions in \eqref{eq:upper-rational-evaluation}; for instance $c(x_2)=1/2-2h^2/[4(1+4\zeta)^2]=1/2-(1-2\zeta)^2/[8(1+4\zeta)^2]$. The breakpoints satisfy
\begin{equation}\label{eq:upper-contact-crossings}
 x_4(\zeta_{43})=\frac{17}{216},\qquad
 x_3(\zeta_{32})=\frac1{16}.
\end{equation}
On each interval below, we first record the upper bound from \eqref{eq:upper-contact-values} and then choose $x$ and $w$ so that the four expectations in \eqref{eq:upper-common-contact} are at least the same value for every $n\ge2$.

For $1/8\le\zeta\le\zeta_{43}$, we have $\zeta\ge r_4$, and \eqref{eq:upper-contact-values} gives $U(\zeta)\le v_4=c(x_4)$, the first formula. Choose $x=x_4$ and $w=1-c(x)/j(x)$. Then
\[
 c(x)=\mathcal H_4(x)=(1-w)j(x)=v_4.
\]
Since $x_4$ decreases in $\zeta$, \eqref{eq:upper-small-index-data} and \eqref{eq:upper-contact-crossings} give $17/216\le x\le353/4096$, and $h\le3/8$ gives $\kappa\ge(17/216)(8/3)=17/81$. As $c$ decreases and $j$ increases on this interval, $w$ increases with $x$, so
\[
 w\ge1-\frac{c(17/216)}{j(17/216)}=\frac{125}{307}>\frac{13}{32}.
\]
If $r\le1/8\le\zeta$, then
\[
 \frac{w\kappa\zeta}{r^2}
 \ge\frac{13}{32}\cdot\frac{17}{81}\cdot\frac18\cdot64
 =\frac{221}{324}>\frac{19}{32}
 \ge(1-w)(1-\kappa),
\]
so $\mathcal M'<0$ on $(0,1/8]$, and $\mathcal M_n(x,w)\ge\mathcal M_4(x,w)$ for $n\ge4$. Put $t=r_4+\xi_4x$. The contact equation $(1+z_4)x=\upsilon_4/2=353/4096$ gives
$t=r_4+(353/4096-x)/\ell_4$. Also $w/r_4+1-w=\{j(x)-(1-r_4)c(x)\}/\{r_4j(x)\}$, so
\begin{align}\label{eq:upper-four-matched-slack}
 j(x)\{\mathcal M_4(x,w)-c(x)\}
 &=\frac{t}{r_4}\{j(x)-(1-r_4)c(x)\}-c(x)j(x)\nonumber\\
 &=\frac{393}{21952}-\frac{241}{1372}x
   -\frac{233}{5488}x^2-\frac{82}{343}x^3-2x^4.
\end{align}
Every nonconstant coefficient is negative, so for $x\le353/4096<3/32$ the last expression exceeds its value $127737/179830784>0$ at $x=3/32$. Hence $\mathcal M_n(x,w)\ge c(x)$ for $n\ge4$. For $n>4$, \eqref{eq:upper-coverage-index-comparison} with $m=4$ gives $\mathcal H_n(x)\ge\mathcal H_4(x)$, since $\zeta\ge r_4$ and $\kappa\ge17/81>1/16$.

For $n=2,3$ in this interval, $\zeta<1/6\le r_n$, so $\xi_n=0$. Since $x\ge17/216=x_3(0)$, the contact identity gives $\mathcal H_3(x)\ge c(x)$, and $\chi_2>\chi_3$ gives $\mathcal H_2(x)\ge\mathcal H_3(x)$. Moreover, with $\xi_n=0$,
\[
 \mathcal M_2(x,w)=w+\frac{1-w}4\ge\mathcal M_3(x,w)
 =w+\frac{1-w}6\ge\frac{97}{192}>\frac12\ge c(x).
\]
All four expectations in \eqref{eq:upper-common-contact} are therefore at least $v_4$ for every $n$, so $U(\zeta)\ge v_4$, proving the first piece.

For $\zeta_{43}\le\zeta\le1/6$, we have $\zeta\le r_3$, so $\xi_3=0$ and \eqref{eq:upper-contact-values} gives $U(\zeta)\le v_3=c(17/216)=11375/23328$. Use $x=17/216$ and $w=13/32$. Then
\[
 c(x)=\mathcal H_3(x)=\frac{11375}{23328},\qquad
 (1-w)j(x)-c(x)=\frac{125}{165888}>0.
\]
For $n=2,3$, the matched and coverage bounds of the preceding paragraph apply unchanged. For $n\ge4$, we still have $\zeta\ge1/8$ and $\kappa\ge(17/216)/(3/8)=17/81$, so the derivative bound of the first piece gives $\mathcal M_n(x,w)\ge\mathcal M_4(x,w)$. At fixed $x,w$, the latter increases with $\zeta$ through $\xi_4$, and its slack at the left endpoint is
\[
 \mathcal M_4(x,w)\big|_{\zeta=\zeta_{43}}-c(x)
 =\frac{917591}{256048128}>0.
\]
At that endpoint, \eqref{eq:upper-contact-crossings} gives $x=x_4(\zeta_{43})$, so $\mathcal H_4(x)=c(x)$; this contribution also increases with $\zeta$. Equation~\eqref{eq:upper-coverage-index-comparison} with $m=4$ compares every $n>4$ with $n=4$, which proves the second piece.

For $1/6\le\zeta\le\zeta_{32}$, we have $\zeta\ge r_3$, and \eqref{eq:upper-contact-values} gives $U(\zeta)\le v_3=c(x_3)$, the third formula. Take $x=x_3$ and $w=9/23$. Since $x_3$ decreases in $\zeta$, \eqref{eq:upper-contact-crossings} gives $x\ge1/16$, and since $j$ increases and $c$ decreases,
\[
 c(x)=\mathcal H_3(x)=v_3,\qquad
 (1-w)j(x)\ge\frac{14}{23}j(1/16)=\frac{63}{128}=c(1/16)\ge c(x).
\]
By \eqref{eq:upper-explicit-contacts}, $\kappa=17/[3(11+78\zeta)]$, which is at least $1/5$ for $\zeta\le2/9$, hence on this interval; and $\kappa\zeta$ increases in $\zeta$, so $\kappa\zeta\ge17/432$, its value at $\zeta=1/6$. For $r\le1/6$, the two terms in \eqref{eq:upper-matched-index-derivative} satisfy
\[
 \frac{w\kappa\zeta}{r^2}\ge\frac9{23}\cdot\frac{17}{432}\cdot36
 =\frac{153}{276}>\frac{56}{115}
 \ge(1-w)(1-\kappa).
\]
Thus $\mathcal M_n(x,w)\ge\mathcal M_3(x,w)$ for $n\ge3$, and
\[
 \mathcal M_3(x,w)\ge r_3\Bigl(\frac w{r_3}+1-w\Bigr)=w+\frac{1-w}6
 =\frac{34}{69}>\frac{63}{128}\ge c(x).
\]
Equation~\eqref{eq:upper-coverage-index-comparison} with $m=3$ gives $\mathcal H_n(x)\ge\mathcal H_3(x)$ for $n>3$, because $\kappa\ge1/5>1/12$. For $n=2$, $\zeta<1/4$ gives $\xi_2=0$; then $x\ge1/16=x_2(0)$ gives $\mathcal H_2(x)\ge c(x)$, while $\mathcal M_2(x,w)=w+(1-w)/4=25/46>c(x)$. This proves the third piece.

For $\zeta_{32}\le\zeta\le1/4$, we have $\xi_2=0$, and \eqref{eq:upper-contact-values} gives $U(\zeta)\le v_2=63/128$. Use $x=1/16$ and $w=9/23$. Then
\[
 c(x)=\mathcal H_2(x)=(1-w)j(x)=\frac{63}{128},\qquad
 \mathcal M_2(x,w)=\frac{25}{46}>c(x).
\]
Here $\zeta\ge\zeta_{32}>1/5$, so $h<3/10$ and $\kappa\ge5/24$. For $r\le1/6$,
\[
 \frac{w\kappa\zeta}{r^2}
 \ge\frac9{23}\cdot\frac5{24}\cdot\frac15\cdot36
 =\frac{27}{46}>\frac{133}{276}
 \ge(1-w)(1-\kappa).
\]
It follows that $\mathcal M_n(x,w)\ge\mathcal M_3(x,w)$ for $n\ge3$. The latter increases with $\zeta$ and satisfies
\[
 \mathcal M_3(x,w)\big|_{\zeta=\zeta_{32}}-c(x)
 =\frac{15607}{662400}>0.
\]
At this left endpoint, $x=x_3(\zeta_{32})$ by \eqref{eq:upper-contact-crossings}, so $\mathcal H_3(x)=c(x)$, and $\mathcal H_3(x)$ increases with $\zeta$. Equation~\eqref{eq:upper-coverage-index-comparison} with $m=3$ and $\kappa>1/12$ gives the coverage comparison for every $n>3$, which proves the fourth piece.

Finally, for $1/4\le\zeta<1/2$, \eqref{eq:upper-contact-values} gives $U(\zeta)\le v_2=c(x_2)$, the fifth formula. Choose $x=x_2$ and $w=1/3$. Then
\[
 c(x)=\mathcal H_2(x)=v_2,\qquad
 (1-w)j(x)=\frac23j(x)\ge\frac12\ge c(x),\qquad
 \kappa=\frac1{2+8\zeta}\ge\frac16.
\]
For $r=r_n\le1/4\le\zeta$, multiplication by the positive number $6r(2+8\zeta)$ shows that $\mathcal M_n(x,w)\ge1/2$ is equivalent to
\[
 2\zeta(1-2r+16r^2)-4r+4r^2\ge0.
\]
Since $1-2r+16r^2>0$ and $\zeta\ge1/4$, the left side is at least
\[
 \frac12-5r+12r^2=(3r-1/2)(4r-1).
\]
This is nonnegative for $r\le1/6$ and for $r=1/4$, which covers every $r_n$ with integer $n\ge2$. Hence $\mathcal M_n(x,w)\ge1/2\ge c(x)$. The coverage comparison \eqref{eq:upper-coverage-index-comparison} applies with $m=2$, since $\zeta\ge r_2$ and $\kappa\ge1/6>1/8$. Thus every $V_n(\zeta)\ge v_2$, proving the fifth piece.

At the breakpoints, \eqref{eq:upper-contact-crossings} gives $c(x_4(\zeta_{43}))=c(17/216)=11375/23328$ and $c(x_3(\zeta_{32}))=c(1/16)=63/128$, while \eqref{eq:upper-explicit-contacts} gives $x_3(1/6)=17/216$ and $x_2(1/4)=1/16$. Hence all neighboring pieces agree, and the last rational piece tends to $1/2$ as $\zeta\uparrow1/2$.

For the limit at zero, let $a>0$. Corollary~\ref{cor:upper-endpoint} supplies a finite $n$ with $V_n(0)<\beta_*+a$. For $0\le\zeta\le r_n$, monotonicity of $U$ and $\xi_n=0$ give
\[
 \beta_*=U(0)\le U(\zeta)\le V_n(\zeta)=V_n(0)<\beta_*+a.
\]
This proves $U(\zeta)\to\beta_*$ as $\zeta\downarrow0$.
\end{proof}

\bibliography{references}

@article{aggarwal2026,
  author={V. Aggarwal and Y. Lu},
  title={Online Non-Monotone {DR}-Submodular Maximization Matching the Offline 0.401 Factor},
  journal={arXiv preprint arXiv:2609.02145},
  year={2026}
}

@article{feldman2026stronger,
  title={Stronger Hardness for Submodular Maximization Subject to a Matroid Constraint},
  author={Feldman, Moran and Kuhnle, Alan},
  journal={arXiv preprint arXiv:2610.07387},
  year={2026}
}

@inproceedings{lu2026,
  author={Y. Lu and H. Jadav and M. Pedramfar and R. Singh and V. Aggarwal},
  title={Upper-Linearizability of Online Non-Monotone {DR}-Submodular Maximization over Down-Closed Convex Sets},
  booktitle={Proceedings of the International Conference on Machine Learning},
  year={2026},
  note={Full version: arXiv:2602.20578v2}
}

@inproceedings{pedramfar2024linear,
  author={M. Pedramfar and V. Aggarwal},
  title={From Linear to Linearizable Optimization: A Novel Framework with Applications to Stationary and Non-stationary {DR}-submodular Optimization},
  booktitle={Advances in Neural Information Processing Systems},
  year={2024},
  note={Extended version: arXiv:2405.00065v4}
}

@inproceedings{buchbinder2024,
  author={N. Buchbinder and M. Feldman},
  title={Constrained Submodular Maximization via New Bounds for {DR}-Submodular Functions},
  booktitle={Proceedings of the 56th Annual ACM Symposium on Theory of Computing},
  year={2024}
}

@inproceedings{bian2017guaranteed,
  author={Bian, A. A. and Mirzasoleiman, B. and Buhmann, J. M. and Krause, A.},
  title={Guaranteed Non-convex Optimization: Submodular Maximization over Continuous Domains},
  booktitle={Proceedings of the 20th International Conference on Artificial Intelligence and Statistics},
  year={2017}
}

@inproceedings{bian2017continuous,
  author={Bian, A. and Levy, K. Y. and Krause, A. and Buhmann, J. M.},
  title={Continuous {DR}-Submodular Maximization: Structure and Algorithms},
  booktitle={Advances in Neural Information Processing Systems},
  year={2017}
}

@inproceedings{bian2019,
  author={Y. Bian and J. Buhmann and A. Krause},
  title={Optimal Continuous {DR}-Submodular Maximization and Applications to Provable Mean Field Inference},
  booktitle={Proceedings of the 36th International Conference on Machine Learning},
  year={2019}
}

@inproceedings{feldman2011,
  author={M. Feldman and J. Naor and R. Schwartz},
  title={A Unified Continuous Greedy Algorithm for Submodular Maximization},
  booktitle={Proceedings of the 52nd Annual IEEE Symposium on Foundations of Computer Science},
  year={2011}
}

@inproceedings{ene2016,
  author={A. Ene and H. L. Nguyen},
  title={Constrained Submodular Maximization: Beyond 1/e},
  booktitle={Proceedings of the 57th Annual IEEE Symposium on Foundations of Computer Science},
  year={2016}
}

@article{buchbinder2019,
  author={N. Buchbinder and M. Feldman},
  title={Constrained Submodular Maximization via a Nonsymmetric Technique},
  journal={Mathematics of Operations Research},
  year={2019}
}

@inproceedings{chen2018,
  author={L. Chen and H. Hassani and A. Karbasi},
  title={Online Continuous Submodular Maximization},
  booktitle={Proceedings of the 21st International Conference on Artificial Intelligence and Statistics},
  year={2018}
}

@inproceedings{thang2021,
  author={N. K. Thang and A. Srivastav},
  title={Online Non-Monotone {DR}-Submodular Maximization},
  booktitle={Proceedings of the AAAI Conference on Artificial Intelligence},
  year={2021}
}

@inproceedings{zhang2023,
  author={Q. Zhang and Z. Deng and Z. Chen and K. Zhou and H. Hu and Y. Yang},
  title={Online Learning for Non-Monotone {DR}-Submodular Maximization: From Full Information to Bandit Feedback},
  booktitle={Proceedings of the 26th International Conference on Artificial Intelligence and Statistics},
  year={2023}
}

@inproceedings{pedramfar2024unified,
  author={M. Pedramfar and Y. Y. Nadew and C. J. Quinn and V. Aggarwal},
  title={Unified Projection-Free Algorithms for Adversarial {DR}-Submodular Optimization},
  booktitle={The Twelfth International Conference on Learning Representations},
  year={2024}
}

@inproceedings{roughgarden2018,
  author={T. Roughgarden and J. R. Wang},
  title={An Optimal Learning Algorithm for Online Unconstrained Submodular Maximization},
  booktitle={Proceedings of the 31st Conference on Learning Theory},
  year={2018}
}

@inproceedings{zinkevich2003,
  author={M. Zinkevich},
  title={Online Convex Programming and Generalized Infinitesimal Gradient Ascent},
  booktitle={Proceedings of the 20th International Conference on Machine Learning},
  year={2003}
}

@inproceedings{oveis2011,
  author={Oveis Gharan, S. and Vondr{\'a}k, J.},
  title={Submodular Maximization by Simulated Annealing},
  booktitle={Proceedings of the 22nd Annual ACM--SIAM Symposium on Discrete Algorithms},
  year={2011}
}

@article{vondrak2013,
  author={Vondr{\'a}k, J.},
  title={Symmetry and Approximability of Submodular Maximization Problems},
  journal={SIAM Journal on Computing},
  year={2013}
}

@article{kolmogorov2004,
  author={V. Kolmogorov and R. Zabih},
  title={What Energy Functions Can Be Minimized via Graph Cuts?},
  journal={IEEE Transactions on Pattern Analysis and Machine Intelligence},
  year={2004}
}

@article{billionnet1985,
  author = {A. Billionnet and M. Minoux},
  title = {Maximizing a Supermodular Pseudo-Boolean Function: A Polynomial Algorithm for Supermodular Cubic Functions},
  journal = {Discrete Applied Mathematics},
  year = {1985}
}

@inproceedings{chekuri2010swap,
 author={C. Chekuri and J. Vondr{\'a}k and R. Zenklusen},
 title={Dependent Randomized Rounding via Exchange Properties of Combinatorial Structures},
 booktitle={Proceedings of the 51st Annual IEEE Symposium on Foundations of Computer Science},
 year={2010}
}

@article{kulik2013knapsack,
 author={A. Kulik and H. Shachnai and T. Tamir},
 title={Approximations for Monotone and Nonmonotone Submodular Maximization with Knapsack Constraints},
 journal={Mathematics of Operations Research},
 year={2013}
}

@article{hoeffding1963,
 author={W. Hoeffding},
 title={Probability Inequalities for Sums of Bounded Random Variables},
 journal={Journal of the American Statistical Association},
 year={1963}
}

@article{tardos1986,
 author={E. Tardos},
 title={A Strongly Polynomial Algorithm to Solve Combinatorial Linear Programs},
 journal={Operations Research},
 year={1986}
}

@inproceedings{mualem2023,
 author={L. Mualem and M. Feldman},
 title={Resolving the Approximability of Offline and Online Non-monotone {DR}-Submodular Maximization over General Convex Sets},
 booktitle={Proceedings of the 26th International Conference on Artificial Intelligence and Statistics},
 year={2023}
}

@article{aggarwal2026geometryv1,
 author={V. Aggarwal},
 title={Geometry-Dependent Bounds for Online Non-Monotone {DR}-Submodular Maximization},
 journal={arXiv preprint arXiv:2610.00545v1},
 year={2026}
}

@article{du2026roofs,
 author={D. Du},
 title={Concave Roofs and Downward Recovery for {DR}-Submodular Polynomial Maximization},
 journal={arXiv preprint arXiv:2610.04027v1},
 year={2026}
}

@article{qi2022,
  author={B. Qi},
  title={On Maximizing Sums of Non-monotone Submodular and Linear Functions},
  journal={arXiv preprint arXiv:2205.15874},
  year={2022}
}

@inproceedings{pedramfar2026weakly,
  author={M. Pedramfar and V. Aggarwal},
  title={$\gamma$-Weakly $\theta$-Up-Concavity: A Unified Framework for Non-Convex Optimization Beyond {DR}-Submodular and {OSS} Functions},
  booktitle={Advances in Neural Information Processing Systems},
  year={2026},
  note={Available at arXiv:2602.13506}
}
\end{document}